\documentclass[runningheads]{llncs}

\usepackage{eccv}

\usepackage{eccvabbrv}

\usepackage{graphicx}
\usepackage{booktabs}

\usepackage[accsupp]{axessibility} 

\usepackage{hyperref}

\usepackage{orcidlink}

\usepackage{dsfont}
\usepackage{etoolbox}
\usepackage{color}
\usepackage{xspace}

\makeatletter
\DeclareRobustCommand*\cal{\@fontswitch\relax\mathcal}
\makeatother

\newif\ifshowedits
\showeditstrue

\newcommand{\addeditor}[3]{%
  \definecolor{#1color}{rgb}{#3}
  \expandafter\newcommand\csname #1\endcsname[1]{%
  \ifshowedits
    {\color{#1color} ##1}%
  \else
    {##1}%
  \fi
  }%
  \expandafter\newcommand\csname #1rmk\endcsname[1]{%
  \ifshowedits
    {\color{#1color} {\bf [#2: ##1]}}
  \fi
  }%
  \expandafter\newcommand\csname #1rpl\endcsname[2]{%
  \ifshowedits
    {\color{#1color} ##1 \sout{##2}}
  \else
    {##1}
  \fi
  }%
}

\newcommand{\createtextvar}[1]{
  \expandafter\newcommand\csname #1\endcsname{%
  {\text{#1}}
}%
}
\newcommand{\mycomment}[1]{}

\newcommand{\calL}{{\cal L}}

\newcommand{\calN}{{\cal N}}
\newcommand{\calO}{{\cal O}}

\newcommand{\calS}{{\cal S}}

\newcommand{\posdefm}{{\mathbb{S}^3_{++}}}

\newcommand{\IE}{{\mathds{E}}}

\newcommand{\IR}{{\mathds{R}}}

\newcommand{\methodname}{Gaussian Wrapping\xspace}
\newcommand{\meshingmethodname}{Primal Adaptive Meshing\xspace}

\newcommand{\rayw}{{\textbf{w}}}
\newcommand{\rayo}{{\textbf{o}}}
\newcommand{\ray}{{r_{\rayo, \rayw}}}
\newcommand{\rayx}{{\textbf{x}}}
\newcommand{\tff}{{t^*_{\rayo, \rayw}}}
\newcommand{\trans}{{\text{T}_{\rayo,\rayw}}}
\newcommand{\pdft}{{p_{\rayo,\rayw}^{\text{ff}}}}

\newcommand{\diff}{{\text{d}}}
\newcommand{\mean}{{\mu}}
\newcommand{\meani}{{\mu_i}}

\newcommand{\tstar}{{t_{\rayo, \rayw}^{G}}}
\newcommand{\tstari}{{t_{\rayo, \rayw}^{G_i}}}
\newcommand{\tstarj}{{t_{\rayo, \rayw}^{G_j}}}

\newcommand{\transi}{{\text{T}_{\rayo,\rayw}^{(i)}}}
\newcommand{\Gstari}{{\bar{G}_{\rayo, \rayw}^{(i)}}}
\newcommand{\transj}{{\text{T}_{\rayo,\rayw}^{(j)}}}
\newcommand{\Gstarj}{{\bar{G}_{\rayo, \rayw}^{(j)}}}

\newcommand{\ddt}{{\frac{\diff}{\diff t}}}

\newcommand{\charx}{{\mathds{1}_{n_i^T(\rayx-\meani) \geq 0}}}
\newcommand{\charxnorm}{{\mathds{1}_{n_i^T(x-\meani) \geq 0}}}

\newcommand{\median}{{\text{Median}_M}}

\newcommand{\orientedsigma}{{\bar{\sigma}}}

\newcommand{\colorfield}{{\textbf{c}}}

\usepackage{colortbl}

\definecolor{rankFirst}{rgb}{1.0, 0.70, 0.70}   
\definecolor{rankSecond}{rgb}{1.0, 0.85, 0.70}  
\definecolor{rankThird}{rgb}{1.0, 1.00, 0.70}   

\definecolor{tablethree}{rgb}{0.7, 1, 1}
\definecolor{tabletwo}{rgb}{0.7, 0.85, 1}
\definecolor{tableone}{rgb}{0.7, 0.7, 1}

\newcommand{\fst}[1]{\cellcolor{rankFirst}#1}
\newcommand{\snd}[1]{\cellcolor{rankSecond}#1}
\newcommand{\trd}[1]{\cellcolor{rankThird}#1}
\newcommand{\best}{\cellcolor{rankFirst}}
\newcommand{\sbest}{\cellcolor{rankSecond}}
\newcommand{\tbest}{\cellcolor{rankThird}}

\newcommand{\bestbis}{\cellcolor{tableone}}
\newcommand{\sbestbis}{\cellcolor{tabletwo}}
\newcommand{\tbestbis}{\cellcolor{tablethree}}
\usepackage{dsfont}
\usepackage{comment}
\usepackage{amsthm}
\usepackage{bbold}
\usepackage[table]{xcolor}
\usepackage{multirow}

\addeditor{antoine}{AG}{0.0, 0.5, 0.0}
\addeditor{diego}{DG}{0.0, 0.0, 0.8}
\addeditor{bingchen}{BG}{0.9, 0., 0.9}
\addeditor{nissim}{NM}{0.9, 0.4, 0.}
\addeditor{maks}{MO}{0.3, 0.4, 0.95}
\addeditor{george}{GD}{0.0, 0.4, 0.9}
\addeditor{todo}{TODO}{0.9, 0.0, 0.0}

\showeditstrue
\usepackage[breakable]{tcolorbox}
\usepackage{xcolor}
\definecolor{dblue}{HTML}{F3F8FF}
\definecolor{dorange}{HTML}{fffaf3}

\makeatletter
\newcommand{\hypbox}[2][]{%
 \begin{tcolorbox}[
   colframe=white,
   colback=dblue, 
   arc=4mm,
  boxrule=0pt,
  boxsep=0pt
 ]
 \@ifnotempty{#1}{%
  \textbf{\textcolor{dblue!90}{#1}}\\[-0.5em]
  \textcolor{dblue!90}{\rule{\linewidth}{0.4pt}}
 }%
 #2\end{tcolorbox}
}

\newcommand{\proofbox}[2][]{%
 \begin{tcolorbox}[
   colframe=white,
   colback=dorange, 
   arc=4mm,
  boxrule=0pt,
  boxsep=0pt,
  breakable
 ]
 \@ifnotempty{#1}{%
  \textbf{\textcolor{dblue!90}{#1}}\\[-0.5em]
  \textcolor{dblue!90}{\rule{\linewidth}{0.4pt}}
 }%
 #2\end{tcolorbox}
}
\makeatother
\begin{document}

\title{From Blobs to Spokes: High-Fidelity Surface Reconstruction via Oriented Gaussians}

\titlerunning{Gaussian Wrapping}
\author{
    Diego Gomez\inst{1}\thanks{Both authors contributed equally to the paper.}\orcidlink{0009-0005-2847-8617} \and
    Antoine Guédon\inst{1}$^\star$\orcidlink{0009-0001-3107-4454} \and
    Nissim Maruani\inst{1,2}\orcidlink{0009-0005-0866-7058} \and
    Bingchen Gong\inst{1}\orcidlink{0000-0001-6459-6972} \and
    Maks Ovsjanikov\inst{1}\orcidlink{0000-0002-5867-4046}
}

\authorrunning{D.~Gomez et al.}

\institute{
LIX, École Polytechnique, France \and
Inria, Côte d'Azur, France
}

\maketitle

\begin{figure*}[!h]
\centering

\includegraphics[width=\linewidth]{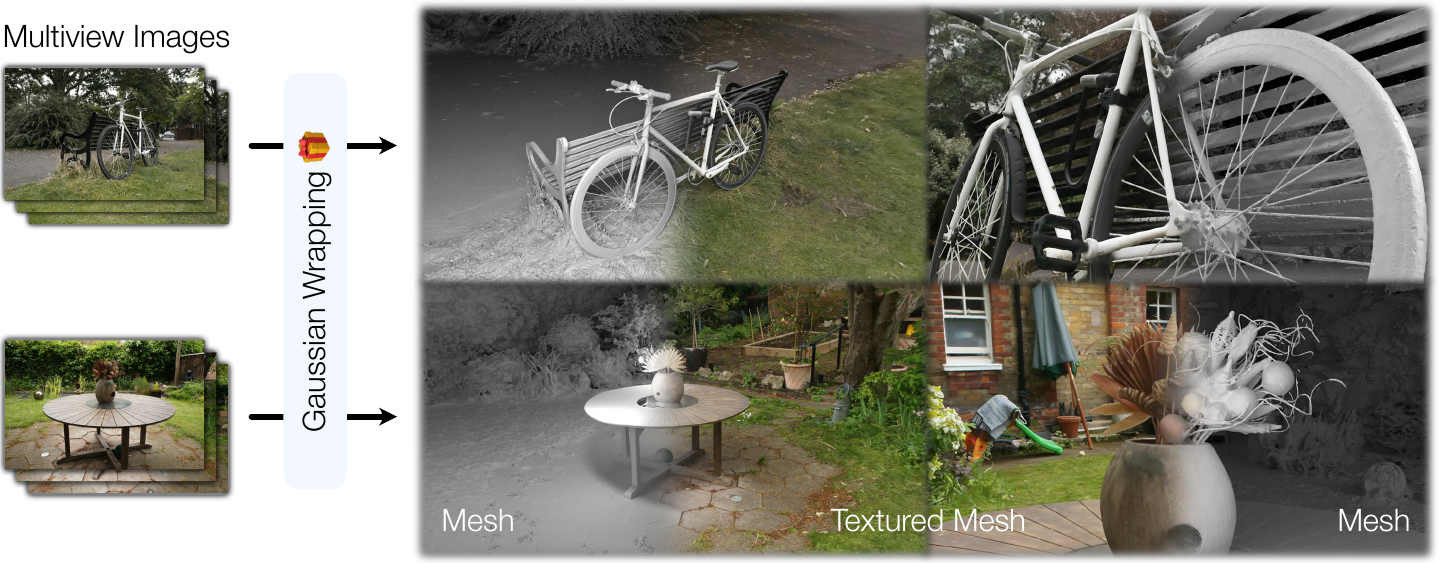}
\vspace{-0.5cm}
\caption{
\textbf{High-fidelity surface reconstruction from 3D Gaussian Splatting.}
Our method, \methodname, reconstructs watertight and textured surface meshes of full 3D scenes given multiview RGB images, by interpreting 3D Gaussians as stochastic oriented surface elements. This figure illustrates the resulting surface meshes, with and without the RGB texture. Our meshes represent the full scene---including background geometry and extremely thin structures such as bicycle spokes where existing methods fail---with a \textit{significantly} more compact representation than concurrent works~\cite{Zhang2026GeometryGrounded,chen2024pgsr}.
\vspace*{-10mm}}
\label{fig:teaser}
\end{figure*}

\begin{abstract}
3D Gaussian Splatting (3DGS) has revolutionized fast novel view synthesis, yet its opacity-based formulation makes \textit{surface extraction} fundamentally difficult.
Unlike implicit methods built on Signed Distance Fields or occupancy, 3DGS lacks a global geometric field, forcing existing approaches to resort to heuristics such as TSDF fusion of blended depth maps.
Inspired by the \textit{Objects as Volumes} framework~\cite{miller2024objectsasvolumes}, we derive a principled occupancy field for Gaussian Splatting and show how it can be used to extract highly accurate watertight meshes of complex scenes. Our key contribution is to introduce a learnable oriented normal at each Gaussian element and to define an adapted attenuation formulation, which leads to closed-form expressions for both the normal and occupancy fields at arbitrary locations in space. We further introduce a novel consistency loss and a dedicated densification strategy to enforce Gaussians to \textit{wrap} the entire surface by closing geometric holes, ensuring a complete shell of oriented primitives. We modify the differentiable rasterizer to output depth as an isosurface of our continuous model, and introduce \textbf{\meshingmethodname} for Region-of-Interest meshing at arbitrary resolution.
We additionally expose fundamental biases in standard surface evaluation protocols and propose two more rigorous alternatives. Overall, our method
\textbf{\methodname{}} sets a new state-of-the-art on DTU and Tanks and Temples, producing complete, watertight meshes at a fraction of the size of concurrent work---recovering thin structures such as the notoriously elusive bicycle spokes. Our project page is available \href{https://diego1401.github.io/BlobsToSpokesWebsite/index.html}{here}.
\keywords{3D Gaussian Splatting \and Surface Reconstruction \and Meshes}
\end{abstract}

\section{Introduction}

Reconstructing high-quality 3D surfaces from a set of 2D images is a key problem in computer vision. While recent advances in neural rendering, such as NeRF~\cite{mildenhall2020nerf} and 3DGS~\cite{kerbl3Dgaussians}, have shown impressive results on the novel view synthesis problem, extracting accurate and explicit geometry continues to be challenging. To bridge efficient volumetric rendering and surface reconstruction requires addressing their inherent differences: neural rendering is particularly suited to soft, semi-transparent representations, whereas surface reconstruction demands hard, well-defined boundaries to extract watertight geometry.

Current approaches to this inverse problem generally fall into two categories. Implicit representations, such as Signed Distance Functions (SDFs) and occupancy fields, excel at representing continuous, topologically consistent surfaces. However, they are often computationally expensive to train and tend to produce over-smoothed results, as their global solvers struggle to capture high-frequency details~\cite{wang2021neus,li2023neuralangelo}. On the other hand, explicit particle-based methods (i.e, 3DGS) exhibit real-time rendering and high-fidelity visual quality by representing scenes as discrete primitives. While they can capture fine details, there is no straightforward way to recover an ordered structure, i.e, a mesh, from such representations. Recent hybrid approaches attempt to close the gap by regularizing 3DGS~\cite{guedon2024sugar,guedon2025milo,yu2024gaussian,ren20252dgs,radl2025sof,Zhang2026GeometryGrounded}, but often resort to ad-hoc density thresholds or compromise the rendering speed and quality that make 3DGS attractive.

A core limitation of extracting surfaces from 3DGS is the interpretation of the primitive itself. Standard approaches treat Gaussians as symmetric ``blobs'' of mass or density. This assumption conflicts with the nature of an orientable surface, which is an asymmetric boundary separating empty and occupied space. By modeling surface points with symmetric primitives, existing methods bias reconstruction, and make surface extraction significantly more challenging. 

In this work, we address this limitation by introducing \textbf{\methodname}, a novel framework that reinterprets the role of the Gaussian primitive inspired by Objects as Volumes~\cite{miller2024objectsasvolumes}. Rather than treating Gaussians as symmetric clouds, we view them as \textit{oriented probabilistic surface elements} that represent the underlying geometry only within an oriented half-space. In this view, the Gaussian models the density decay on the outward-facing side of the boundary, while the inward-facing side is considered as fully occupied. This formulation allows us to derive a robust ``wrapping'' strategy: we identify geometric gaps where our oriented surface assumption is violated and densify the representation to obtain a watertight shell of oriented primitives. Our framework allows the derivation of closed-form geometric quantities---namely occupancy, its complement (\textit{vacancy}), and the normal field---enabling reconstruction of thin structures without artifacts or missing geometry. Links to extracted meshes, code and more are available \href{https://diego1401.github.io/BlobsToSpokesWebsite/index.html}{here}. Our key contributions are as follows.

\begin{itemize}
    \item We introduce \textbf{Oriented-Gaussians} and their associated training strategy in the multi-view setting.
    \item We derive a theoretical connection between 3DGS and implicit surface reconstruction by formulating Gaussians as oriented surface elements, inspired by \textit{Objects as Volumes} \cite{miller2024objectsasvolumes}. Importantly, this leads to closed-form expressions for both normal and occupancy fields at arbitrary locations without any additional learnable parameters.
    \item We propose \textbf{\meshingmethodname}, a mesh extraction procedure that leverages the derived Gaussian fields to produce high-quality, watertight meshes at controllable resolution, enabling recovery of extremely thin structures such as bicycle spokes~(\cref{fig:teaser}).
\end{itemize}

\begin{figure}[t]
\centering

\begin{subfigure}[b]{0.49\linewidth}
    \centering
    \includegraphics[width=0.49\linewidth]{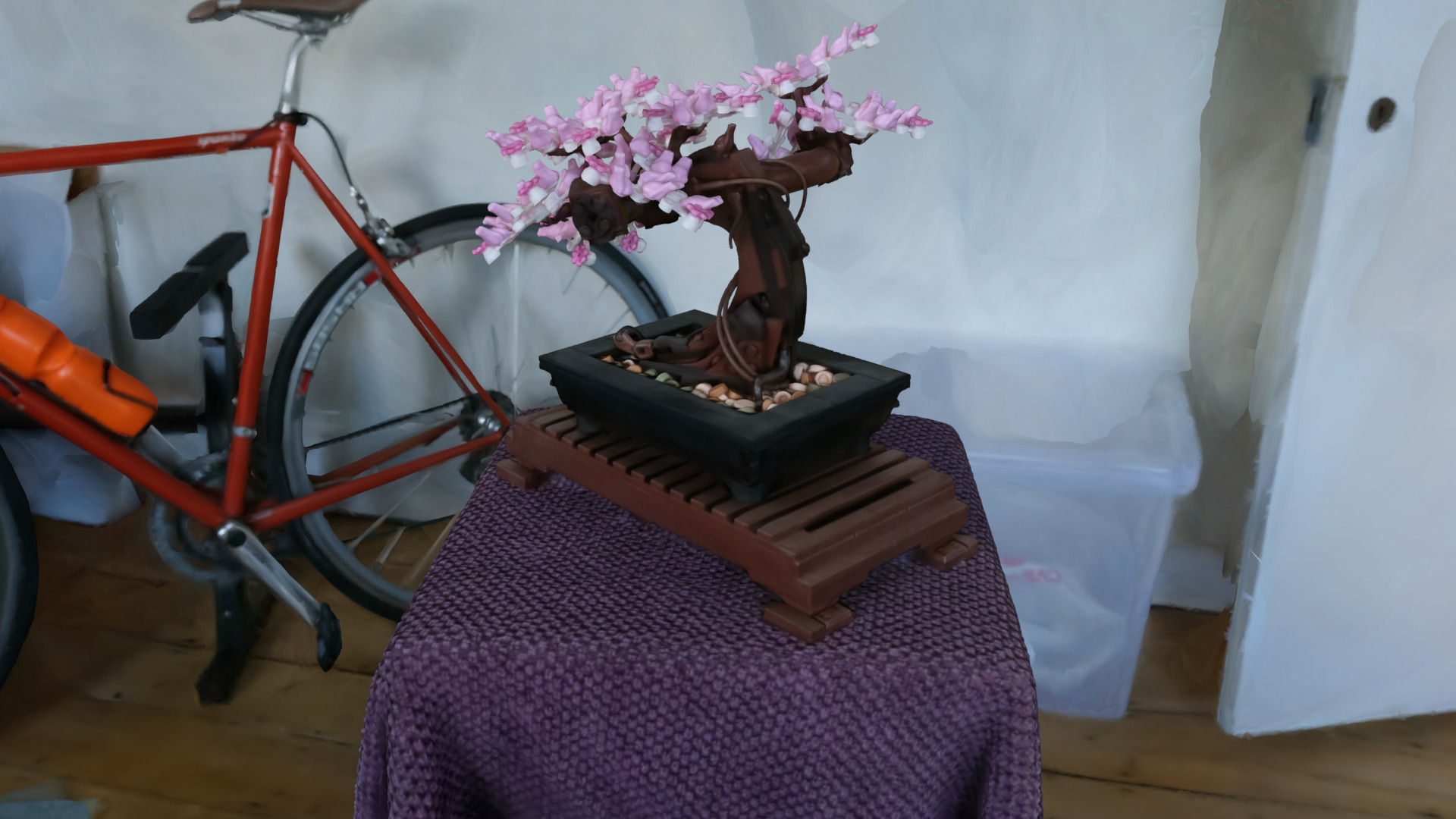}
    \includegraphics[width=0.49\linewidth]{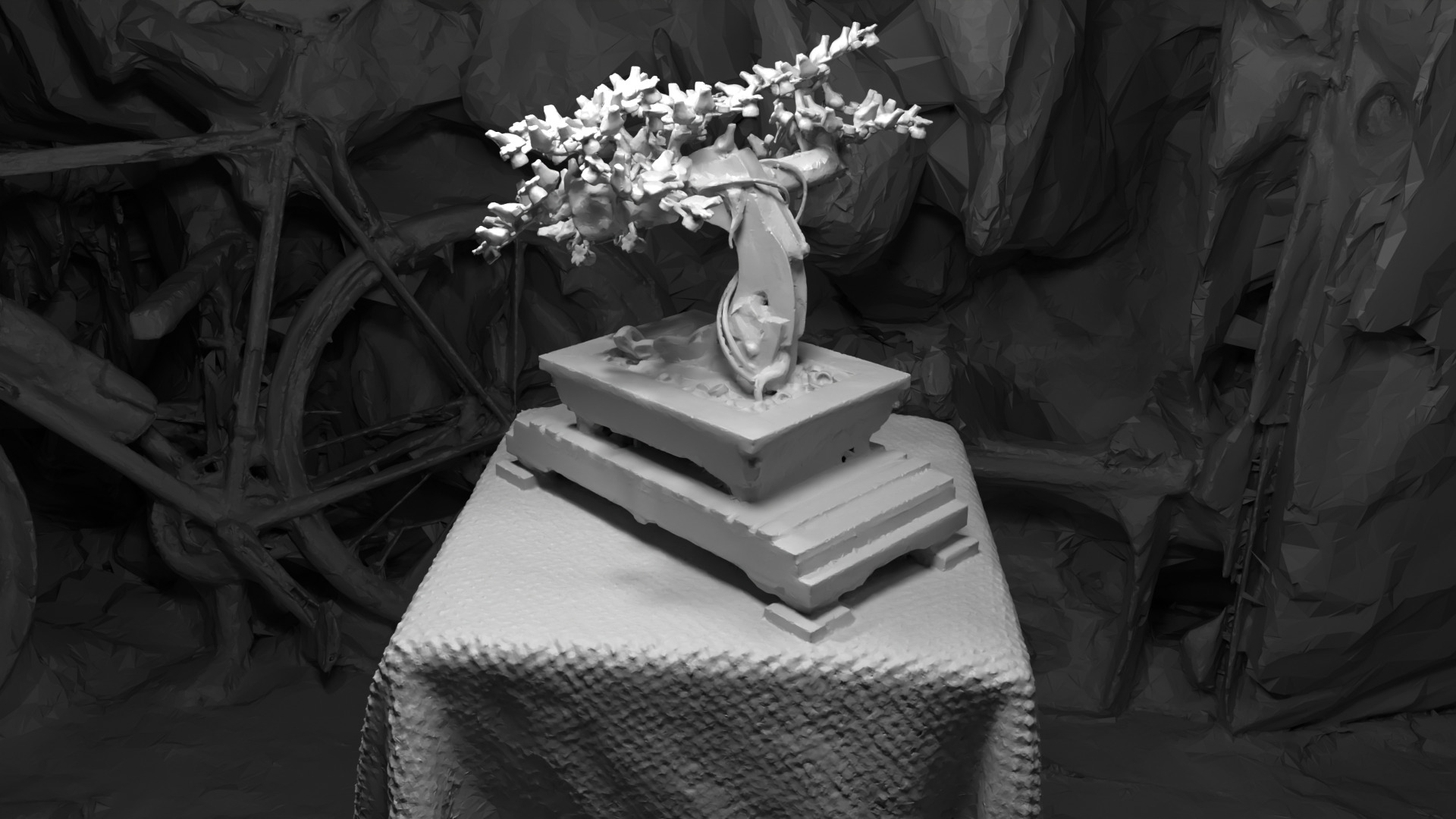}
    
\end{subfigure}
\begin{subfigure}[b]{0.49\linewidth}
    \centering
    \includegraphics[width=0.49\linewidth]{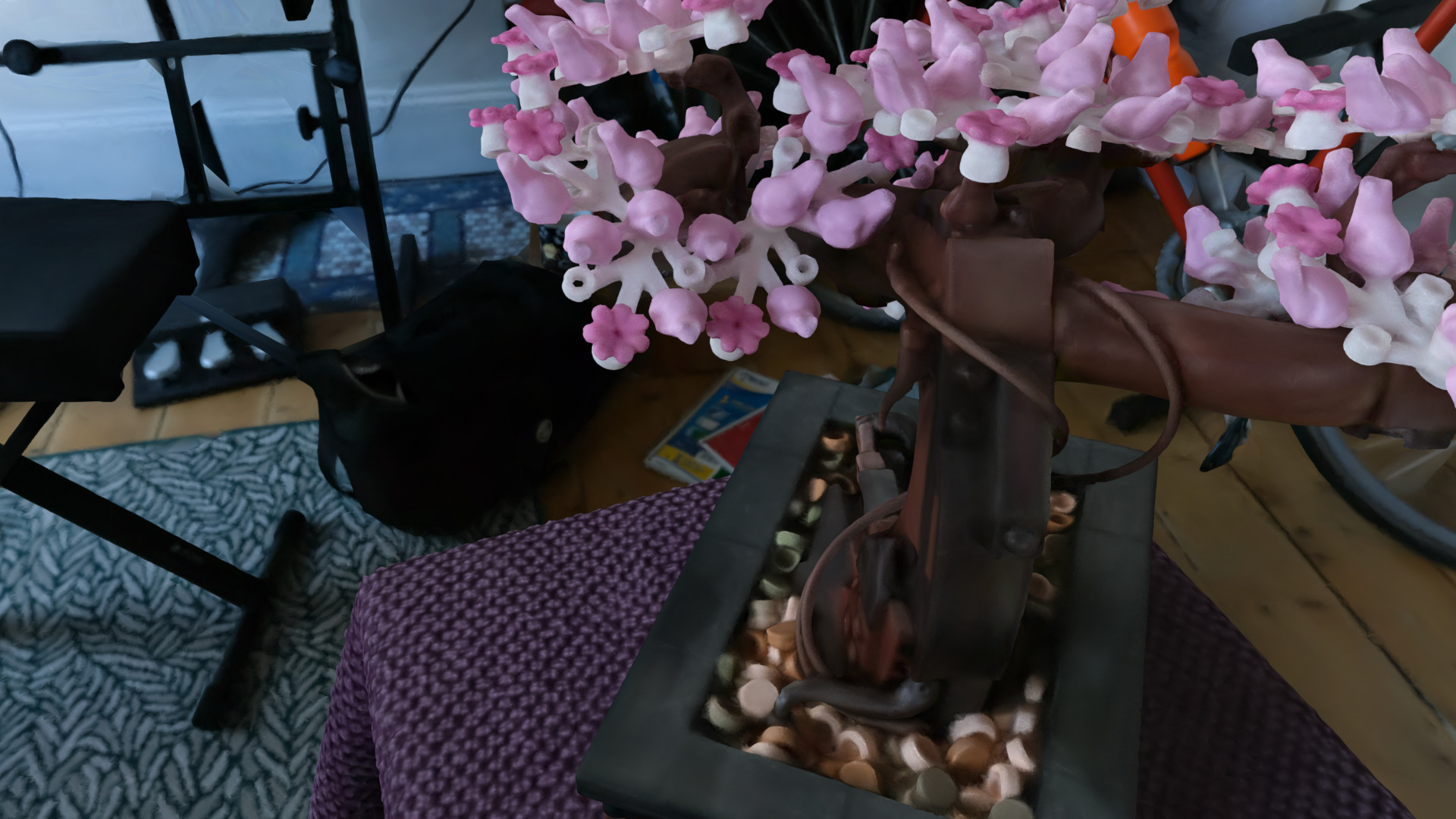}
    \includegraphics[width=0.49\linewidth]{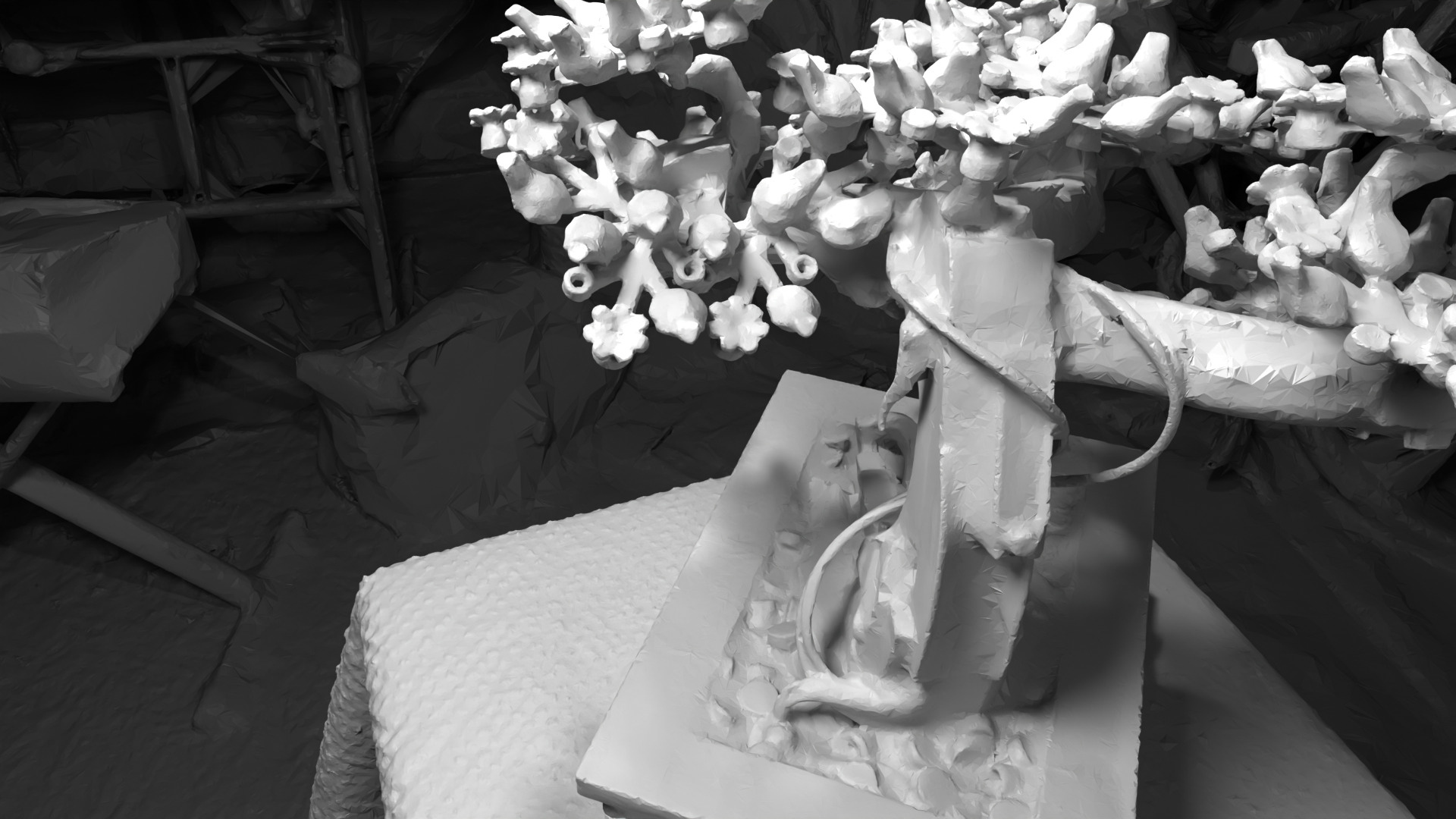}
    
\end{subfigure}
\begin{subfigure}[b]{0.49\linewidth}
    \centering
    \includegraphics[width=0.49\linewidth]{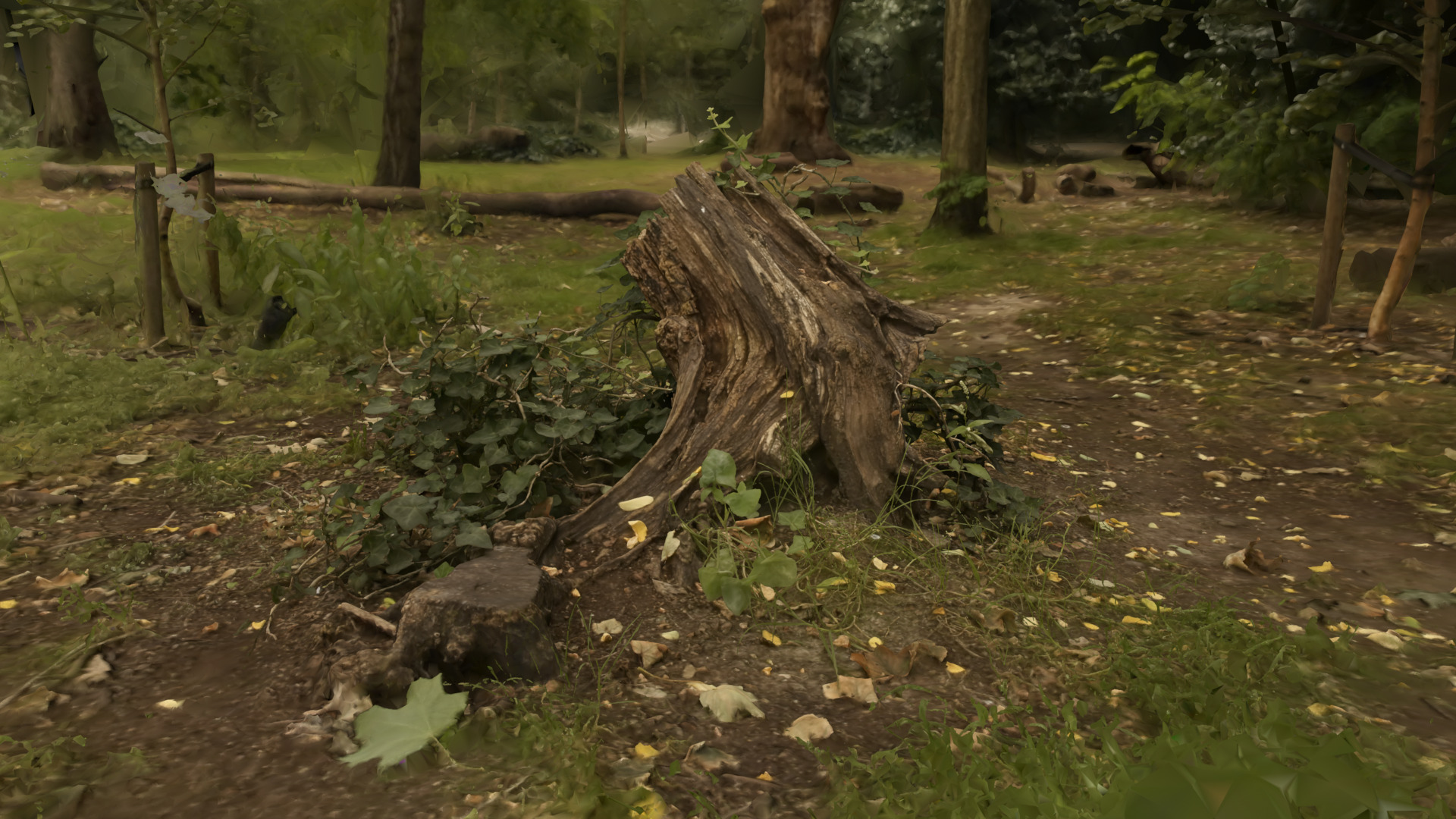}
    \includegraphics[width=0.49\linewidth]{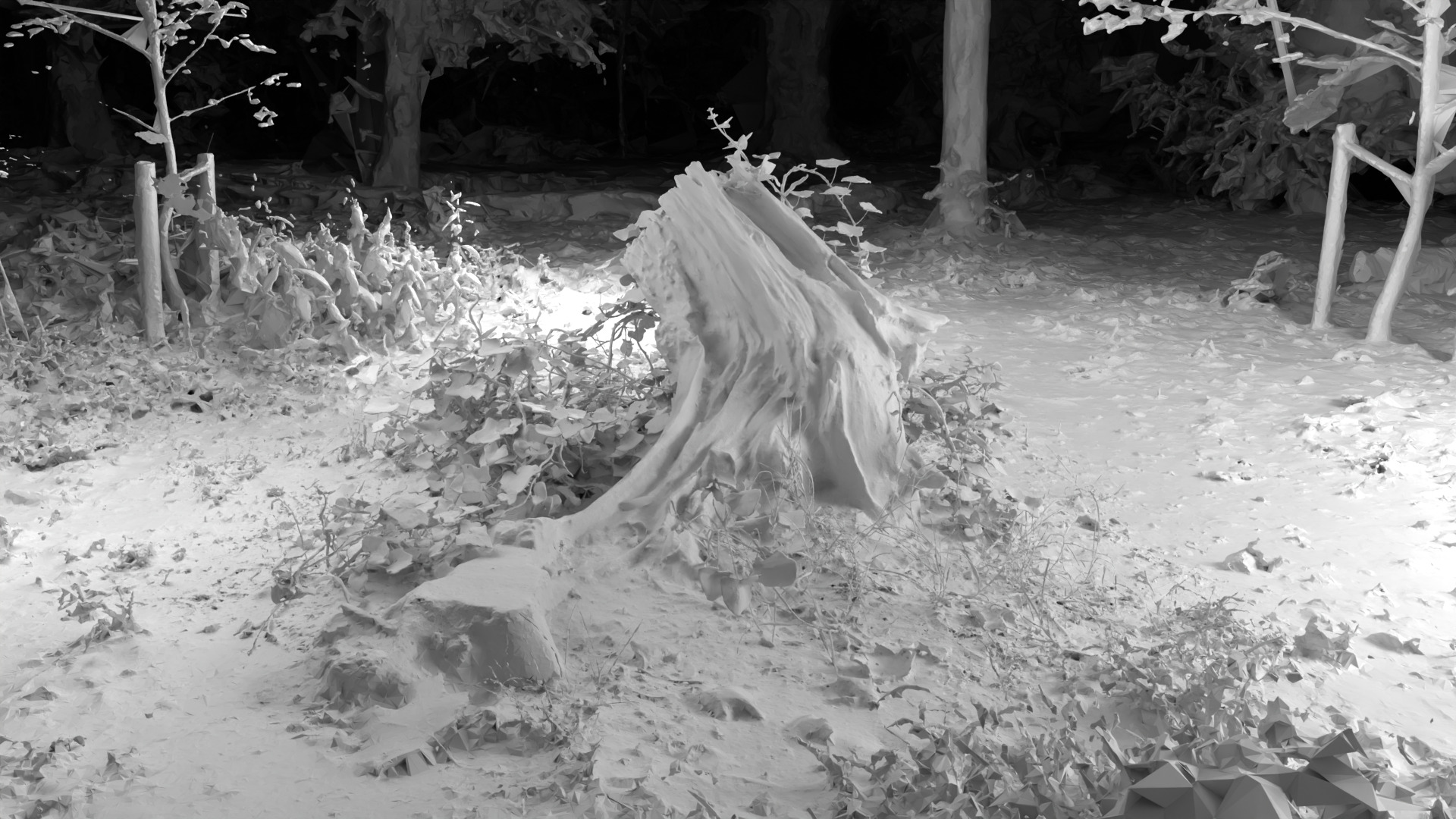}
    \caption{Mesh w/wo texture}
\end{subfigure}
\begin{subfigure}[b]{0.49\linewidth}
    \centering
    \includegraphics[width=0.49\linewidth]{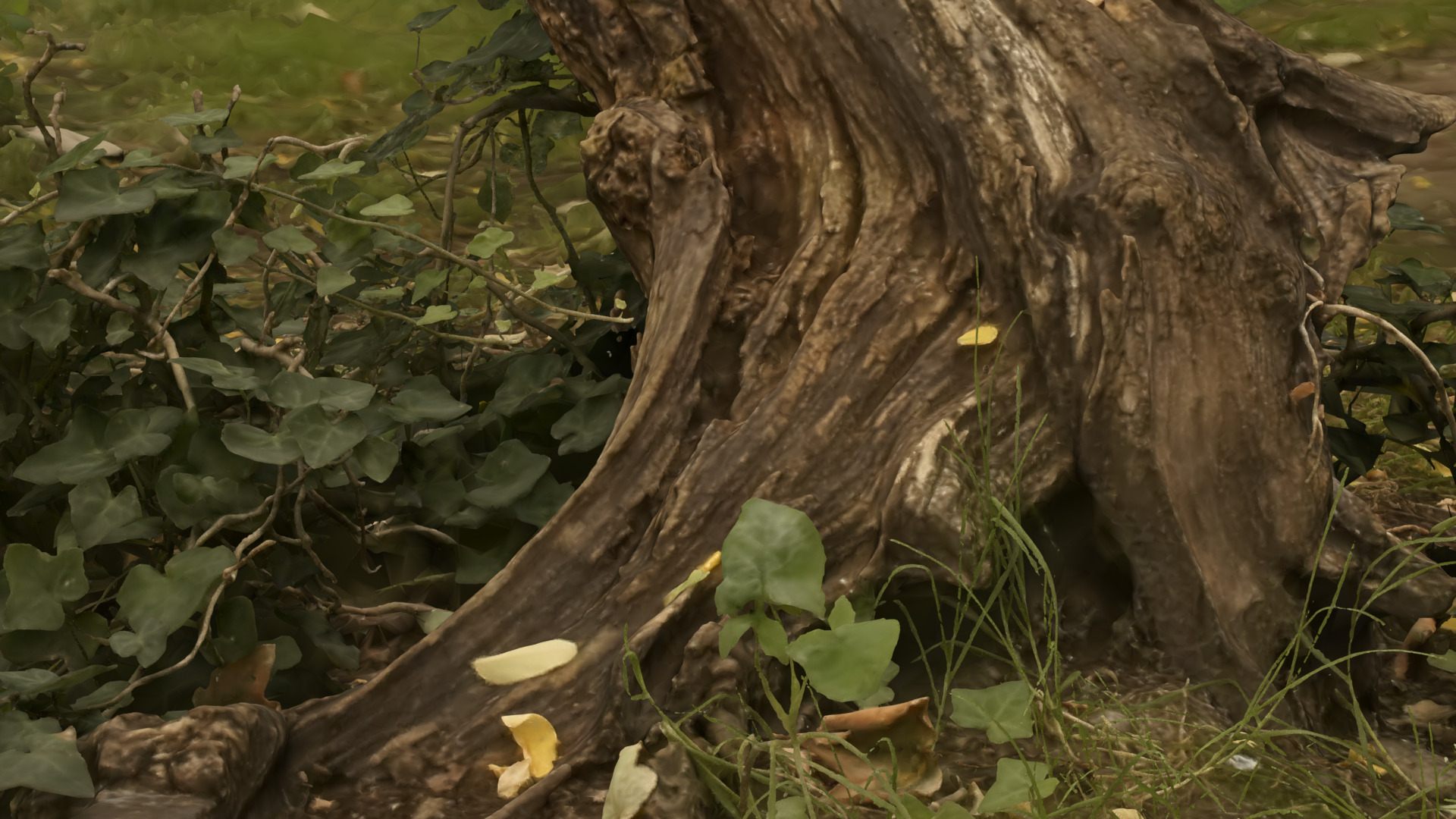}
    \includegraphics[width=0.49\linewidth]{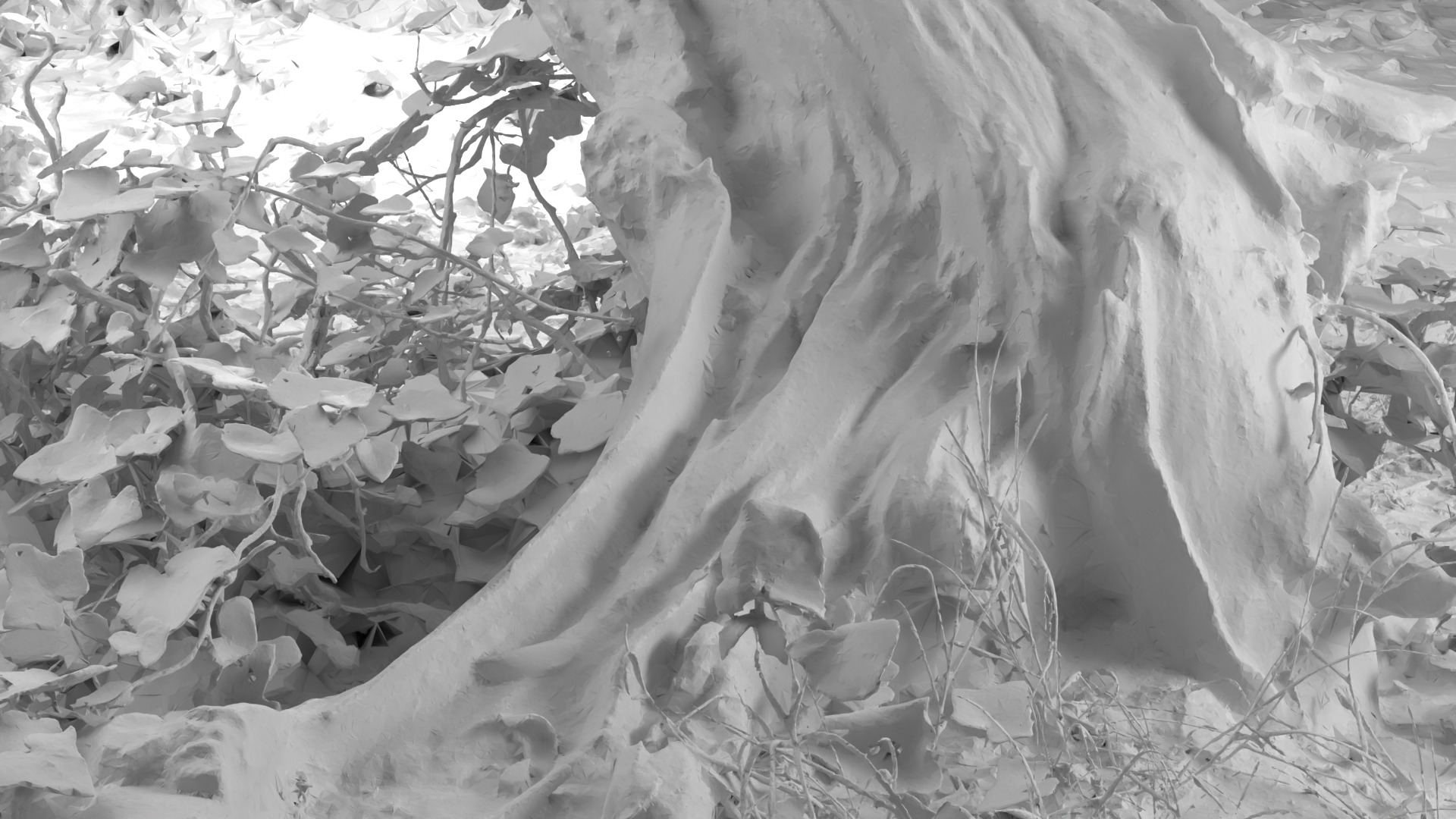}
    \caption{Close-up view of Mesh w/wo texture}
\end{subfigure}

\caption{\textbf{Examples of textured meshes reconstructed with our method.} While being watertight and lighter than meshes reconstructed by state-of-the-art methods, 
our approach is able to reconstruct extremely fine details such as a few blades of grass---as can be seen in the close-up view of the stump scene~\cite{barron22mipnerf360}.\vspace{-3mm}
}
\label{fig:showcase}
\end{figure}
\section{Related work}

\subsubsection{Novel View Synthesis.}
NeRF~\cite{mildenhall2020nerf} models scenes as continuous radiance fields whose image formation relies on ray-marching through an attenuation field (also called \textit{density}) parameterized by an MLP.
Subsequent work improves anti-aliasing~\cite{barron2021mip, barron22mipnerf360} and training efficiency via hash encodings~\cite{mueller2022instantngp} or explicit volumetric structures~\cite{yu2021plenoxels, chen2022tensorf}.
3D Gaussian Splatting~\cite{kerbl3Dgaussians} redefined this landscape by replacing ray-marching with rasterization of learned Gaussian primitives, achieving real-time, high-quality rendering.
Further extensions address aliasing~\cite{yu2024mip}, scalability~\cite{mallick2024taming, kheradmand20243d}, and rendering fidelity.
Despite impressive visual quality, none of these methods target explicit surface mesh extraction.

\vspace*{-4mm}
\subsubsection{Surface Reconstruction from Images.}
Implicit approaches coupling volume rendering with signed distance functions (NeuS~\cite{wang2021neus}, VolSDF~\cite{yariv2021volume_volsdf}, Neuralangelo~\cite{li2023neuralangelo}) convert a learnable SDF into a view-dependent attenuation field and yield smooth, topologically consistent surfaces but require very long optimization times. Mesh Baking~\cite{yariv2023bakedsdf, Reiser2024SIGGRAPH} approaches employ a two-stage pipeline that first trains a volumetric representation, then extracts a mesh to bake rendering onto it; this line of work however focuses on rendering and does not measure the geometric loyalty of the produced meshes.
Discrete mesh methods~\cite{munkberg22nvdiffrec, shen23flexicubes} produce strong results on isolated objects but do not scale to entire scenes.

In this regard, the efficiency of Gaussian Splatting motivates a growing family of surface-extraction approaches.
Several jointly optimize Gaussians with an auxiliary neural implicit field~\cite{chen2023neusg, lyu20243dgsr, yu2024gsdf, li2025gaussianudf, zhang2024gspull}, but the dual optimization limits scalability and the implicit branch remains bounded by MLP capacity.
Others regularize Gaussians directly: 2DGS~\cite{huang20242d} replaces 3D primitives with lower-dimensional ones, RaDe-GS~\cite{zhang2024rade}, PGSR~\cite{chen2024pgsr}, Quadratic Gaussian Splatting~\cite{zhang2024quadratic}, and VCR-GauS~\cite{chen2024vcr} explore alternative rendering strategies or multi-view consistency constraints, while SuGaR~\cite{guedon2024sugar} and GOF~\cite{yu2024gaussian} propose specialized extraction procedures.
A shared limitation is the reliance on heuristic depth definitions---typically TSDF fusion of blended depth maps---without a principled link between the Gaussian representation and a well-defined geometric field.

\textit{Objects as Volumes}~\cite{miller2024objectsasvolumes} establishes a closed-form relationship between vacancy and view-dependent attenuation under reciprocal exponential transport, grounding the theory behind Neural SDF methods~\cite{wang2021neus, yariv2021volume_volsdf}.
Since the 3DGS image formation model is not attenuation-based, this result does not apply directly---though Celarek~\etal~\cite{celarek2025does} show 3DGS approximates volumetric rendering with view-dependent attenuation.
Concurrent work GGGS~\cite{Zhang2026GeometryGrounded} builds upon~\cite{miller2024objectsasvolumes} and exploits a continuous transmittance derived from Gaussians for improved depth estimation. However, their formulation suffers from surface erosion and fails to capture fine geometric details. Our method addresses this limitation by introducing oriented Gaussians that yield a reciprocal attenuation, from which closed-form vacancy and normal fields naturally follow---enabling principled, high-fidelity mesh extraction and limiting erosion.

\vspace*{-3mm}
\subsubsection{Predicting Normals from Radiance Fields.} 

Estimating normals from radiance fields is established practice in inverse rendering, where accurate normal maps help disentangle material properties~\cite{liang2024gs, jin2023tensoir, mai2023neural, gomez2024rrm, verbin2022refnerf}.
In contrast, we predict normals not for appearance decomposition but to construct an oriented surface---directly enabling robust surface reconstruction.

\vspace*{-3mm}

\subsubsection{Voronoi \& Delaunay-based Methods.} 
Voronoi diagrams and their dual Delaunay triangulations were first utilized as backbones for surface reconstruction with provable guarantees~\cite{aurenhammer1991voronoi, amenta_new_1998, amenta_power_2001, dey_tight_2003}, and later integrated into learning-based pipelines~\cite{sulzer2021scalable, maruani_voromesh_2023, maruani_ponq_2024, binninger2025tetweave, son2024dmesh, son2025dmesh++, zhang_high-fidelity_2025}.  Within neural rendering, GOF~\cite{yu2024gaussian} and MILo~\cite{guedon2025milo} build Delaunay triangulations over Gaussian-derived vertices, while other methods~\cite{govindarajan_radiant_2025, mai_radiance_meshes_2025,held2025meshsplatting} integrate them into ray-tracing pipelines for view synthesis rather than reconstruction. In contrast, our \meshingmethodname incorporates ideas from modern reconstruction pipelines~\cite{zhang_high-fidelity_2025} to generate Delaunay-based watertight meshes at any desired resolution.

\section{Background and Motivation}
\label{sec:background}

We focus on \textit{fully opaque objects}, where the scene is binary: each point is either inside or outside the object. Our goal is to derive a \textit{principled occupancy field} from 3D Gaussian Splatting, \ie, a view-independent scalar field encoding scene geometry, from which high-quality surface meshes can be extracted~(\cref{fig:showcase}).

\subsubsection{Objects as Volumes.} 
Reconstructing well-defined surfaces requires a properly defined geometric field such as an occupancy function $\calO:\IR^3\rightarrow[0,1]$, yet volumetric rendering---the backbone of both NeRF~\cite{mildenhall2020nerf} and 3DGS~\cite{kerbl3Dgaussians}---operates on continuous, semi-transparent volumes rather than on sharp boundaries.
\textit{Objects as Volumes}~(OaV)~\cite{miller2024objectsasvolumes} bridges this gap by interpreting the scene as a stochastic surface: the occupied region is modeled as a random set and the occupancy $\calO$ as the probability of a point to be occupied. Standard volumetric rendering quantities then emerge as expectations over this random geometry.

Specifically, under the assumption of exponential light transport, this framework grounds the theoretical justification of Neural SDFs~\cite{wang2021neus, yariv2021volume_volsdf}, and links the \emph{vacancy} $v(\rayx) = 1 - \calO(\rayx)$ (the probability that $\rayx\in\IR^3$ is unoccupied) to the attenuation (or density) $\sigma: \IR^3 \times \calS^2\rightarrow\IR_+$ of standard volumetric ray-marching:
\begin{equation}
    \forall (\rayx, \rayw)\in \IR^3 \times \calS^2, \quad
    \sigma(\rayx, \rayw) = |\rayw \cdot \nabla \log v(\rayx)| \> ,
    \label{objects_as_volumes}
\end{equation}
where attenuation is assumed to be \emph{reciprocal}, \ie $\sigma(\rayx, \rayw) = \sigma(\rayx, -\rayw)$.
This equation directly connects a geometric field, the vacancy, to the attenuation driving volumetric rendering.
It is the cornerstone of our approach: if Gaussian Splatting can be cast as an attenuation-based model satisfying reciprocity, then Eq.~\ref{objects_as_volumes} yields an explicit occupancy field and, in turn, a principled surface.

\subsubsection{Gaussian Splatting.} 3D Gaussian Splatting~\cite{kerbl3Dgaussians} represents a scene as $N$ Gaussian primitives, each defined by
$G_i(\rayx) = \alpha_i \exp ( -\frac{1}{2} {(\rayx - \mean_i)}^T \Sigma_i^{-1} (\rayx - \mean_i))$,
with opacity $\alpha_i \in [0,1]$, mean $\mean_i\in\IR^3$, and precision matrix $\Sigma_i^{-1}\in\posdefm$.
Unlike NeRF~\cite{mildenhall2020nerf}, which ray-marches through a continuous attenuation field $\sigma$, 3DGS relies on \emph{opacity}: it renders the pixel $p$ by projecting each primitive onto the image plane and alpha-compositing contributions in depth order:
\begin{equation}
    C(p) = \sum_{i} c_i\, \alpha_i G_i^{*}(p) \prod_{j < i}\!\bigl(1 - \alpha_j G_j^{*}(p)\bigr) \> ,
    \label{eq:3dgs_rendering}
\end{equation}
where $G_i^{*}$ is either a projected 2D Gaussian as in 3DGS~\cite{kerbl3Dgaussians}, or the evaluation of the 3D Gaussian at the point of maximum contribution along the ray~\cite{yu2024gaussian,zhang2024rade,Zhang2026GeometryGrounded}. A key consequence is that each Gaussian's contribution depends only on its projected value at the pixel—not on the ray path length through the primitive. Gaussian opacities are therefore bounded in $[0,1]$ and cannot be identified with the potentially unbounded attenuation coefficient $\sigma$. As a result, the OaV framework cannot be applied to 3DGS directly.

\subsubsection{Finding an attenuation model equivalent to 3DGS.} Recent work~\cite{celarek2025does} shows that the 3DGS image formation model can nonetheless be reinterpreted as \textit{``a volumetric rendering model with view-dependent extinction''}, assuming primitives do not overlap in 3D. To be precise, we show in the appendix that, \textit{assuming
statistical independence and non-overlapping Gaussian primitives}, the image formation model of a Gaussian Splatting representation—relying on
alpha-blending with opacity—is equivalent to a ray-marching volumetric rendering model with the following continuous attenuation:
\[
\sigma(\rayx, \rayw) = \sum_{i=1}^N \max\left(
            0,
            -\rayw \cdot \nabla \log (1-G_i(\rayx))
        \right) \> ,
\]
for any ray origin $\rayo\in\IR^3$ in empty space, $\rayw\in\calS^2$, and $t\in [0,+\infty)$. 
Contrary to previous work~\cite{Zhang2026GeometryGrounded}, this attenuation factor enforces geometric properties necessary for accurate surface extraction; please refer to the appendix for more details.
However, this expression is not reciprocal in general, \ie $\sigma(\rayx, \rayw) \neq \sigma(\rayx, -\rayw)$. Reciprocity is a necessary condition for applying results from OaV, thus preventing us from deriving a closed-form vacancy expression.

\subsubsection{A reciprocal attenuation model.} To restore reciprocity, we propose to \textit{orient} Gaussians: we equip each primitive with a learnable, oriented normal vector. This results in an explicit attenuation-based formulation of 3DGS that satisfies the requirements of OaV, yielding both more accurate depth rendering and an explicit geometric field for surface extraction. This attenuation-based formulation also induces a volumetric rendering model equivalent to the image formation model of a Gaussian Splatting representation---\textit{under the additional assumption that Gaussians properly ``wrap" the scene}.
Specifically, we define the following oriented attenuation:
\hypbox[]{
\begin{definition}[Oriented Gaussian]\label{definition:oriented_gaussians}
    We assign an oriented normal vector parameter $n_i\in\calS^2$ to each Gaussian $G_i$, and define its oriented attenuation coefficient as
    \begin{equation}
        \orientedsigma_i(\rayx, \rayw) = \charx |\rayw \cdot \nabla\log (1-G_i(\rayx))| \> .
        \label{eq:oriented_attenuation}
    \end{equation}
\end{definition}}
We detail in the appendix---both theoretically and qualitatively---how this formulation not only satisfies the geometric properties required for applying OaV, but is also a valid approximation of the Gaussian Splatting image formation model under specific assumptions. This attenuation-based model constitutes the basis of our approach, and allows us to derive the method below. 

While Geometry Field Splatting~\cite{jiang2025geometry} leverages OaV to define a geometric field out of particles, it applies its analysis to Gaussian surfels only and relies on TSDF for mesh extraction, limiting its capability to recover details and scale to full scenes with background geometry. 
Geometry-Grounded Gaussian Splatting~\cite{Zhang2026GeometryGrounded} adopts a setting similar to our method, applying Equation~\ref{objects_as_volumes} to derive a continuous transmittance and improve depth estimation. However, it does not guarantee the multi-view consistency of the depth, which is essential for a well-defined occupancy field as illustrated in Section~\ref{sec:experiments}.
In contrast, we restore reciprocity while enforcing accurate and multiview-consistent geometry by \emph{orienting} each Gaussian with a normal vector.
We provide in the appendix an illustration of how the non-oriented attenuation from~\cite{Zhang2026GeometryGrounded} fails to produce multi-view consistent depth maps and complete surfaces.

In the following section, we expose our theoretical results and describe how our formulation can be used for improving surface extraction from Gaussian Splatting representations.
\textbf{Please refer to the appendix for extensive proofs and derivations motivating our results.}
\section{Method}\label{sec:method}

\begin{figure}[t]
\centering
\includegraphics[width=\linewidth]{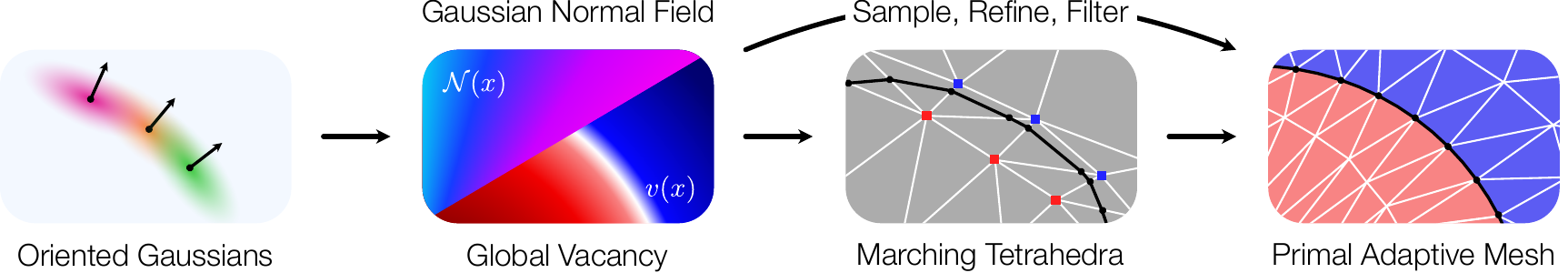}
\vspace*{-2mm}
\caption{
\textbf{Wrapping the surface with Oriented Gaussians.} Our Gaussian Wrapping setting allows to evaluate the normal and vacancy fields of the scene. We first leverage the vacancy on our pivot-based marching tetrahedra meshing approach. We then leverage both the vacancy and normal fields in our Primal Adaptive Meshing.
\vspace*{-5mm}}
\label{fig:pipeline}
\end{figure}

\paragraph{Method Overview:} Our method operates as follows: first, as mentioned above, we endow each Gaussian with a learnable oriented normal parameter (note that these are \textit{the only} additional learnable parameters of our framework), and introduce the corresponding attenuation coefficient, as in Definition \ref{definition:oriented_gaussians}. Based on this construction and thanks to the Objects as Volumes formalism \cite{miller2024objectsasvolumes}, we derive closed-form expressions for both the normal and occupancy fields for arbitrary locations in space, without any additional learnable parameters. Second, we propose a training strategy, which learns the oriented normals through an image-based consistency loss, using an adapted rasterizer and coupled with a new normal-aware densification strategy. Lastly, we exploit our normal field and occupancy estimates in a novel dedicated mesh-extraction approach, which both respects the underlying geometry, while being adaptive to details. For a pipeline figure presenting our method, please see~\cref{fig:pipeline}.

\vspace{-0.3cm}
\subsection{Gaussian Wrapping}\label{sec:gaussian_wrapping}

Our core intuition is simple: for accurate surface reconstruction, Gaussians must form a sealed boundary around the object.
We achieve this by endowing each Gaussian with a learnable oriented normal $n_i$ and enforcing a \emph{wrapping} behavior: visible Gaussians near the surface should point outward, away from the occupied region and toward the camera.
This wrapping creates a continuous shell of oriented primitives even around extremely thin structures---such as bicycle spokes---for which, standard methods only recover the \textit{appearance} with a few unoriented primitives, but cannot recover the geometry.

More fundamentally, we show in the appendix that when Gaussians properly wrap the scene, the oriented attenuation of~\cref{definition:oriented_gaussians} becomes a valid approximation of the Gaussian Splatting image formation model, and the \textit{Objects as Volumes} framework~\cite{miller2024objectsasvolumes} directly yields an explicit vacancy field $v$. 

We promote wrapping in the scene through two complementary mechanisms.

\paragraph{Normal alignment loss.} Rendering depth and normals $n_i$ with the Splatting rasterizer~\cite{zhang2024rade,yu2024gaussian} yields, respectively, an estimate of the surface depth and the expected orientation of Gaussians near the surface~\cite{miller2024objectsasvolumes}. As a result, we enforce the alignment between the normals $n_i$ and the surface during training via:
\begin{equation}
    \mathcal{L}_{\text{N}} = \sum_{p}
    1 - N(p) \cdot \nabla D(p) \> ,
    \label{eq:normal_field_regularization}
\end{equation}
where $N(p)$ is the rendered oriented normal at pixel $p$ and $\nabla D(p)$ the image-space gradient of the rendered depth. 

Specifically, we modify the differentiable CUDA rasterizer introduced in~\cite{Zhang2026GeometryGrounded} to match our framework, and render depth as the exact $0.5$-isosurface of our geometric field.

\paragraph{Densification.} Gaps appear in the wrapping shell where Gaussians fail to cover all sides of a surface, causing $\calL_{\text{N}}(p)$ to rise locally. This provides a natural densification criterion: every $K$ iterations, we compute errors $\calL_{\text{N}}(p)$ across all training views and propagate them to individual Gaussians via their blending weights. High-error Gaussians are cloned with flipped normals, closing holes and reinforcing the shell. We illustrate this process in the appendix.

\paragraph{Loss Details.} During optimization, we combine our novel normal alignment objective $\calL_{\text{N}}$ in conjunction with the following set of losses: The standard photometric~\cite{kerbl3Dgaussians} loss $\calL_{\text{RGB}}$, depth-normal consistency $\calL_{\text{DN}}$~\cite{yu2024gaussian,zhang2024rade} and the multi-view loss $\mathcal{L}_{\text{MV}} = \lambda_{\text{pc}} \mathcal{L}_{\text{pc}} +  \lambda_{\text{gc}} \mathcal{L}_{\text{gc}}$, where the terms control the strength of the photometric and geometric consistency as defined in~\cite{Zhang2026GeometryGrounded,chen2024pgsr}. These have the following weights $\lambda_{\text{DN}}=0.05,\lambda_{\text{N}}=0.05, \lambda_{\text{pc}}=0.6, \lambda_{\text{gc}}=0.02$.
This results in the total loss 
\[
\calL=\calL_{\text{RGB}} +  \lambda_{\text{DN}}\calL_{\text{DN}} + \lambda_{\text{N}}\mathcal{L}_{\text{N}} + \mathcal{L}_{\text{MV}}
\]

\vspace{-0.1cm}
\subsection{Gaussian Vector and Normal Fields}

Assuming Gaussians properly wrap the scene, applying Eq.~\ref{objects_as_volumes} from \textit{Objects as Volumes}~\cite{miller2024objectsasvolumes} to our reciprocal attenuation (Eq.~\ref{eq:oriented_attenuation}) yields our main result.
\hypbox[]{
\begin{proposition}[Gaussian Vector and Normal Fields]\label{prop:vector_normal_field}
    We define the Gaussian vector field $V\colon\IR^3\to\IR^3$ of $N$ oriented Gaussians and its normalization $\calN\colon\IR^3\to\calS^2$ as:
    \begin{align}
        V(x) &:= \nabla\log v(x) = \sum_{i=1}^N \charxnorm \nabla\log (1-G_i(x)) \> \label{prop:vector_field} \\
        \calN(x) &:= \frac{V(x)}{\|V(x)\|} \>  \label{eq:normal_field}
    \end{align}
    
    $\calN(x)$ is well-defined and coincides with the true normal field of the expected stochastic surface of $M$ in a neighborhood of the surface.
    
\end{proposition}
}
These closed-form expressions link Gaussian parameters directly to surface geometry and form the basis of our mesh extraction pipeline (\cref{sec:applications}). Note that we can query these quantities explicitly by querying a neighborhood of Gaussians around the point $x$. Proofs are deferred to the appendix. 

\vspace{-0.1cm}
\subsection{Mesh Extraction}\label{sec:applications}

We now leverage the vacancy $v$ as an implicit function for mesh extraction. While direct integration of $V$ (Eq.~\ref{prop:vector_field}) along any camera ray is possible, Gaussians that do not contribute to rendering can remain hidden inside the geometry after optimization, violating the wrapping assumption and producing artifacts if intersected along a ray. In practice, we compute the vacancy by iterating over the set of all training camera rays $\mathcal{T}_c$ with
\begin{equation}
    v(\rayx) = \max_{(\rayo,\rayw)\in \mathcal{T}_c}\left\{
         \prod_{i=1}^N \left(1 - \Gstari(t)\right) : \rayx = \rayo + t\rayw, t>0
    \right\} \> ,
    \label{eq:lower_bound_vacancy}
\end{equation} 
where $\Gstari(t) = G_i(\rayo + \min(t, \tstari)\rayw)$; and $\tstari := \arg\max_{t\geq 0} G_i(\rayo + t\rayw)$. Intuitively, this computation corresponds to finding the ``most'' unobstructed ray reaching point $x$. This result \textit{formalizes the opacity-field meshing heuristics} empirically adopted in recent works~\cite{yu2024gaussian,Zhang2026GeometryGrounded,radl2025sof}. Strictly speaking, this expression is a lower bound on the true vacancy (see appendix for proof) and is inherently robust: since the products can only decrease along a ray, floating Gaussians hidden inside the geometry cannot inflate the vacancy estimate. 

\paragraph{Pivot-Based Marching Tetrahedra.} Under our oriented Gaussian assumption, the surface intersects each Gaussian between its center and the low-density side, parallel to the oriented plane. We therefore spawn two \emph{Delaunay pivots} per Gaussian: the center $\meani$ and $\meani + 3s_i n_i$, where $s_i = \|S_iR_i^Tn_i\|_2$ is the ellipsoid scaling along $n_i$. Occupancy values at each pivot are obtained using the vacancy lower bound (Eq.~\ref{eq:lower_bound_vacancy}); we then apply Marching Tetrahedra on the resulting Delaunay triangulation and refine vertices to the $0.5$-isosurface via binary search. This approach is conceptually similar to GOF~\cite{yu2024gaussian} and GGGS~\cite{Zhang2026GeometryGrounded}, yet contrasts with the 9 pivots per Gaussian required by prior works~\cite{yu2024gaussian,zhang2024rade,guedon2025milo,Zhang2026GeometryGrounded}, yielding substantially lighter, fully watertight meshes without sacrificing surface fidelity.

\paragraph{Primal Adaptive Meshing.} Although simple and efficient, the pivot-based formulation instantiates vertices directly from Gaussian primitives rather than from the underlying scene geometry. To leverage our global vacancy field $v$ and further improve mesh generation, we introduce an adaptive meshing framework inspired by primal Delaunay-based reconstruction methods~\cite{amenta_new_1998, dey2008delaunay,zhang_high-fidelity_2025, sulzer2021scalable, maruani_ponq_2024, son2024dmesh, son2025dmesh++}. The pipeline proceeds in four stages:

\begin{enumerate}
    \item \textbf{Vertices Initialization.} We initialize our vertices by sampling the marching tetrahedra mesh faces, weighting each face inversely proportional to its distance from the nearest training camera.

    \item \textbf{Isosurface Refinement.} Vertices are projected onto the $0.5$-isosurface of $v$ via an iterative Newton update:\;
    \begin{equation}\label{eq:newton_update}
        x_{i+1} = x_i + \frac{1}{2} (0.5 - v(x_i))\mathcal{N}(x_i),
    \end{equation}
    %
    \vspace*{1mm}
    \item \textbf{Filtering.} Vertices $x$ with $|0.5 - v(x)| > \epsilon$ are considered outliers and removed. Stages 1--3 iterate until no points are removed.

    \item \textbf{Delaunay.} We compute the Delaunay tetrahedralization of the remaining vertices and classify each tetrahedron as inside or outside, following~\cite{zhang_high-fidelity_2025}, based on the vacancy value of randomly sampled interior points.
    
    \item \textbf{Mesh extraction.} The final surface mesh is obtained by extracting triangle faces that separate inside and outside tetrahedra (see \cref{fig:pipeline}).
\end{enumerate}

\begin{figure}[t]
\centering

\begin{subfigure}[b]{0.49\linewidth}
    \centering
    \includegraphics[width=0.49\linewidth]{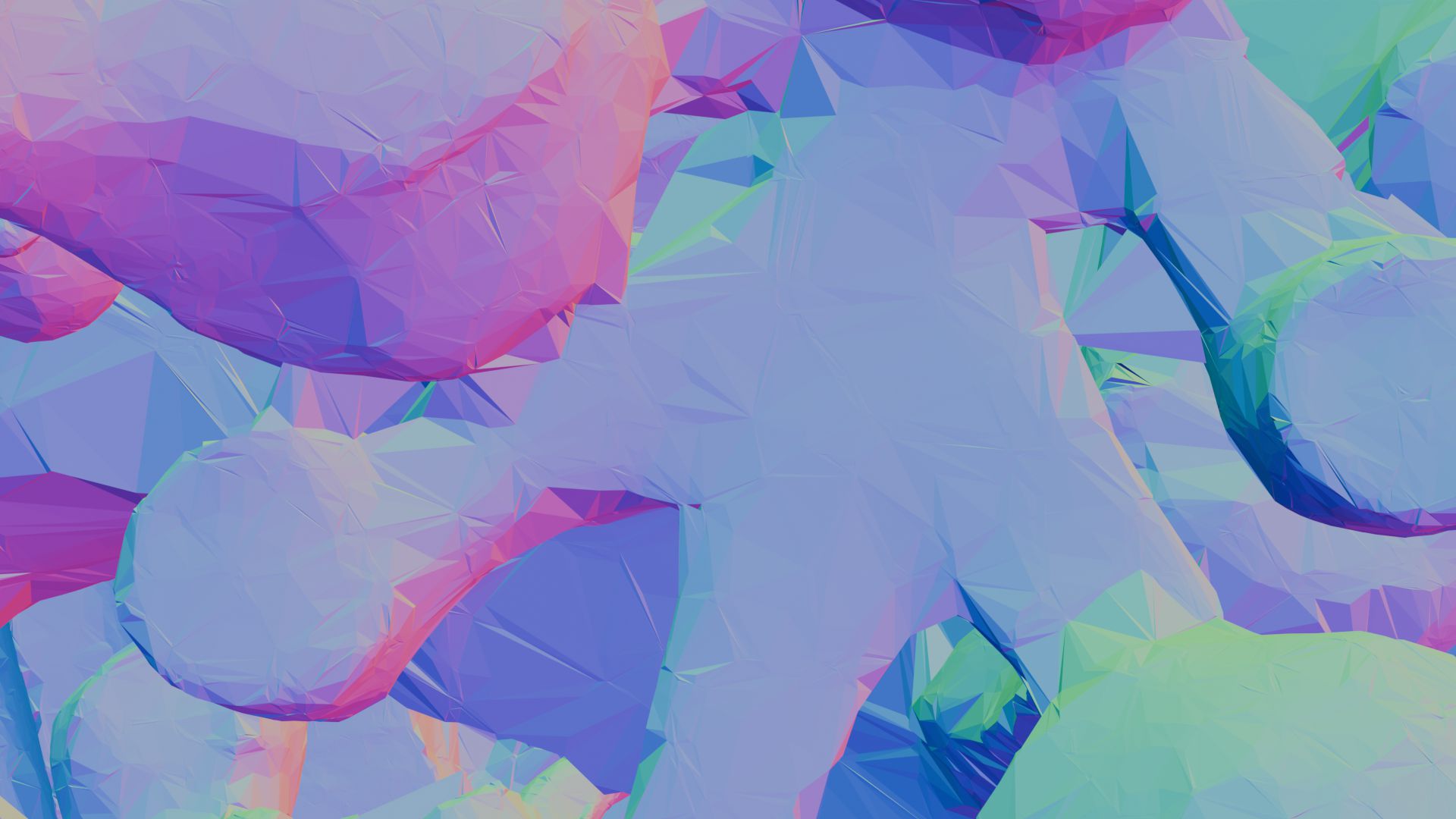}
    \includegraphics[width=0.49\linewidth]{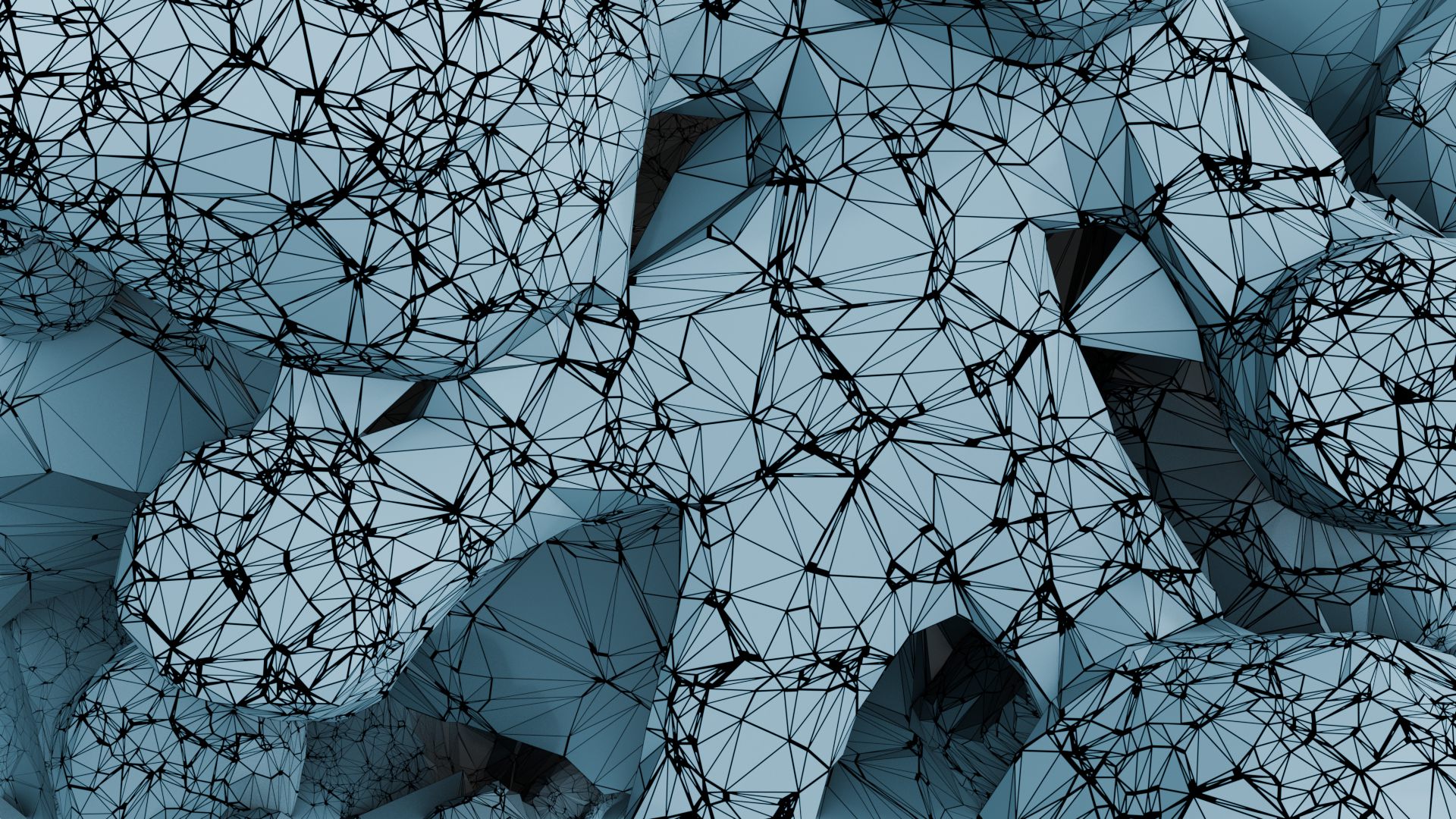}
\end{subfigure}
\begin{subfigure}[b]{0.49\linewidth}
    \centering
    \includegraphics[width=0.49\linewidth]{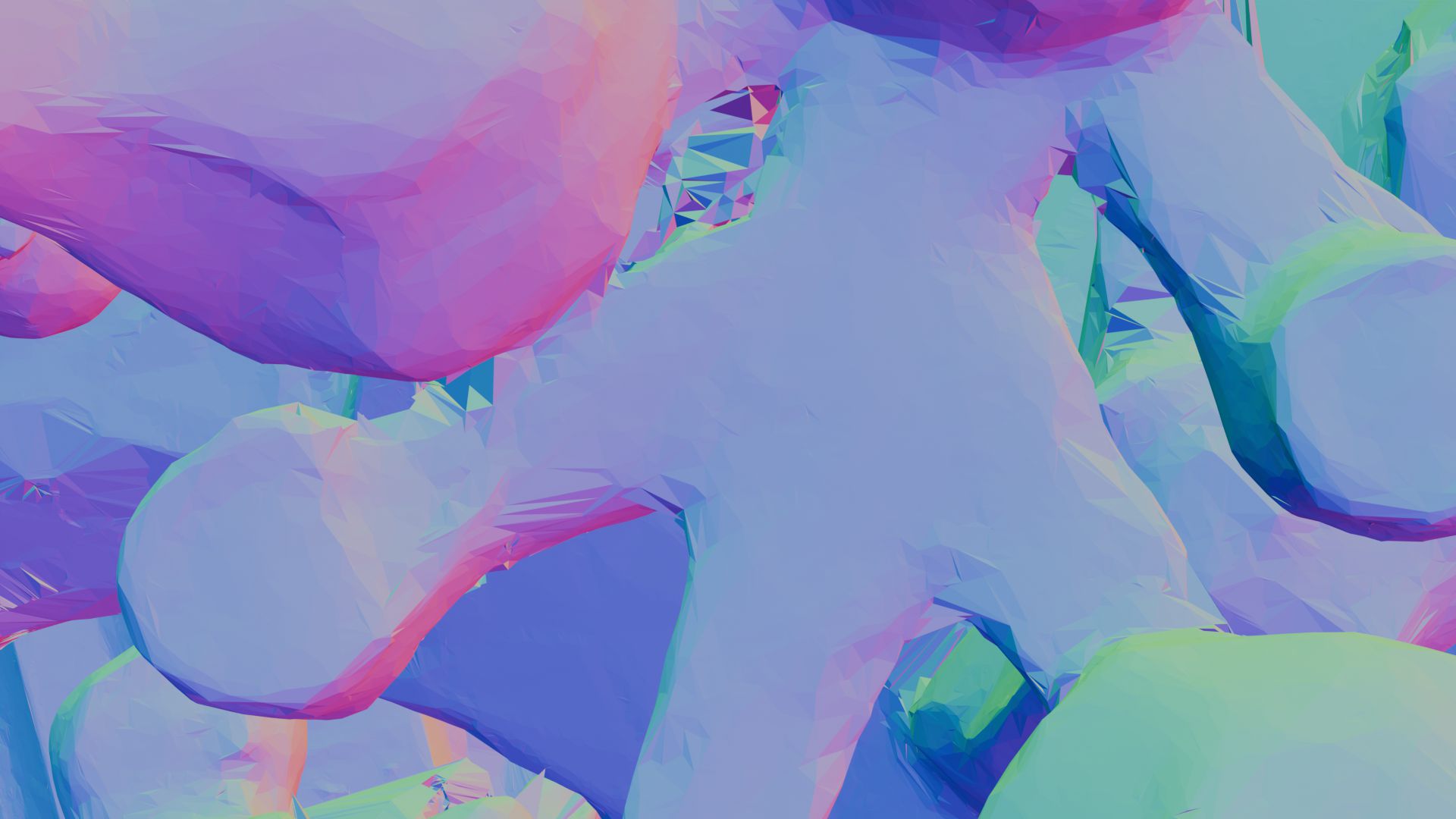}
    \includegraphics[width=0.49\linewidth]{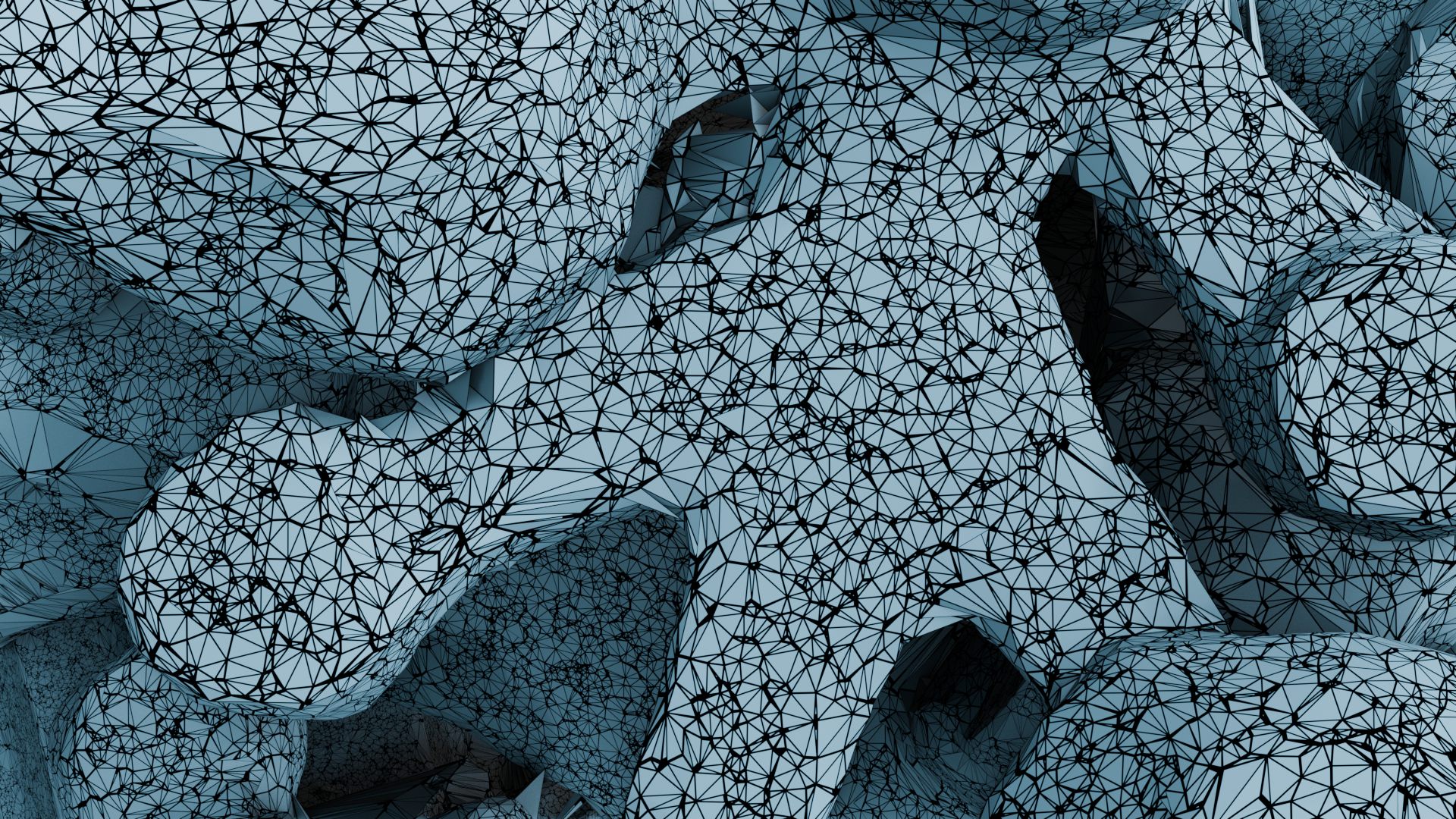}
\end{subfigure}
\begin{subfigure}[b]{0.49\linewidth}
    \centering
    \includegraphics[width=0.49\linewidth]{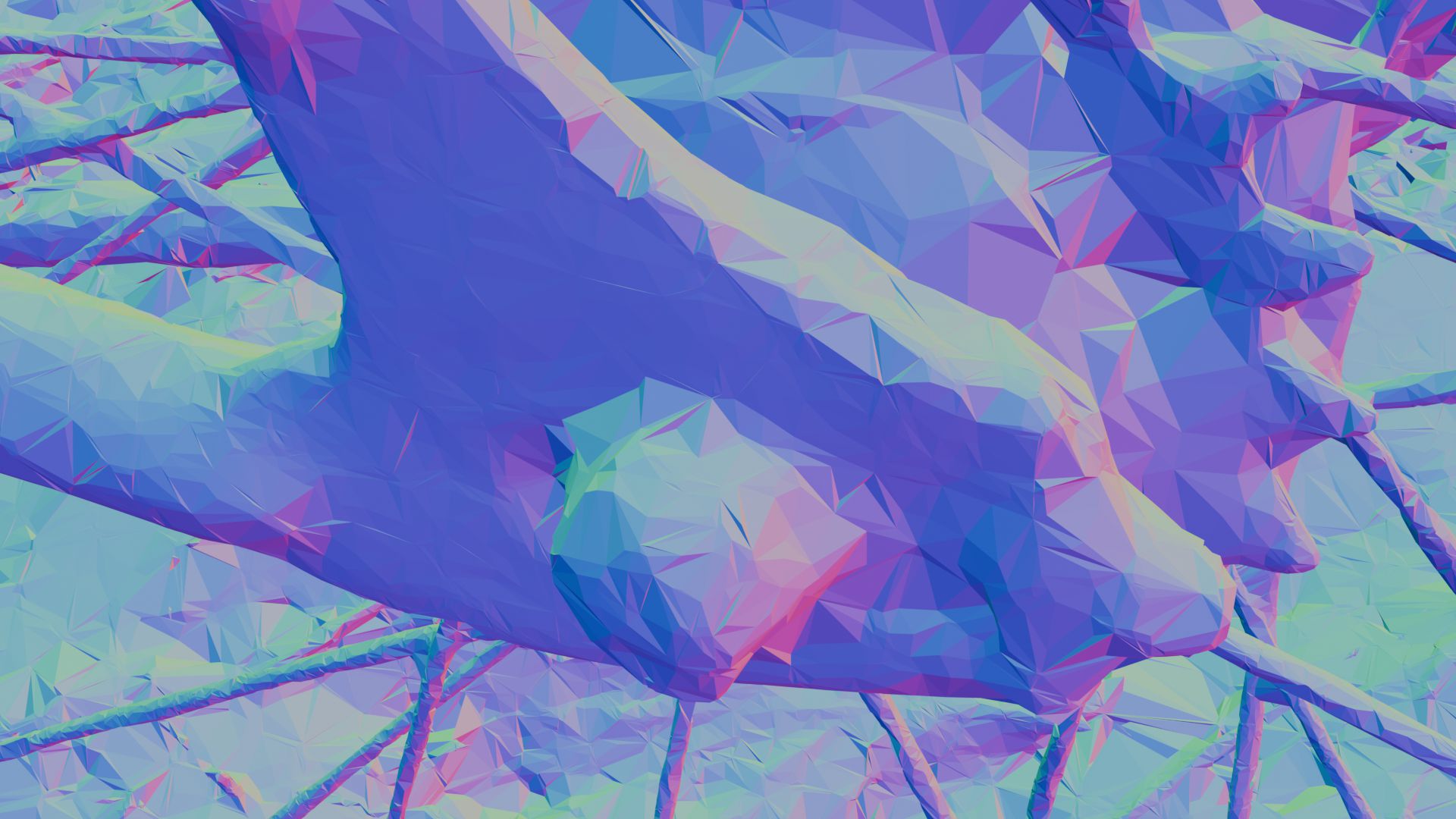}
    \includegraphics[width=0.49\linewidth]{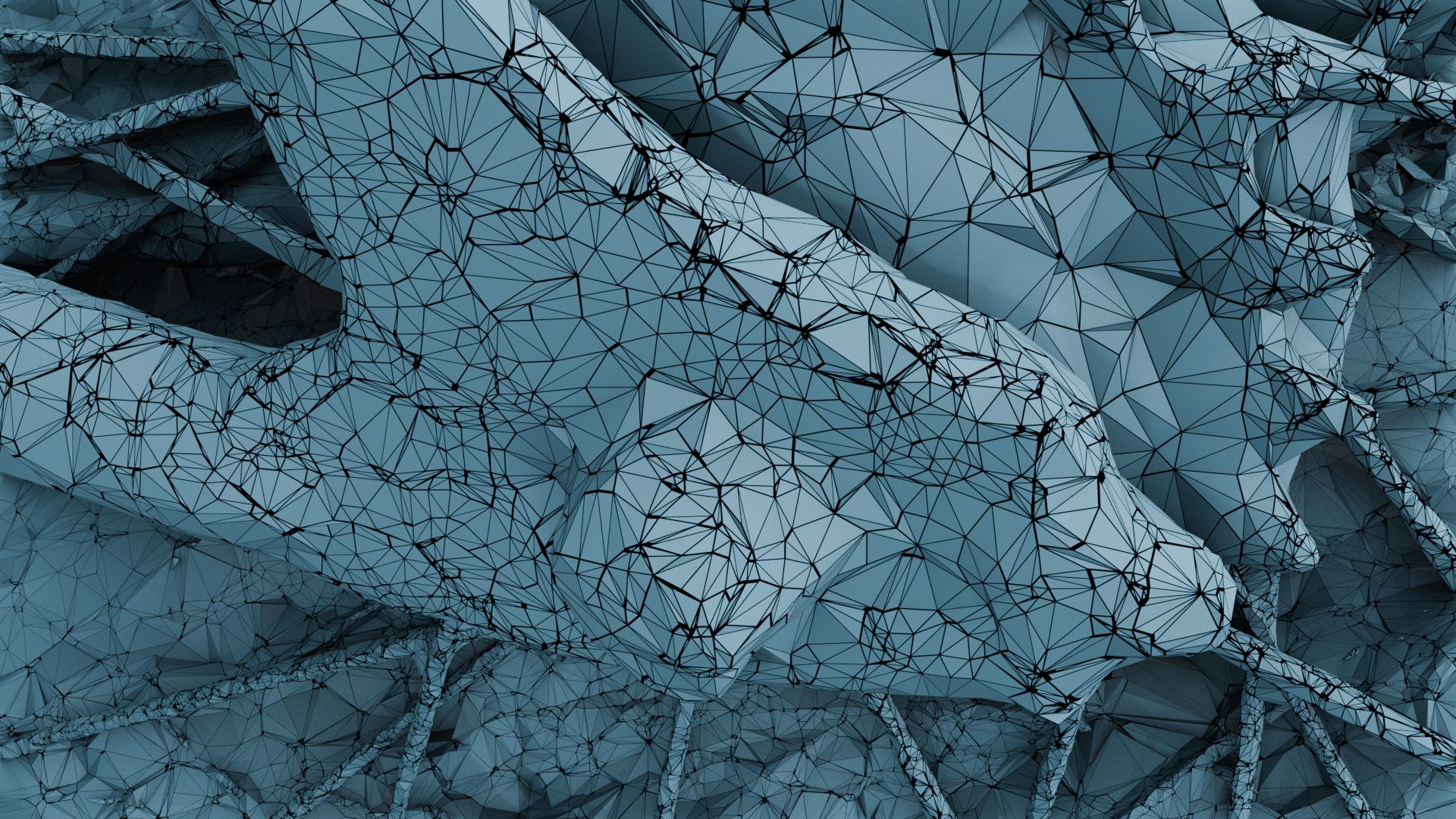}
    \caption{Pivot-based Marching Tetrahedra}
\end{subfigure}
\begin{subfigure}[b]{0.49\linewidth}
    \centering
    \includegraphics[width=0.49\linewidth]{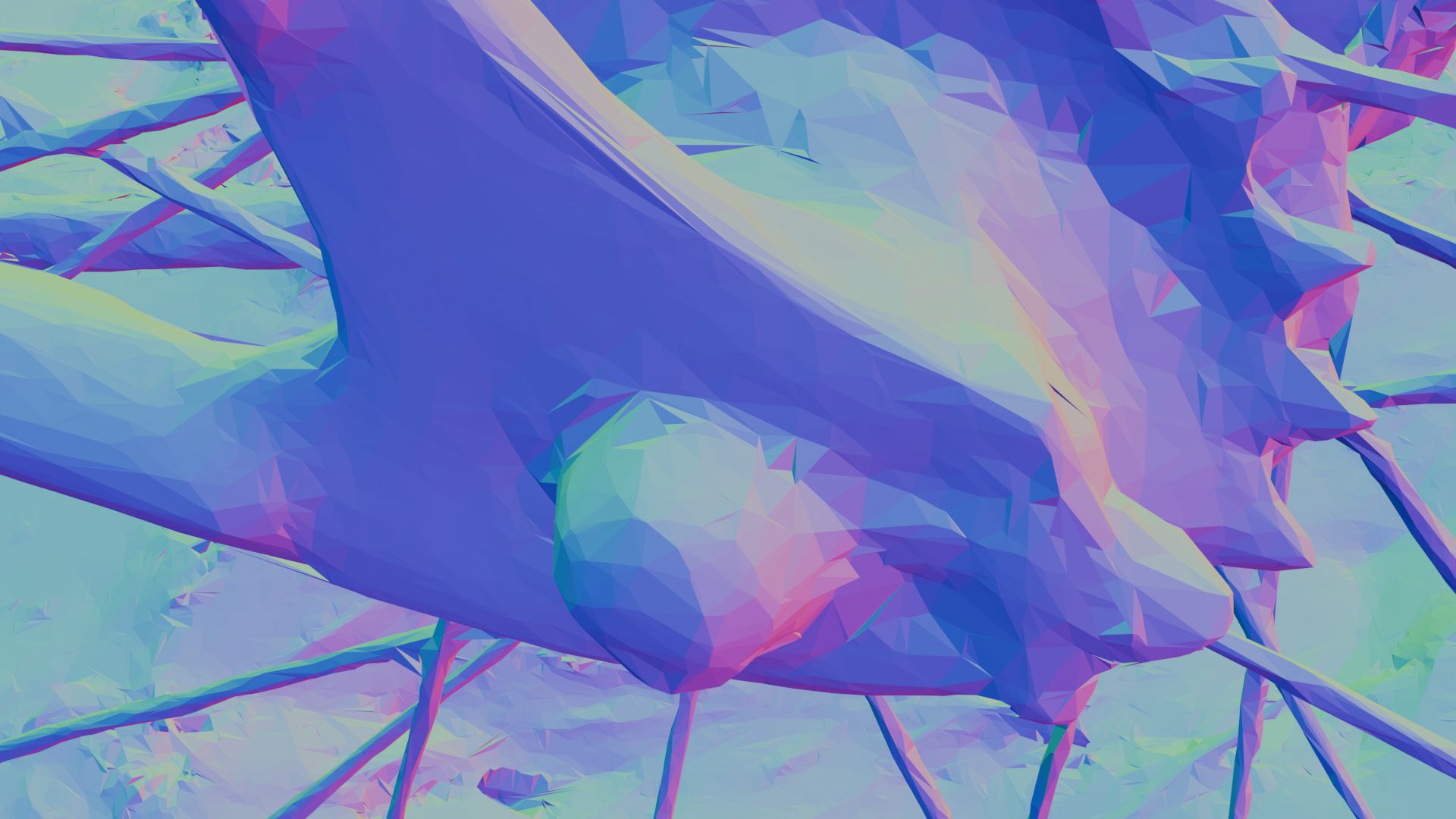}
    \includegraphics[width=0.49\linewidth]{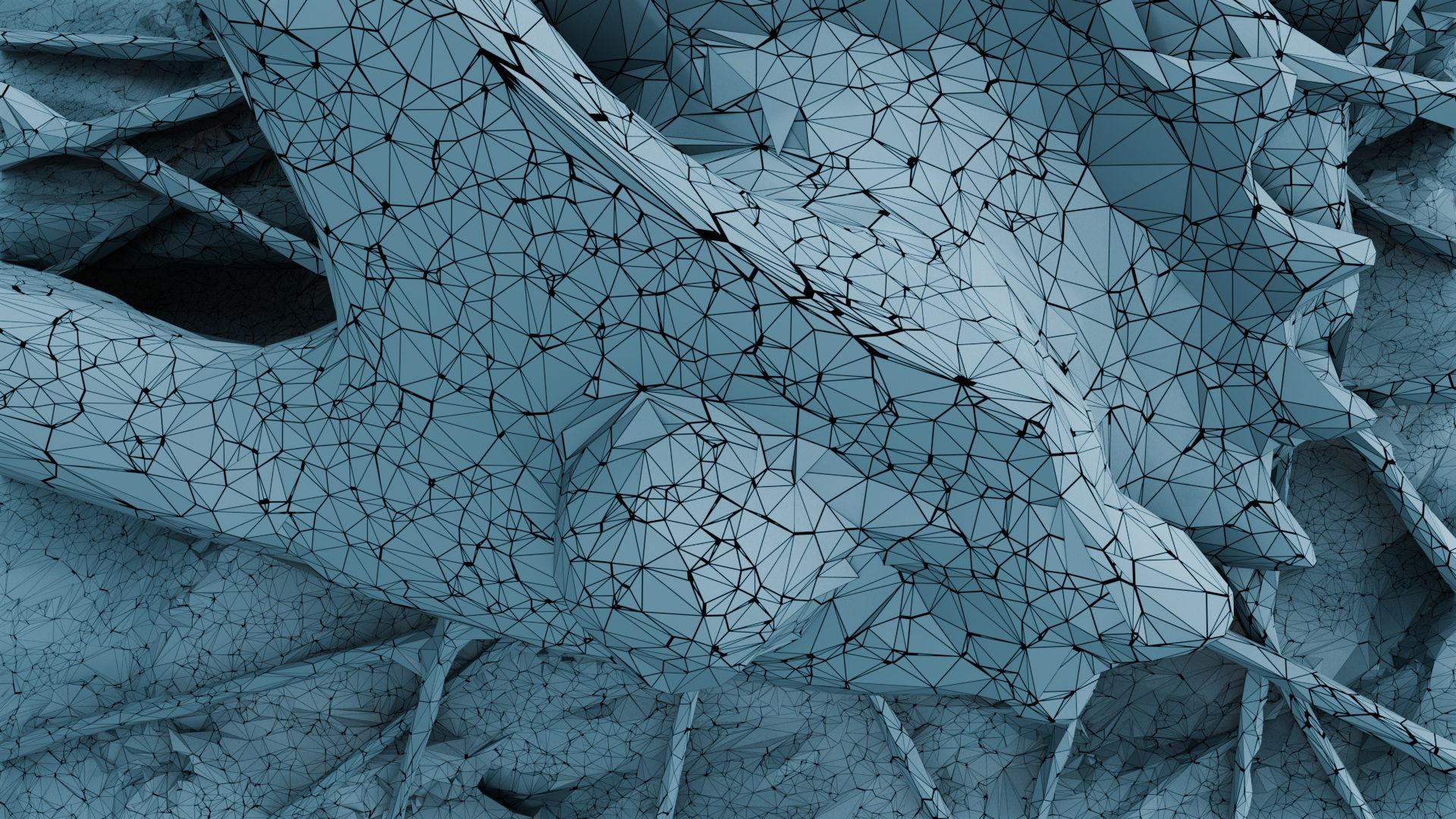}
    \caption{Primal Adaptive Mesh }
\end{subfigure}

\vspace*{-2mm}
\caption{\textbf{Primal Adaptive Meshing on restricted scene segments.} Close-up comparisons between our standard Pivot-based MTet extraction (a) and the Primal Adaptive Mesh (b) applied to the same segment on the bicycle and bonsai scenes of the MipNeRF 360 dataset. The Primal Adaptive Mesh produces topologically clean surfaces that faithfully reflect the underlying Gaussian isosurface. 
\vspace*{-3mm}
}
\label{fig:primal_ablation_meshing_zoom}
\end{figure}

This approach effectively decouples mesh resolution from Gaussian distribution, enabling extremely fine meshing of restricted segments of the scene~(\cref{fig:primal_ablation_meshing_zoom}). Although global distance-based metrics remain largely unchanged ---being inherently insensitive to high-frequency detail--- the resulting meshes exhibit substantially improved smoothness and reduced discretization artifacts.
Details for the isosurface refinement stage can be found in the appendix.
\section{Experiments}\label{sec:experiments}

\begin{figure*}[t]
\centering

\resizebox{\linewidth}{!}{
   
    \includegraphics[width=0.20\linewidth]{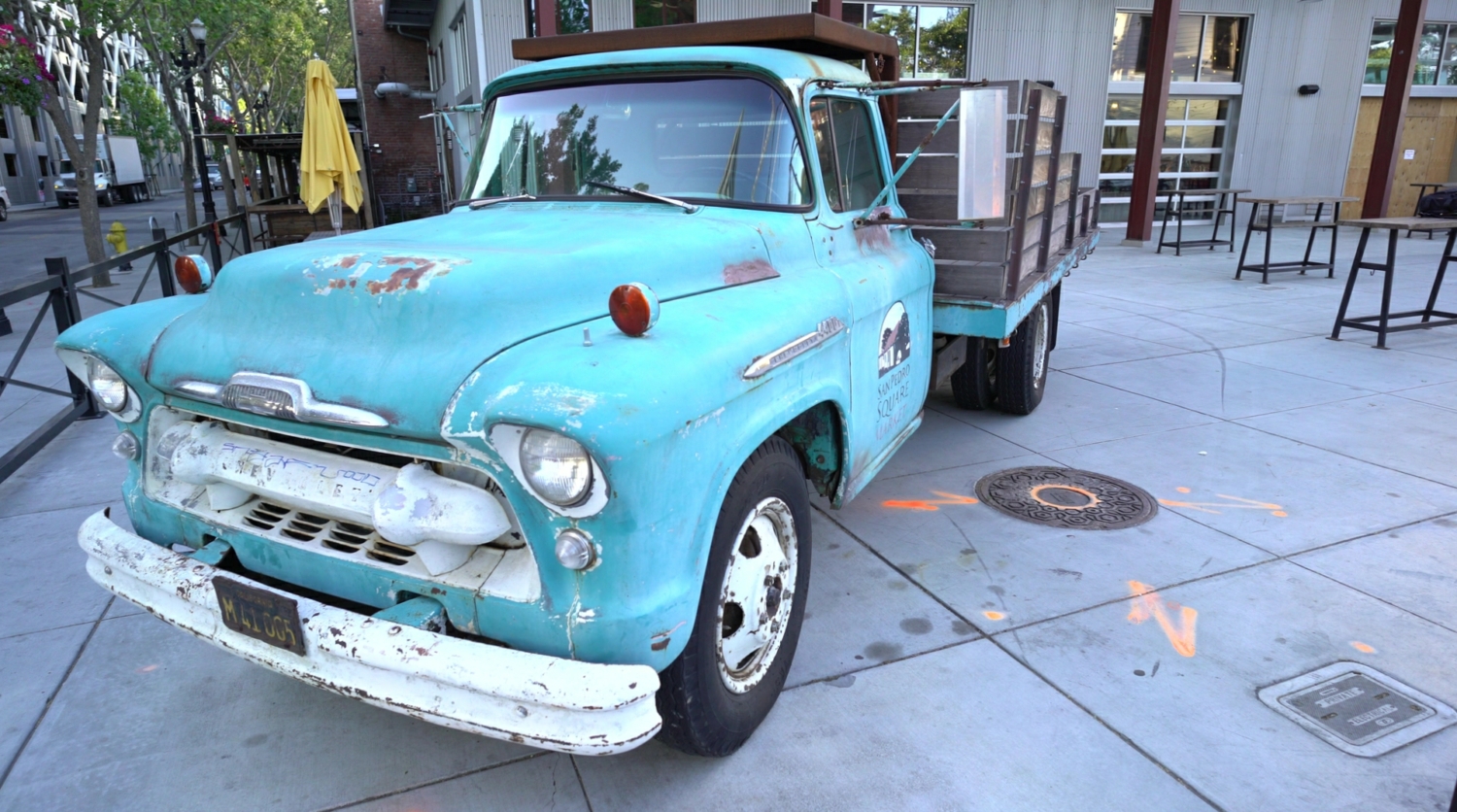}
    \includegraphics[width=0.20\linewidth]{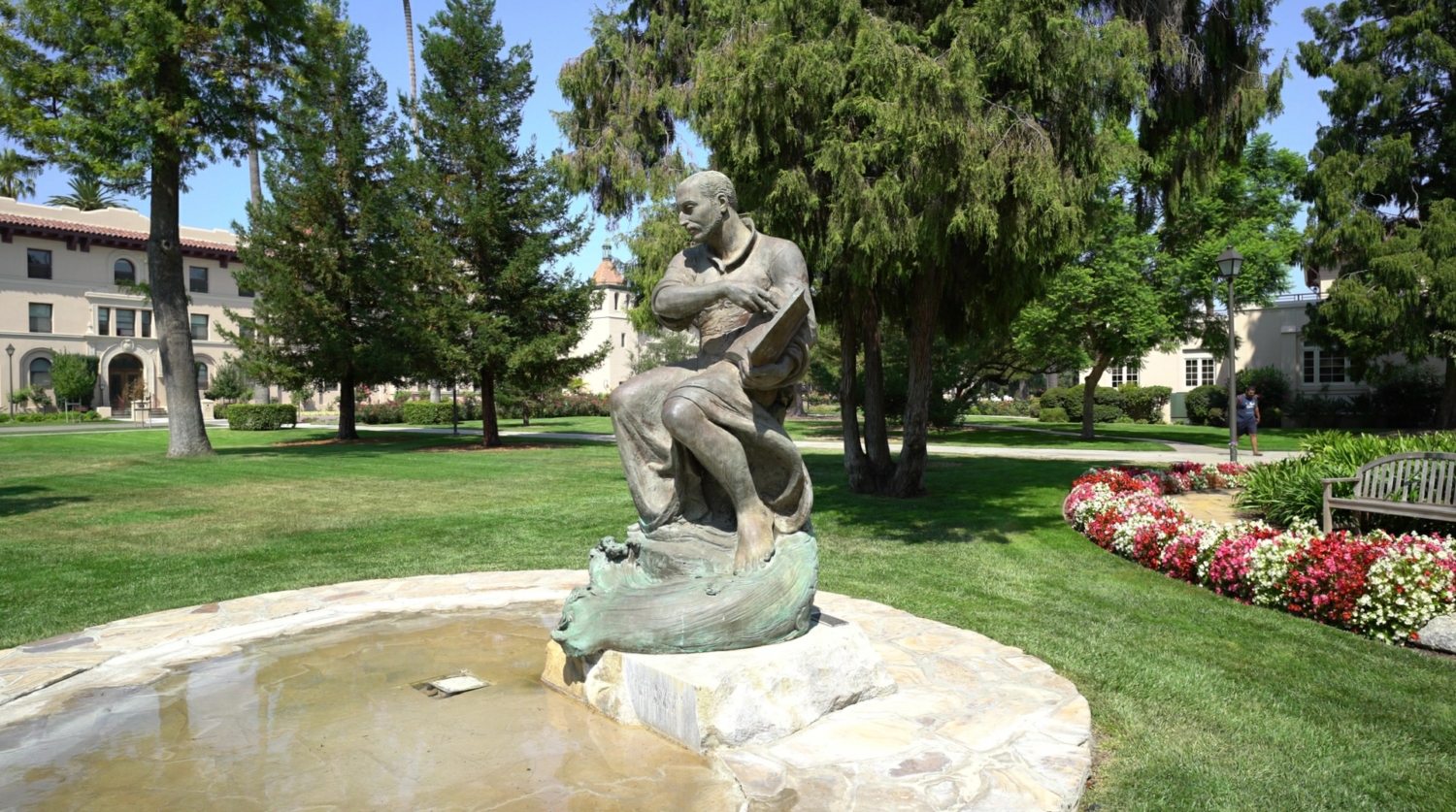}
    \includegraphics[width=0.20\linewidth]{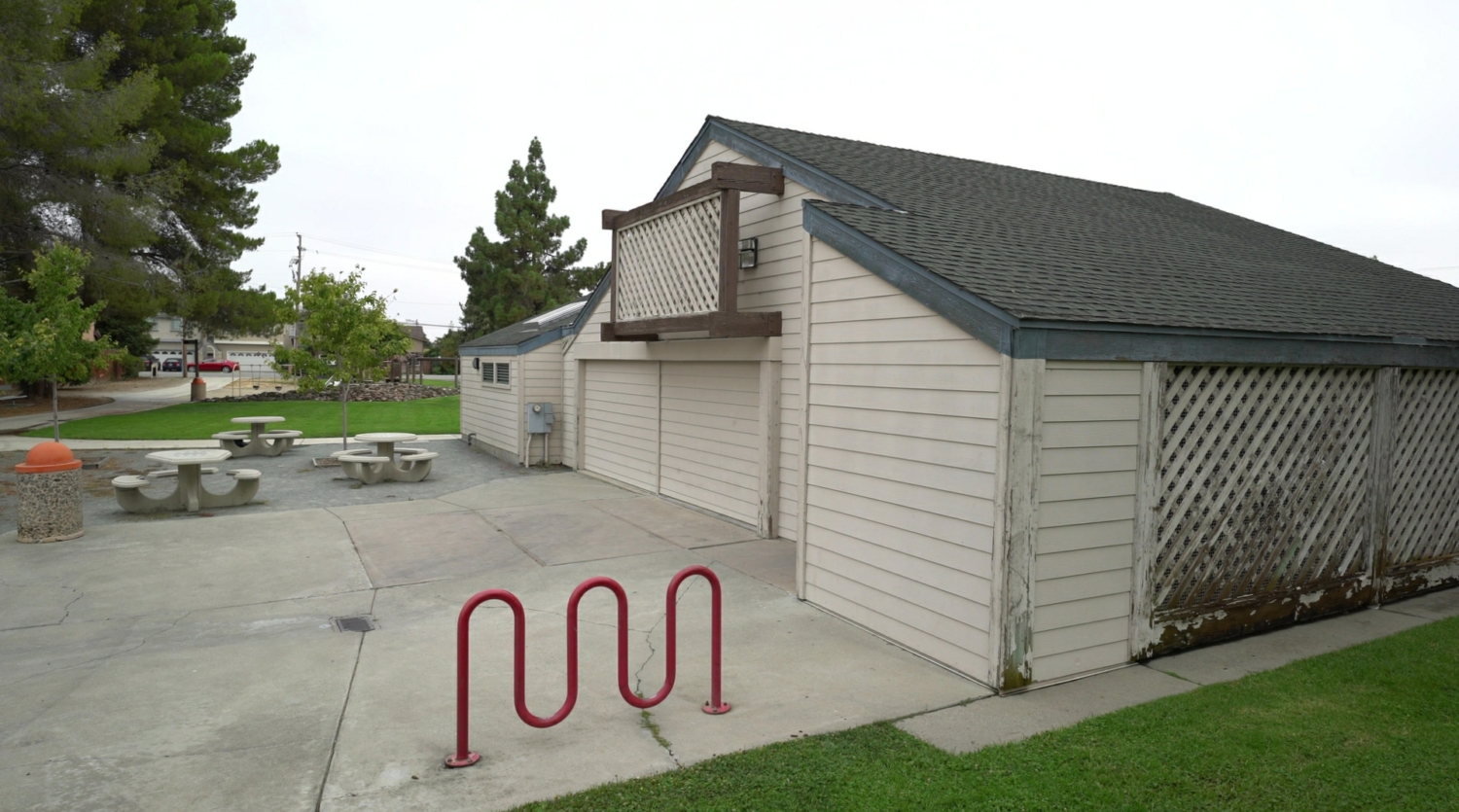}
    \includegraphics[width=0.20\linewidth]{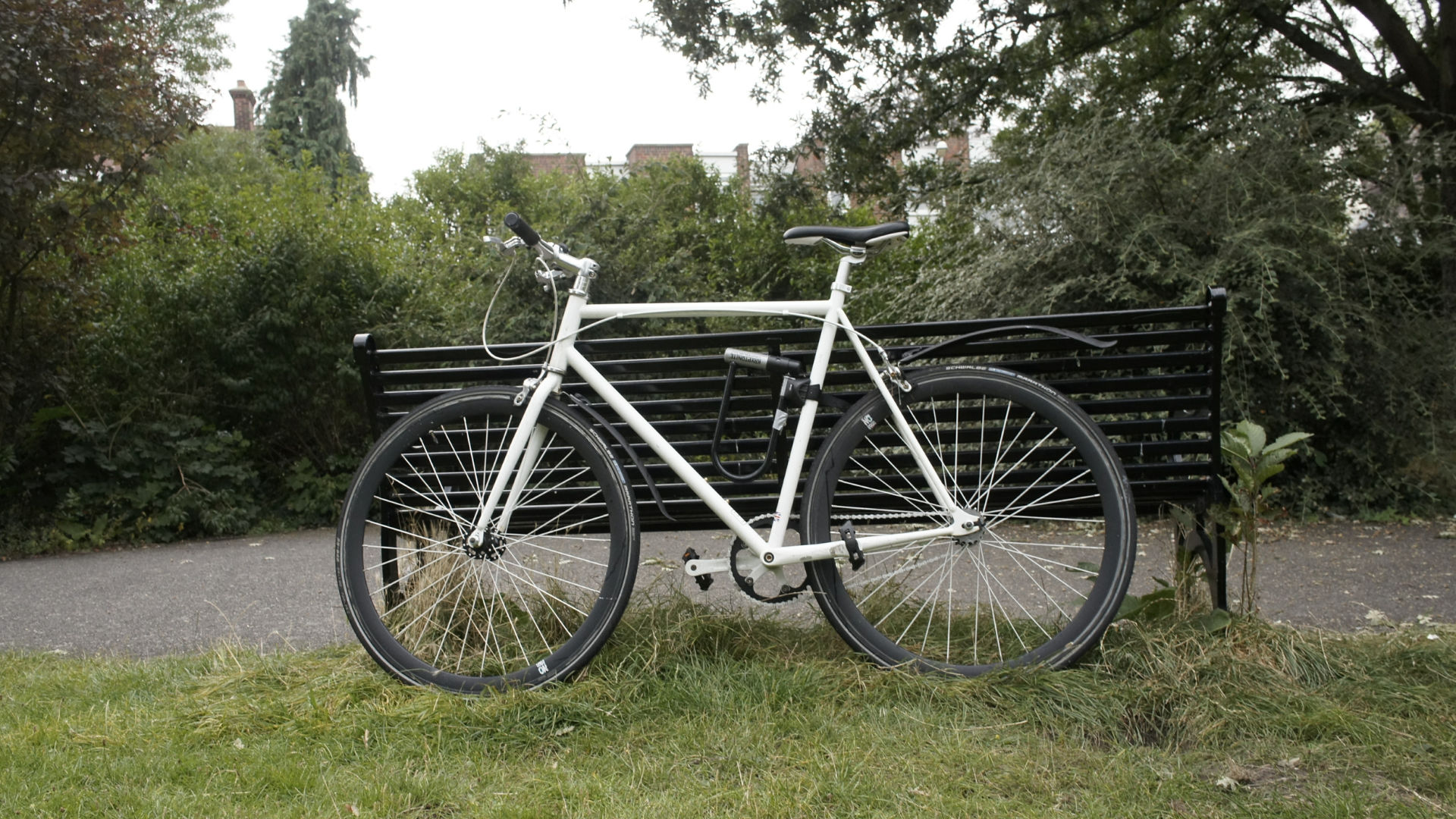}
    \includegraphics[width=0.20\linewidth]{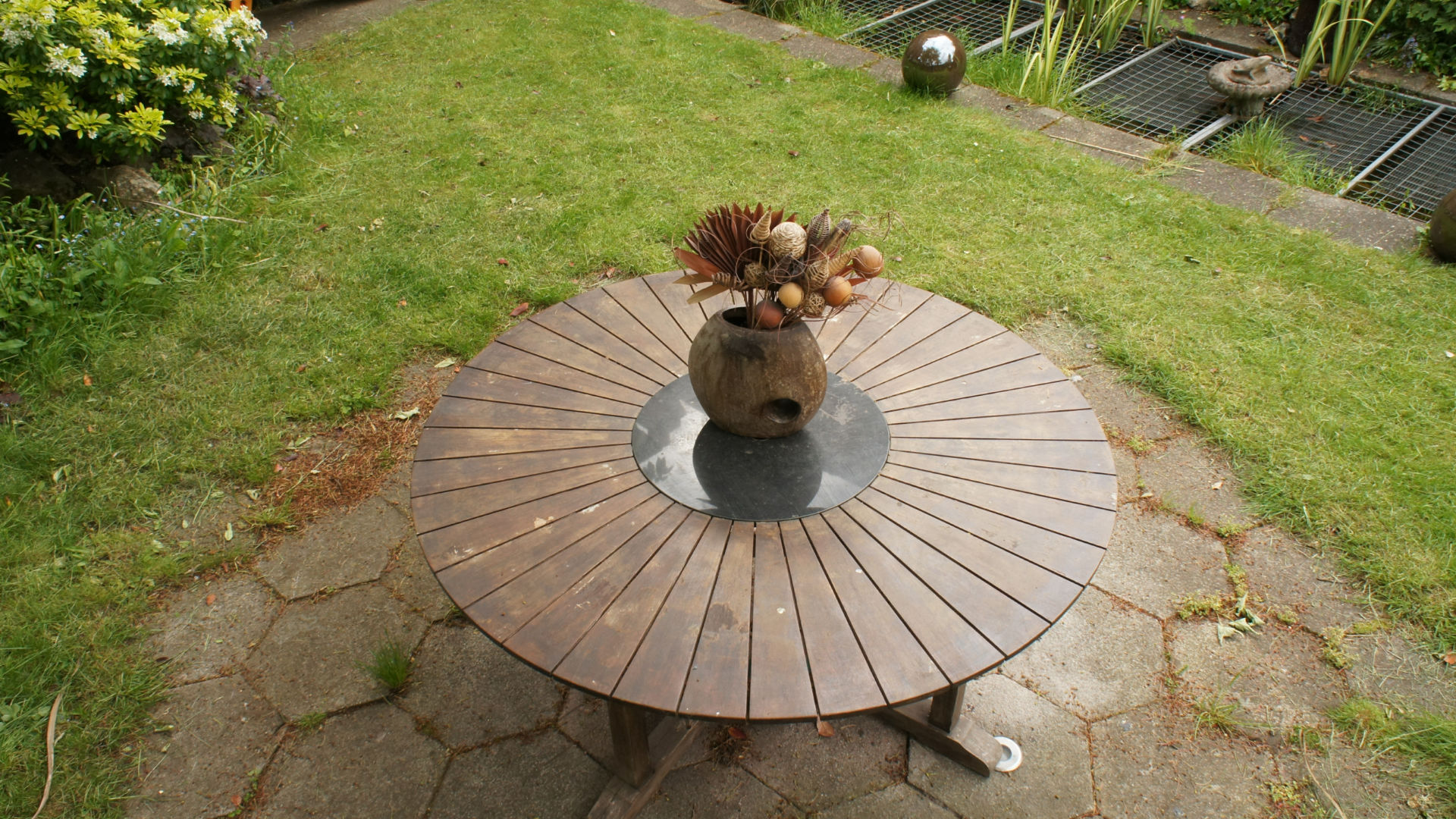}
}
\resizebox{\linewidth}{!}{
    \includegraphics[width=0.20\linewidth]{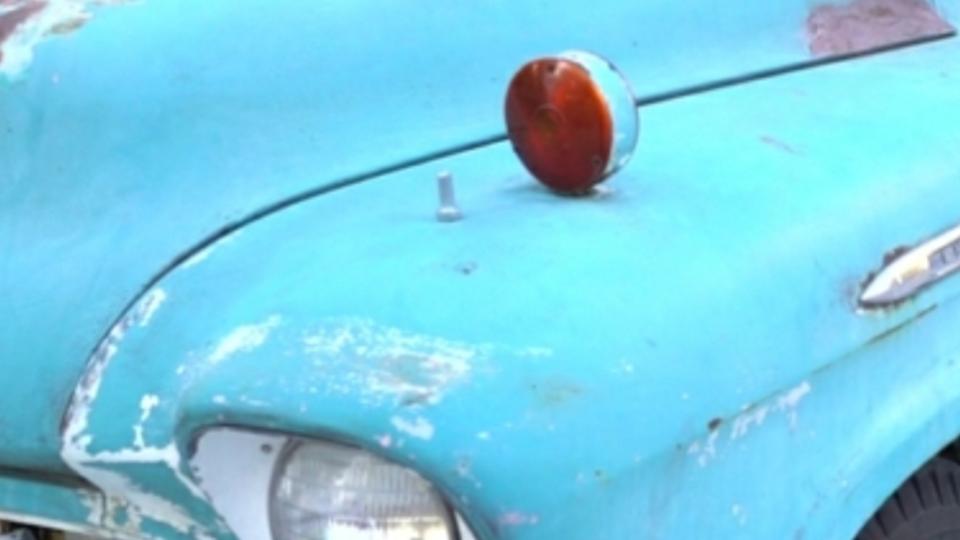}
    \includegraphics[width=0.20\linewidth]{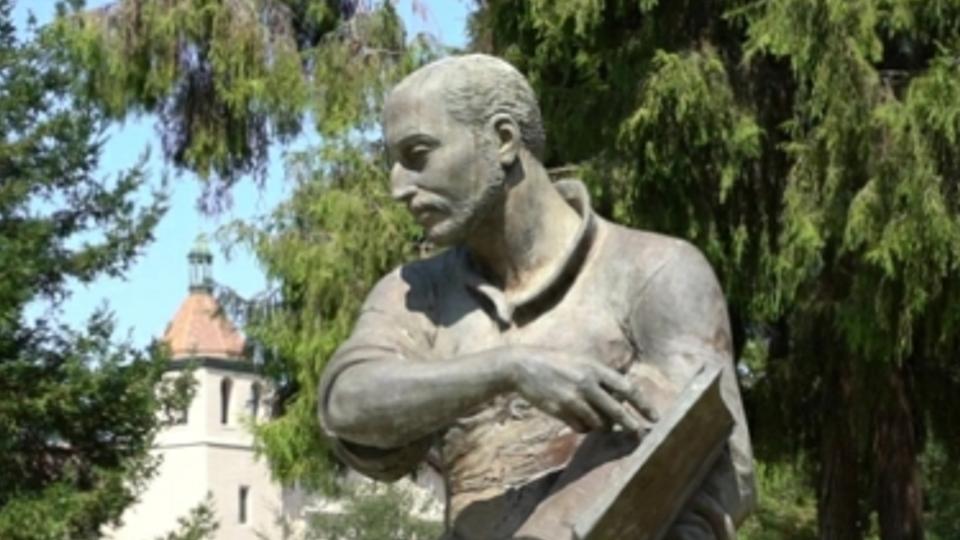}
    \includegraphics[width=0.20\linewidth]{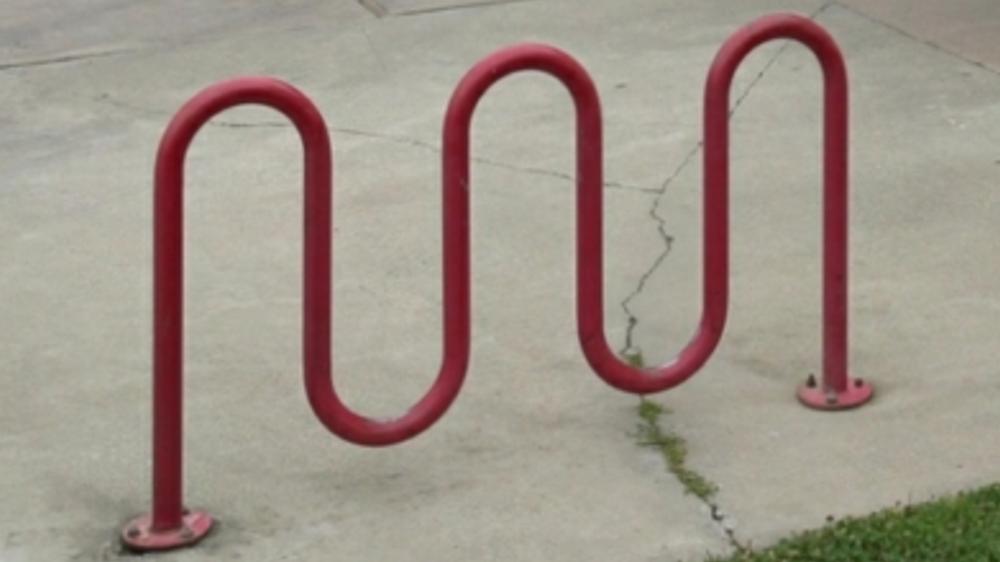}
    \includegraphics[width=0.20\linewidth]{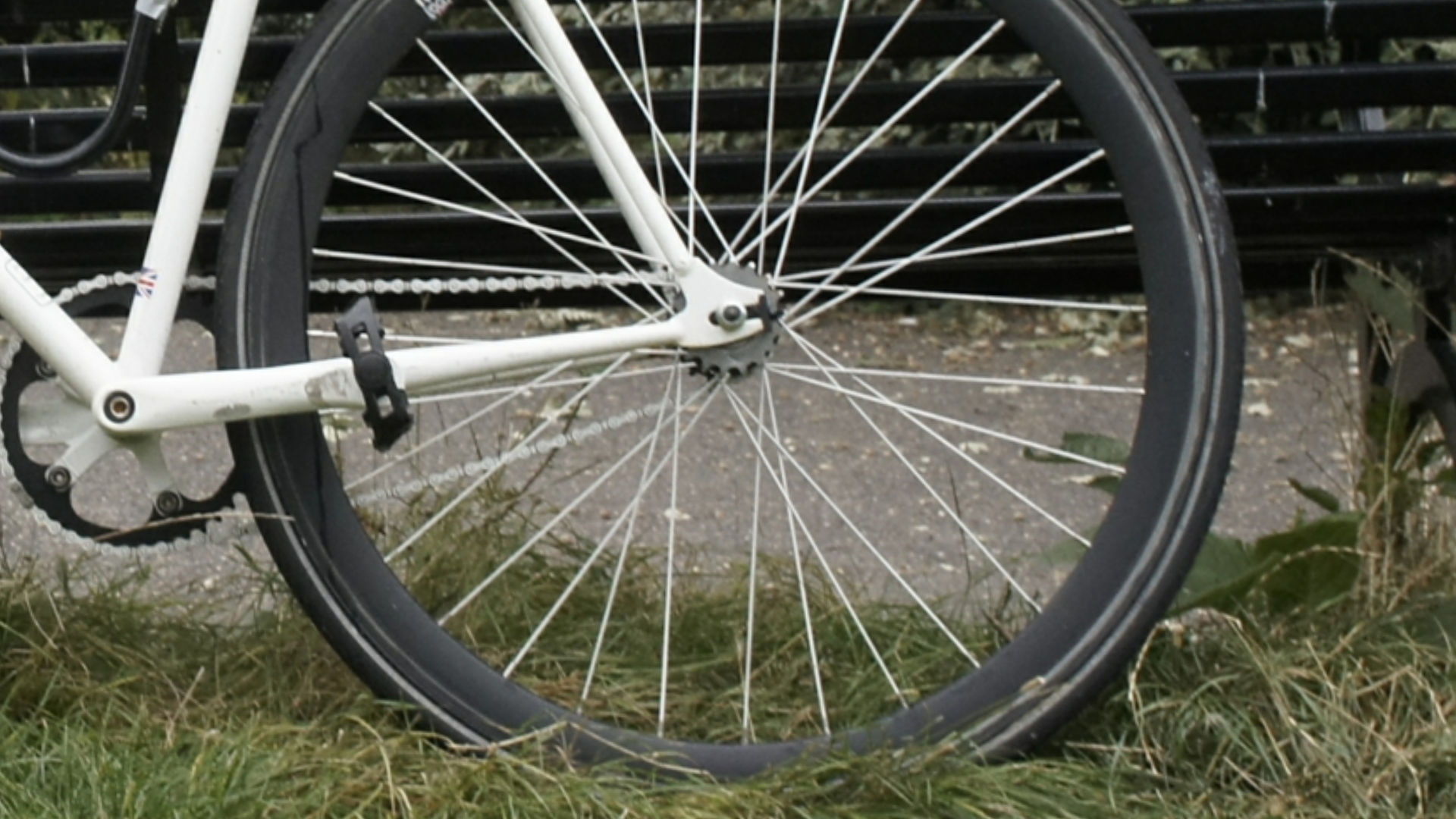}
    \includegraphics[width=0.20\linewidth]{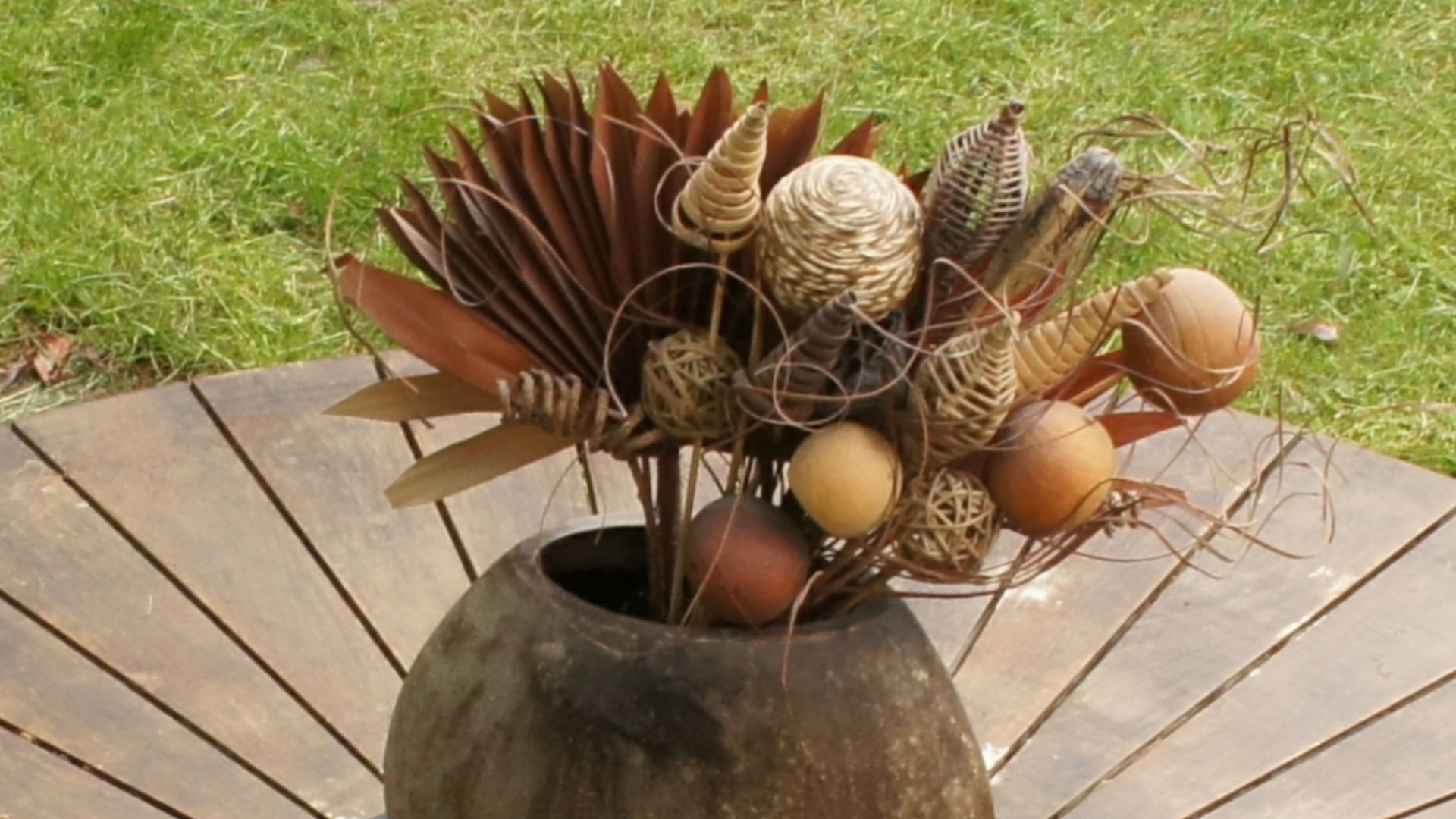}

}
\vspace*{-6mm}
\caption*{Ground Truth Images}
\vspace*{1mm}

\resizebox{\linewidth}{!}{
    \includegraphics[width=0.20\linewidth]{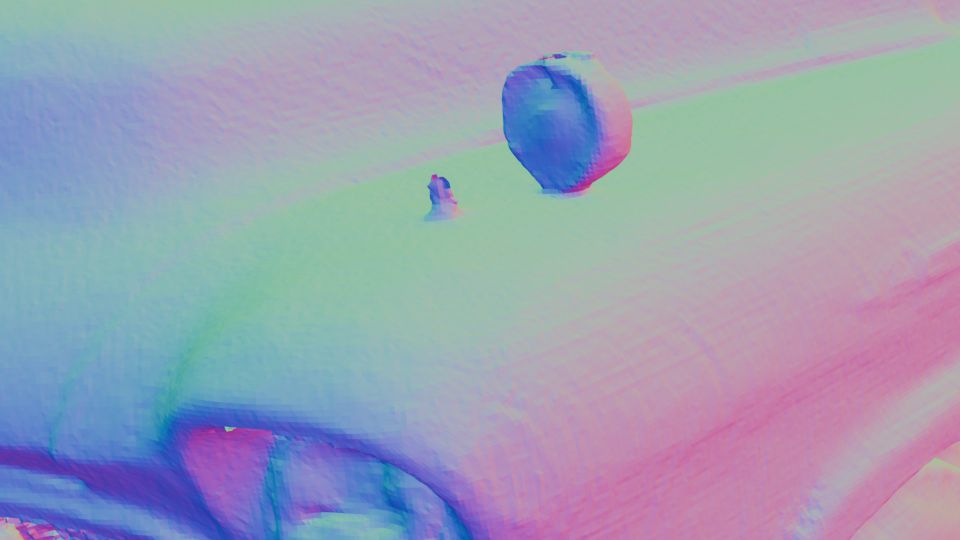} 
    \includegraphics[width=0.20\linewidth]{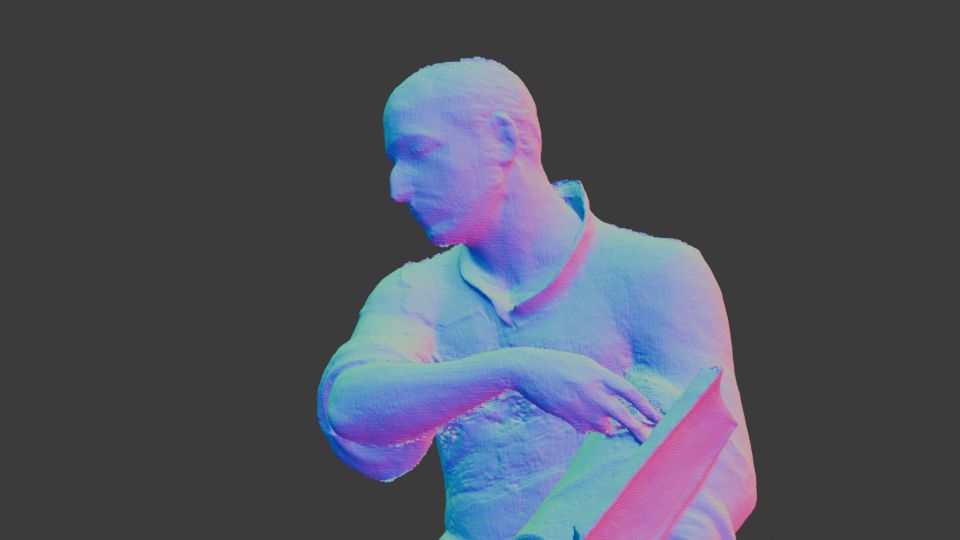} 
    \includegraphics[width=0.20\linewidth]{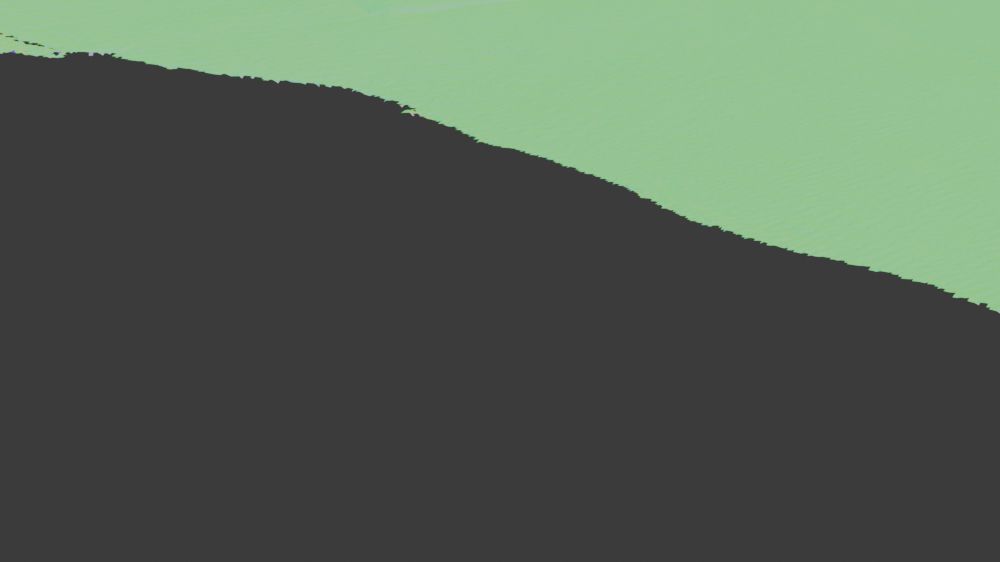}
    \includegraphics[width=0.20\linewidth]{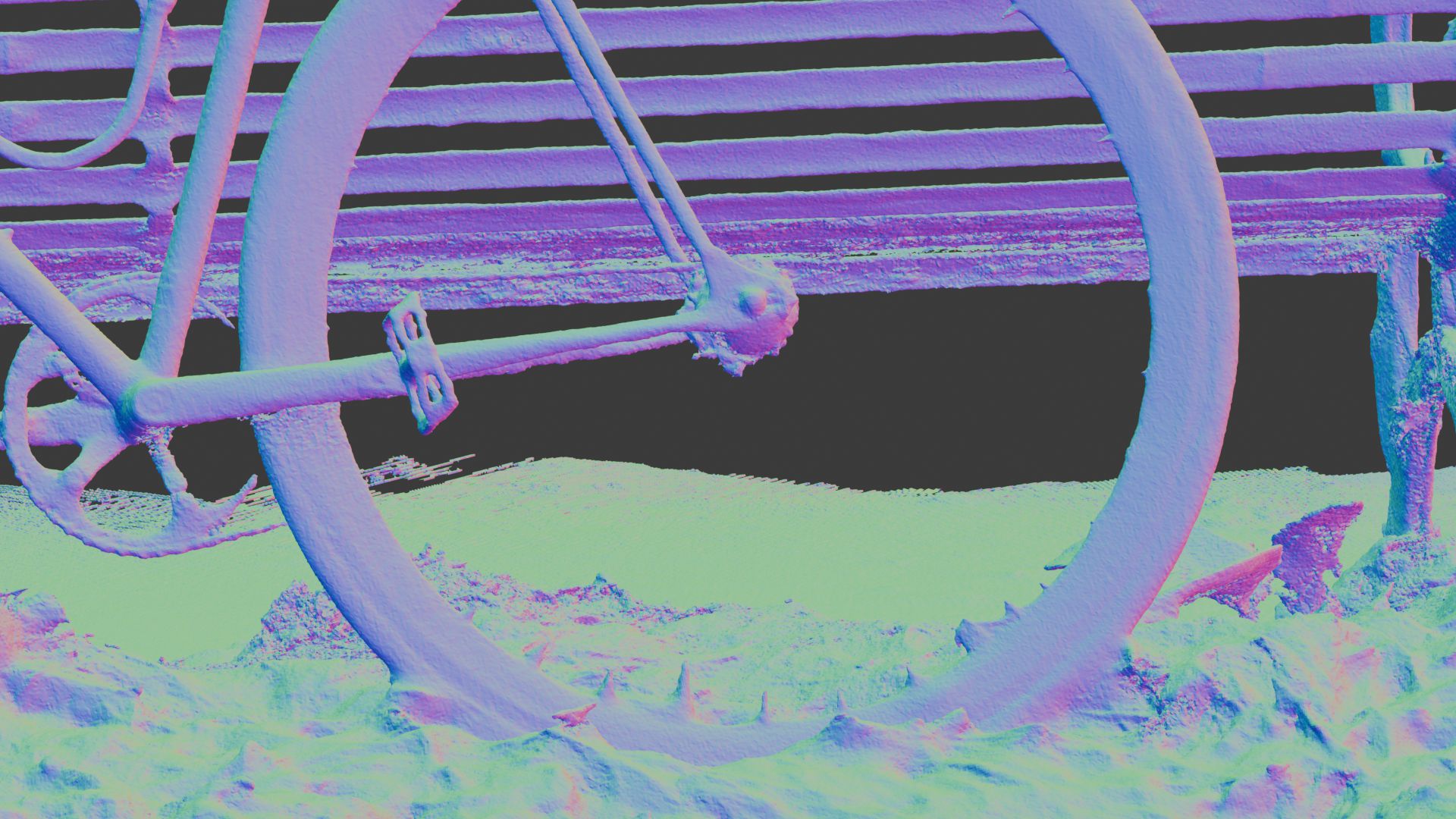}
    \includegraphics[width=0.20\linewidth]{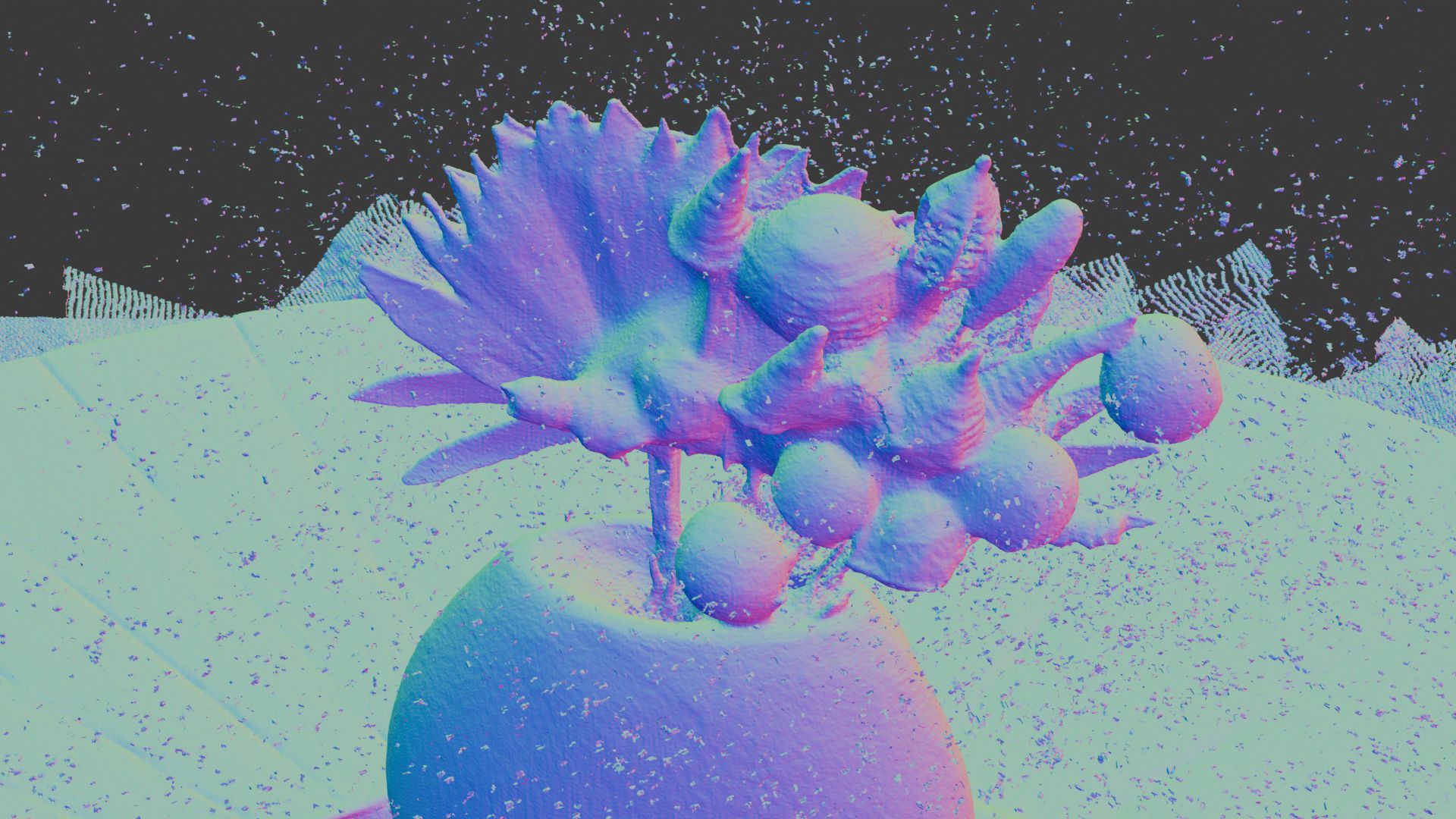}

}
\vspace*{-6mm}
\caption*{(b) PGSR} 
\vspace*{1mm}

\resizebox{\linewidth}{!}{
    \includegraphics[width=0.20\linewidth]{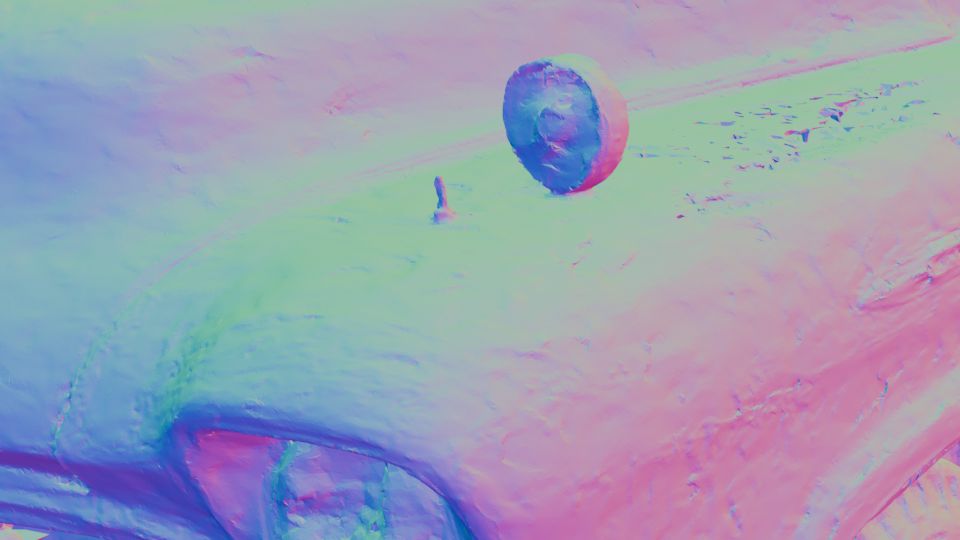} 
    \includegraphics[width=0.20\linewidth]{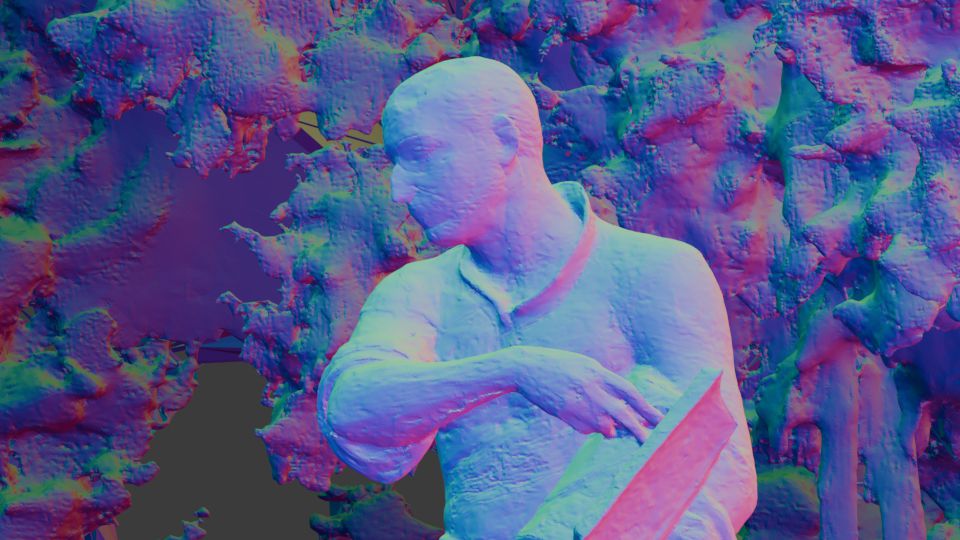} 
    \includegraphics[width=0.20\linewidth]{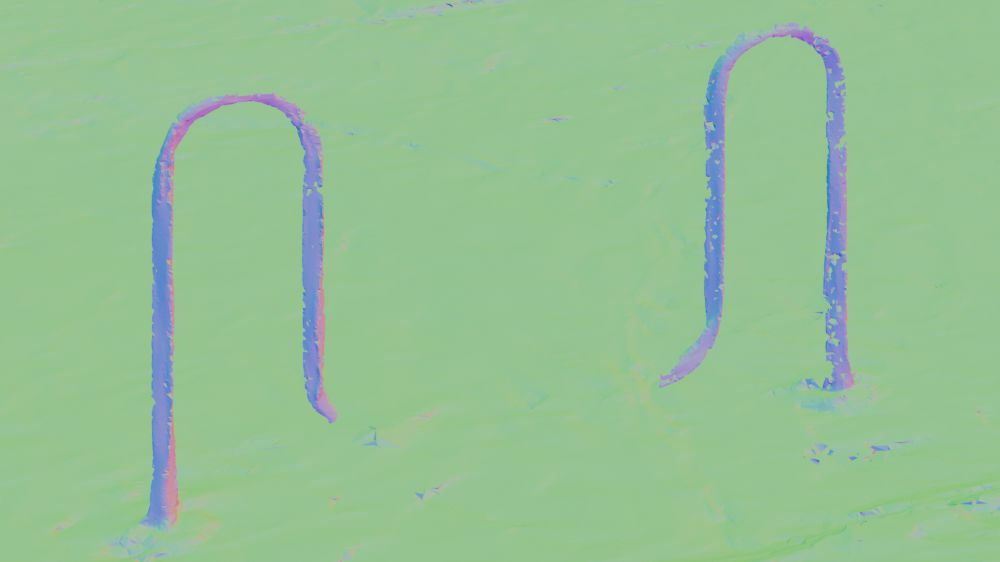}
    \includegraphics[width=0.20\linewidth]{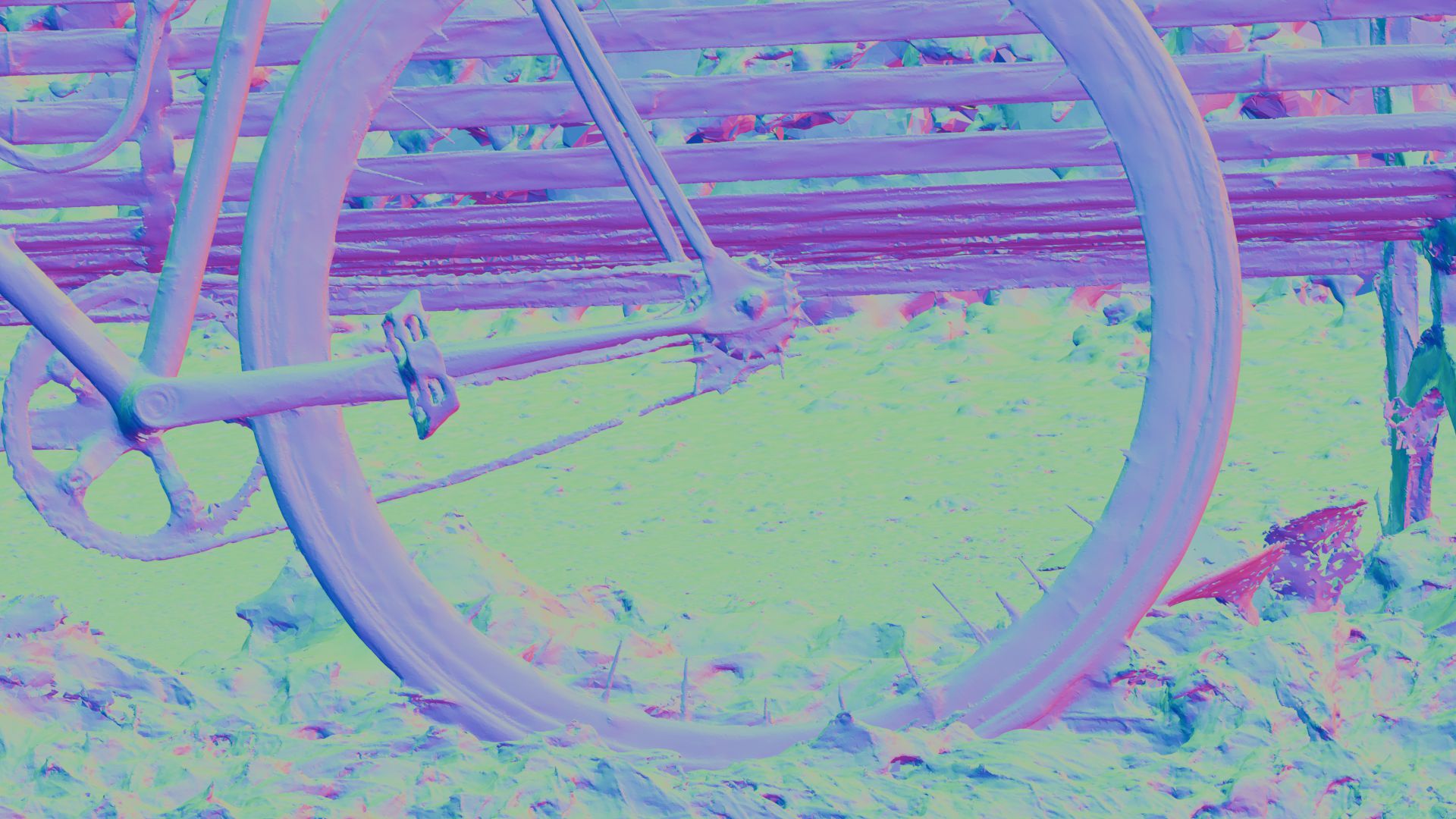}
    \includegraphics[width=0.20\linewidth]{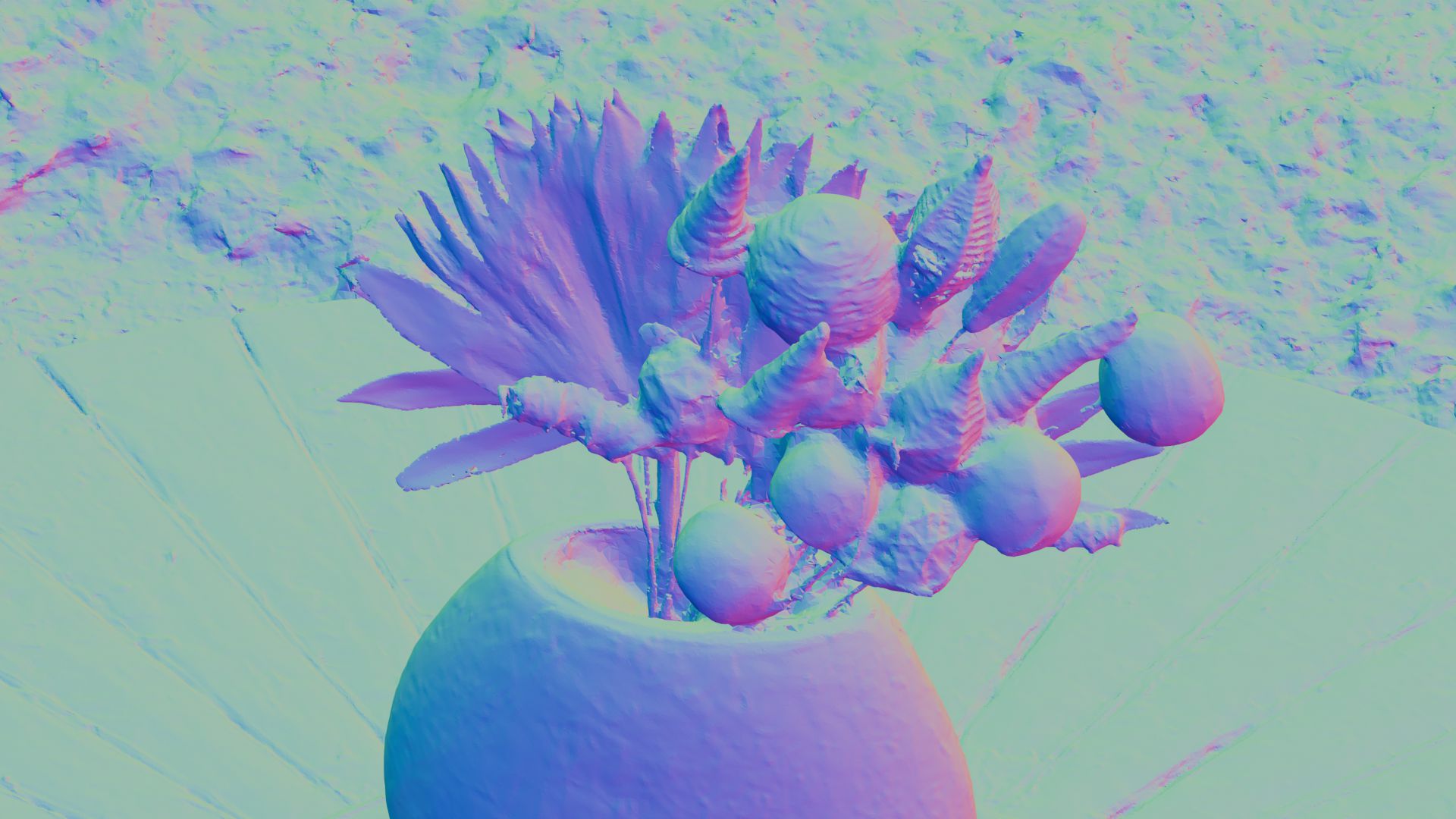}

}
\vspace*{-6mm}
\caption*{(a) GGGS} 
\vspace*{1mm}

\resizebox{\linewidth}{!}{
    \includegraphics[width=0.20\linewidth]{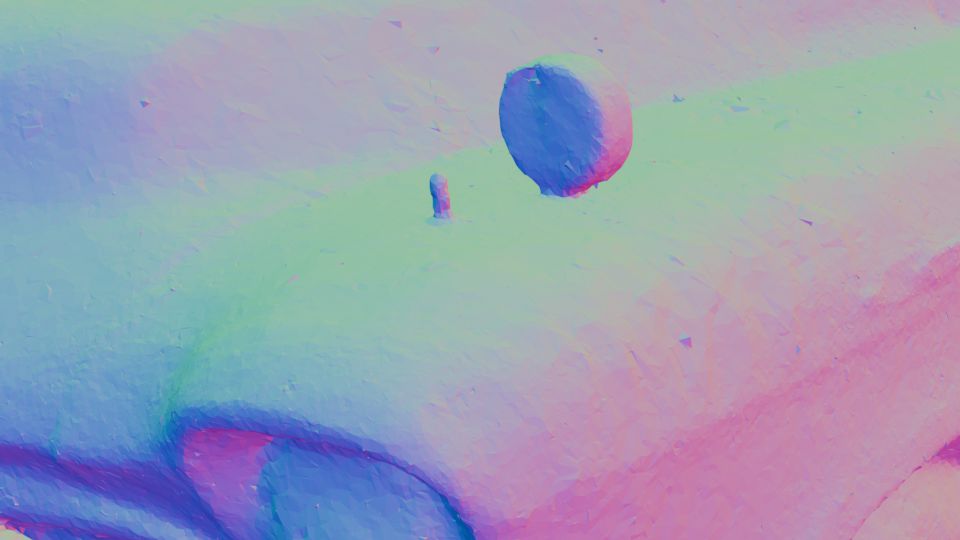} 
    \includegraphics[width=0.20\linewidth]{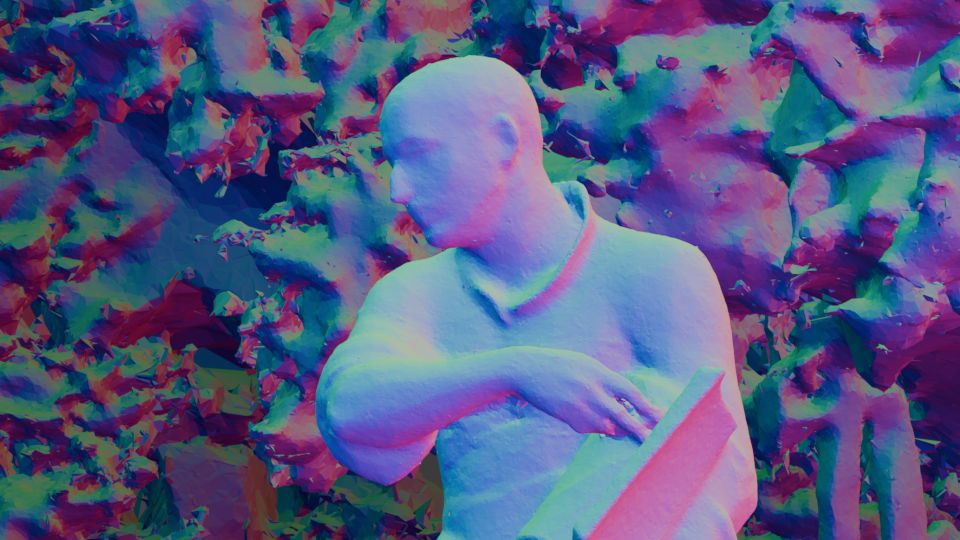} 
    \includegraphics[width=0.20\linewidth]{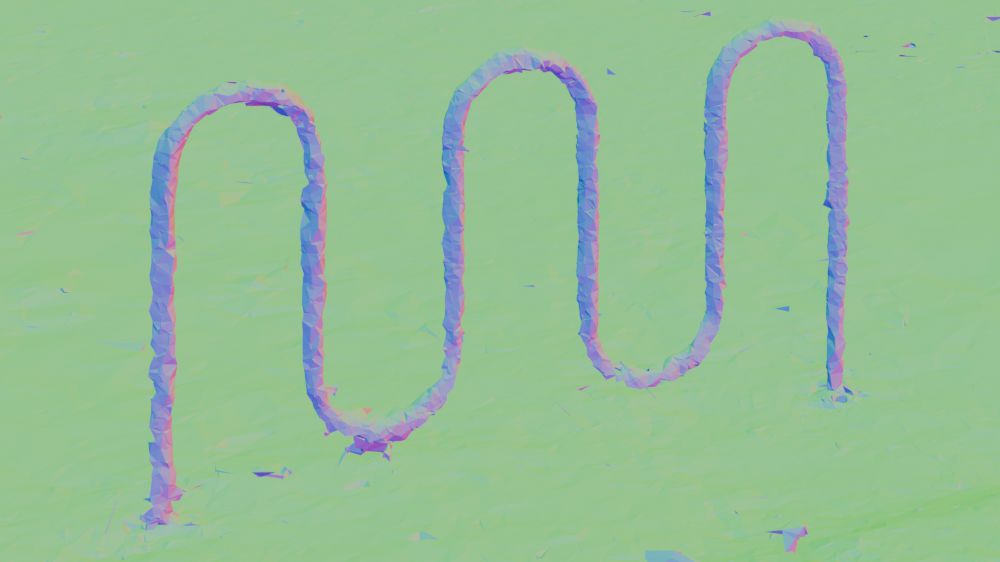}
    \includegraphics[width=0.20\linewidth]{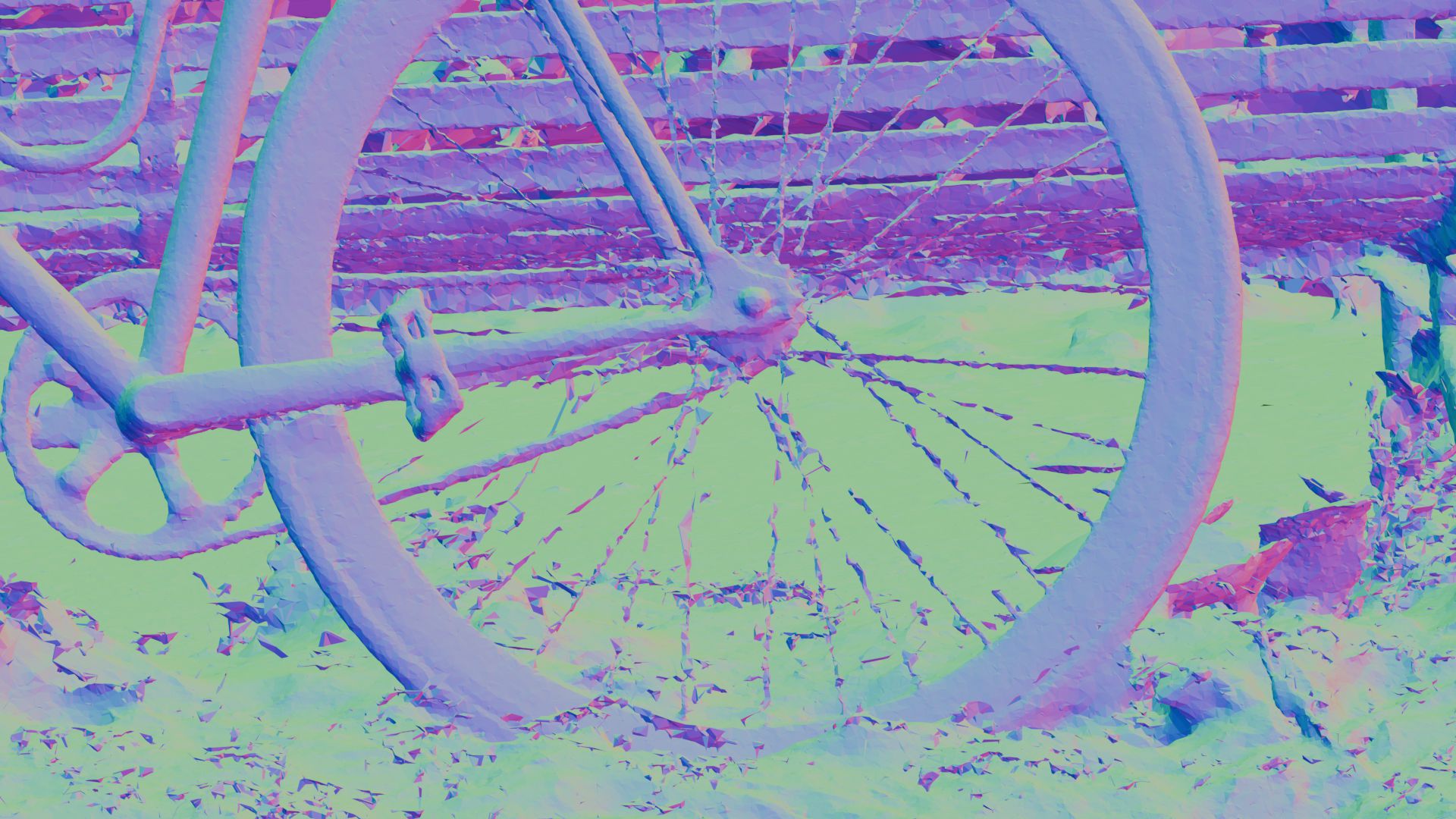}
    \includegraphics[width=0.20\linewidth]{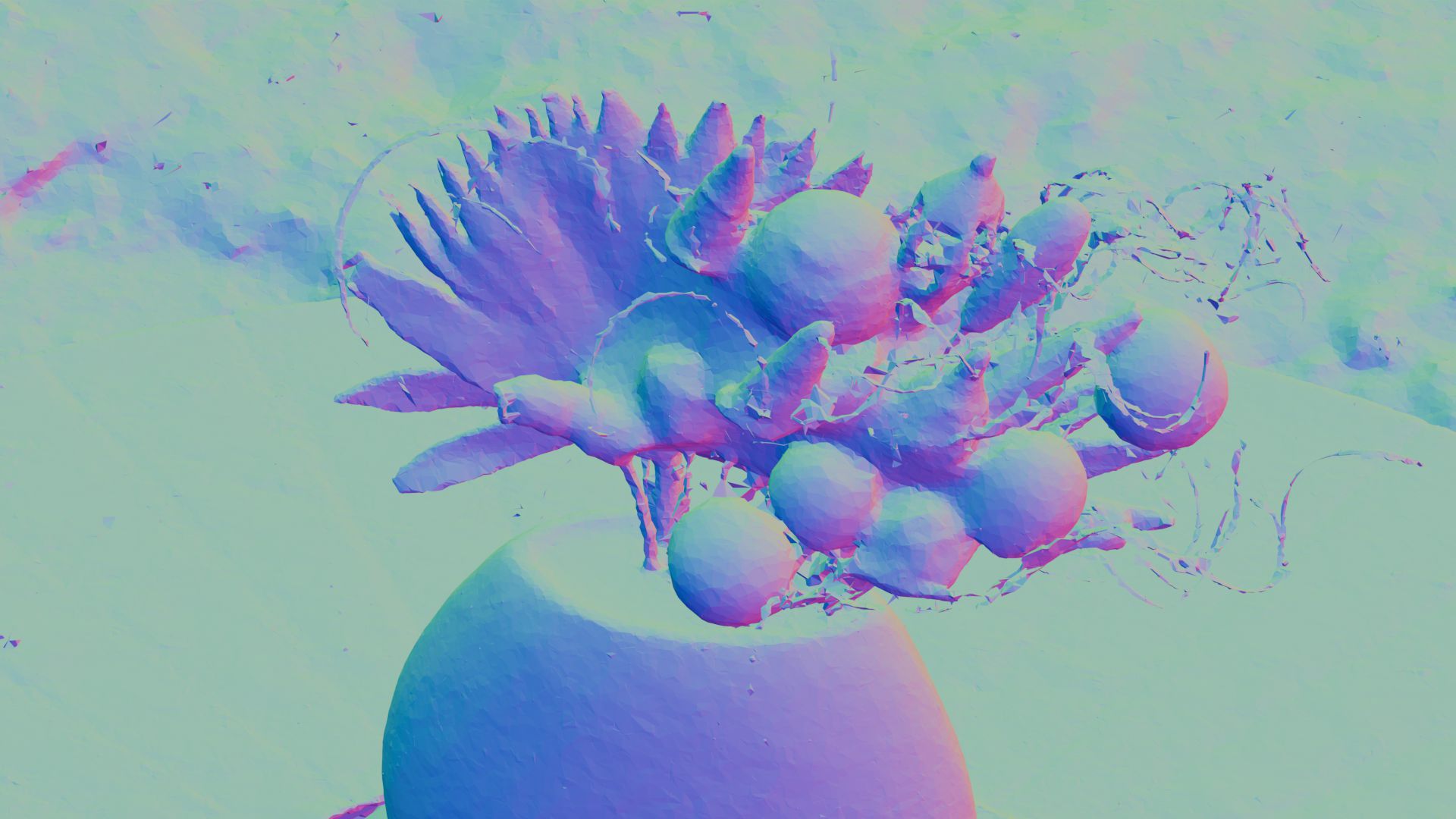}

}
\vspace*{-6mm}
\caption*{(c) MILo} 
\vspace*{1mm}

\resizebox{\linewidth}{!}{
    \includegraphics[width=0.20\linewidth]{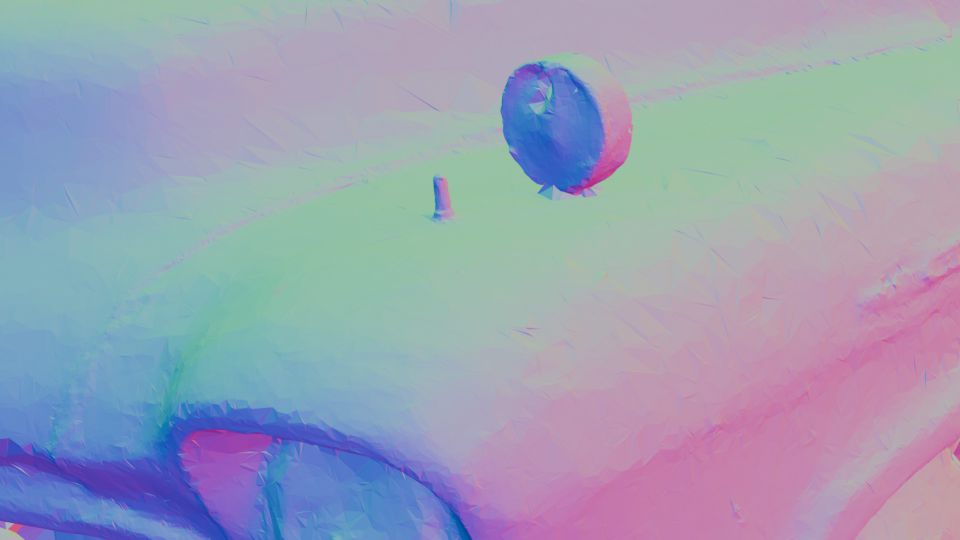} 
    \includegraphics[width=0.20\linewidth]{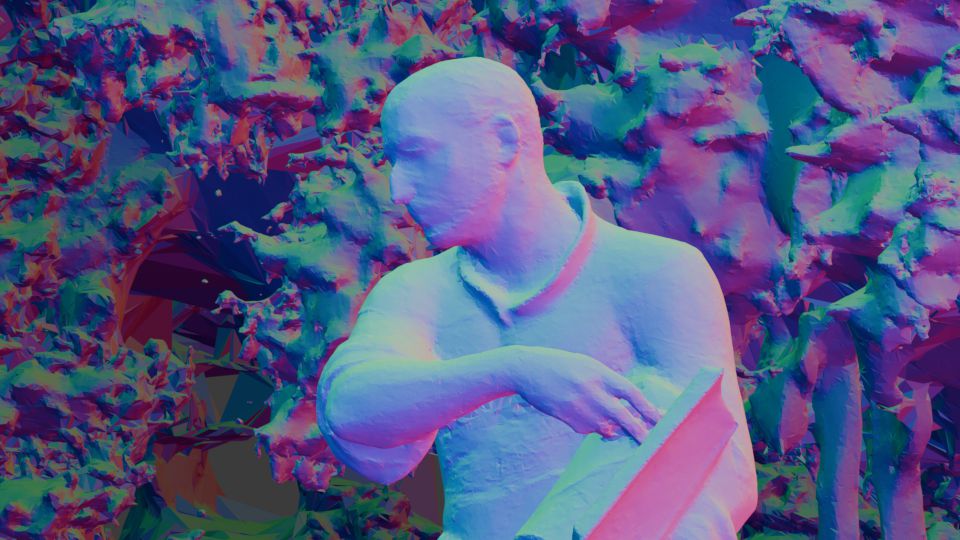} 
    \includegraphics[width=0.20\linewidth]{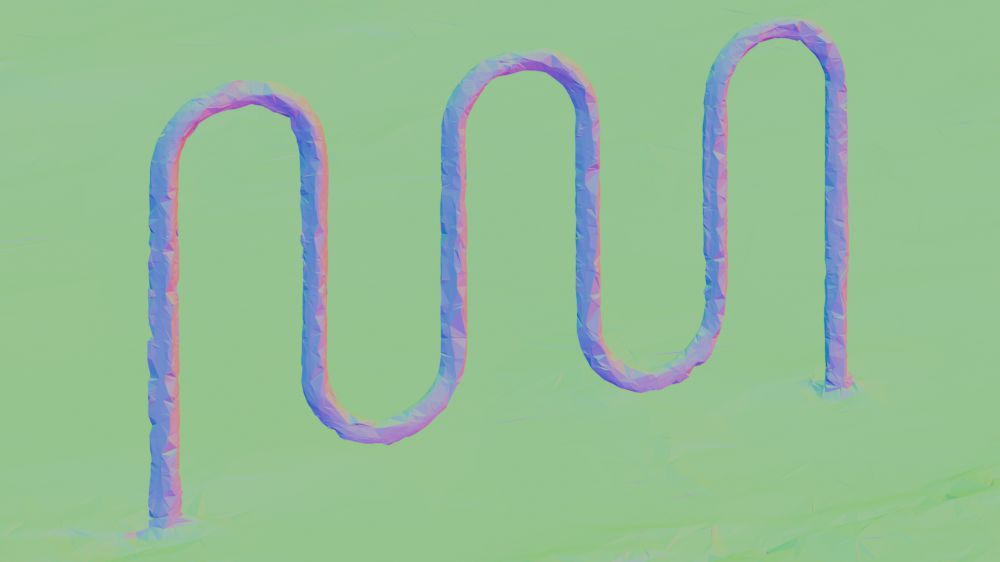}
    \includegraphics[width=0.20\linewidth]{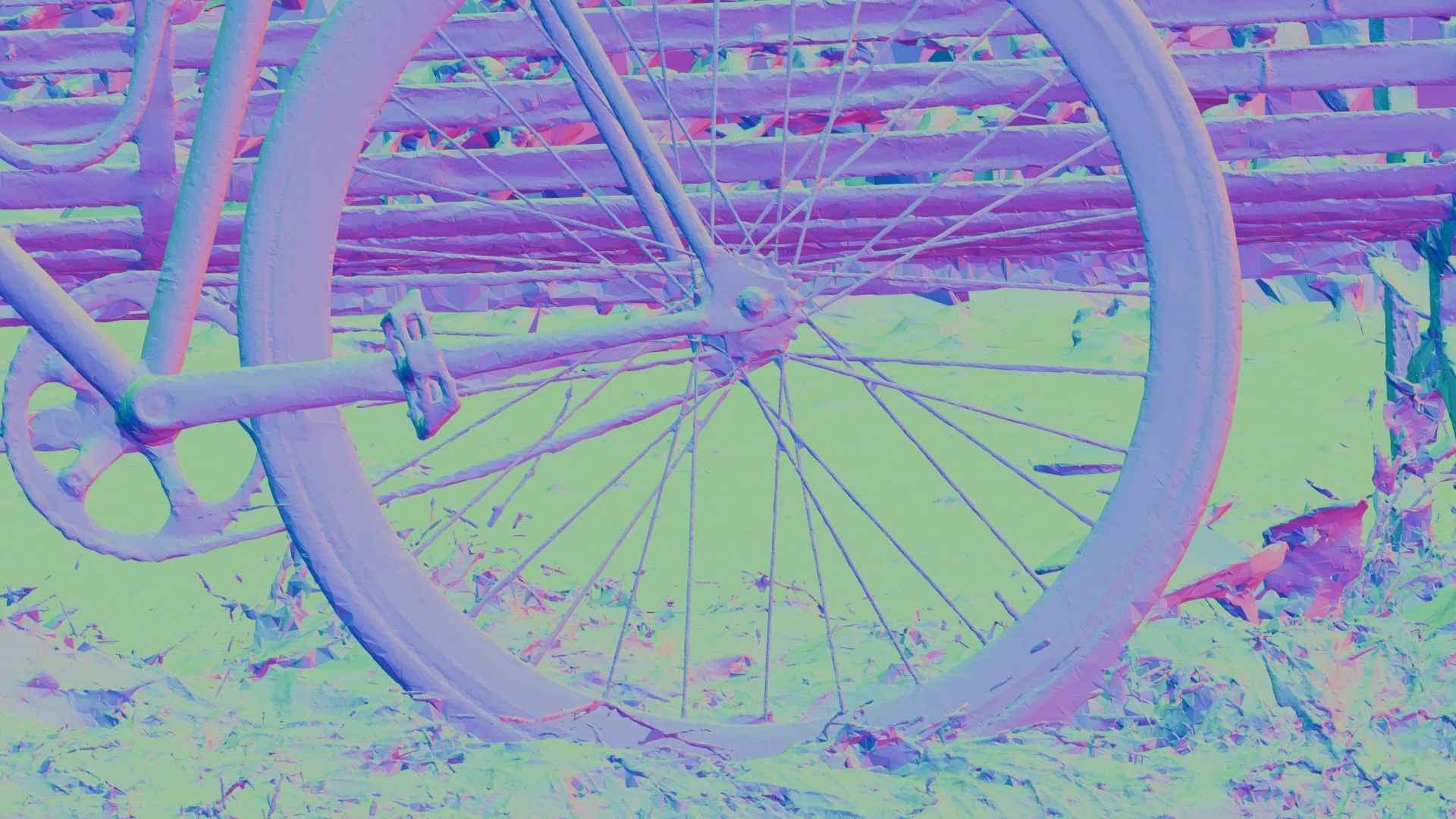}
        \includegraphics[width=0.20\linewidth]{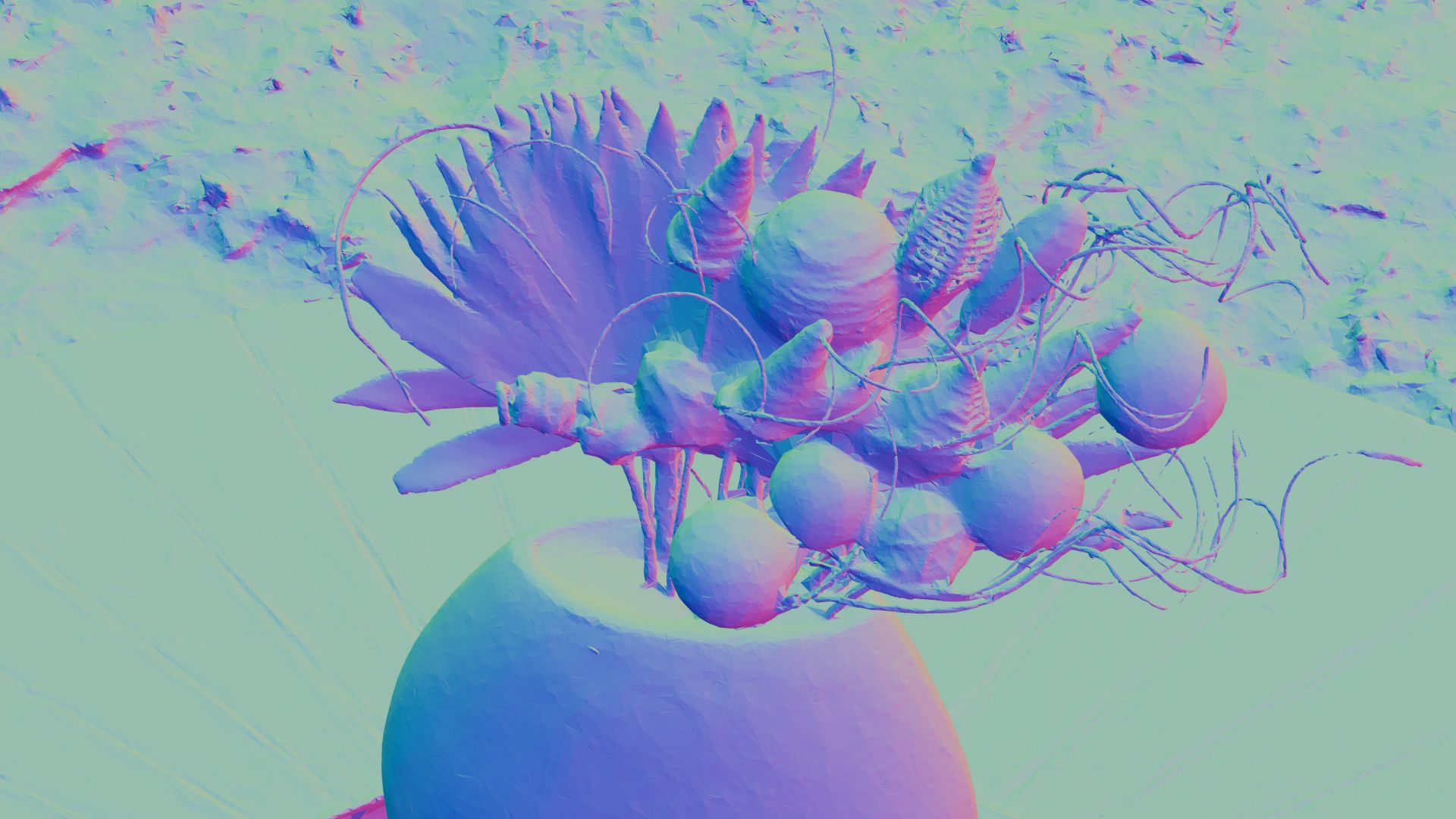}

}
\vspace*{-6mm}
\caption*{(d) Ours} 
\vspace*{-2mm}

\caption{\textbf{Qualitative comparison on T\&T and MipNerf360.} Our method recovers significantly finer surface detail and thin structures compared to baselines. Close-up insets highlight regions where competing methods erode geometry or fail to close holes.
\vspace*{-8mm}
}

\label{fig:qualitative}
\end{figure*}

In this section, we evaluate and ablate \methodname{} on standard benchmarks. We expose a fundamental bias in the standard evaluation protocol for Gaussian-based surface reconstruction, and propose two complementary evaluations in \cref{sec:revisiting_evaluation}: \textit{Uniform Sampling} and \textit{Virtual Scanning}. We refer the reader to the appendix for all of our implementation details.

\vspace{-3mm}
\subsection{Experimental Setup}
\paragraph{Datasets \& Evaluations.} We perform quantitative and qualitative evaluations on the widely used \textbf{DTU}~\cite{jensen2014large} and \textbf{Tanks and Temples (T\&T)}~\cite{Knapitsch2017} datasets. Furthermore, we leverage the Mesh-Based Rendering (MBR) evaluation introduced in MILo~\cite{guedon2025milo} to evaluate the mesh completeness.

\vspace{-2mm}
\paragraph{Baselines.} We compare our method against state-of-the-art explicit surface reconstruction techniques. We separate these into two categories: Methods capable of full scene reconstruction (GOF~\cite{yu2024gaussian}, SOF~\cite{radl2025sof}, MILo~\cite{guedon2025milo}, RaDe-GS~\cite{zhang2024rade}, GGGS~\cite{Zhang2026GeometryGrounded}), and methods focusing on foreground only (PGSR~\cite{chen2024pgsr}, 2DGS~\cite{huang20242d}).

\vspace{-2mm}
\paragraph{Metrics.} We report Chamfer Distance (CD) on the DTU dataset; F1-score on the T\&T dataset; and image metrics for NVS on the Mip-NeRF~360 dataset. Moreover, we present image metrics for the MBR evaluation on T\&T and Mip-NeRF~360 datasets. For \textit{Uniform Sampling}, \textit{Virtual Scanning}, and \textit{Mesh-Based Rendering}, we run all baseline methods ourselves under the exact same protocol.

\vspace{-3mm}
\subsection{Revisiting Surface Evaluation for Gaussian Splatting}
\label{sec:revisiting_evaluation}

The standard evaluation practice on T\&T and structured scan datasets is fundamentally flawed for two reasons: (1) the conventional approach builds the predicted surface point cloud from the vertices and face centers of the extracted mesh, biasing metrics toward denser tessellations; 

and (2) ground data (GT) acquisition, relying on laser scan, unfairly penalizes methods that correctly reconstruct occluded geometry and surfaces visible at a grazing angle. 
To resolve this, we introduce two robust evaluation strategies (\cref{tab:tandt_combined}).

\vspace{-0.2cm}
\paragraph{Uniform Sampling Evaluation.} To eliminate vertex-density bias, we uniformly sample a fixed number of points from the reconstructed mesh surface within the GT crop volume, ensuring a strictly fair geometric comparison.

\vspace{-0.2cm}
\paragraph{Virtual Scanning Evaluation.} Uniform sampling alone does not take into account incomplete meshes that mimic GT bias. We thus propose to simulate the original data acquisition process by rendering depth maps of the reconstructed meshes from the input camera poses and back-projecting them into the GT crop volume. 

\begin{table*}[!ht]
    \centering
    \caption{\textbf{Quantitative comparison on the Tanks \& Temples Dataset~\cite{Knapitsch2017}.}
    F1 scores under two complementary evaluation protocols.
    \textit{Uniform Sampling}: points are uniformly sampled from the mesh surface, eliminating vertex-count bias.
    \textit{Virtual Scanning}: depth maps are back-projected from the original camera viewpoints, mirroring the ground truth acquisition process. We report average training time.
    }
    \vspace{-0.2cm}
    \footnotesize 
    \setlength{\tabcolsep}{4pt} 
    \resizebox{0.98\linewidth}{!}{
    \begin{tabular}{@{}l|ccc|cccccc}\label{tab:tandt_combined}
     & \multicolumn{3}{c|}{Foreground only} & \multicolumn{6}{c}{Full scene extraction} \\
     & 2DGS & PGSR & \textbf{Ours (PAM)} & GOF & SOF & RaDe-GS & MILo & GGGS & \textbf{Ours (2p)} \\
     \hline
     \multicolumn{10}{l}{\textit{Uniform Sampling Evaluation}} \\  
     \hline
    Barn        & \tbestbis 0.48 & \bestbis 0.65  & \sbestbis 0.63 & 0.40 & \trd 0.45 & 0.42        & 0.40        & \snd 0.57 & \fst 0.64 \\
    Caterpillar & \tbestbis 0.24 & \bestbis 0.45  & \sbestbis 0.38 & 0.23 & 0.20      & 0.22        & \trd 0.24   & \fst 0.41 & \fst 0.41 \\
    Courthouse  & \sbestbis 0.13 & \bestbis 0.24  & \tbestbis 0.07 & 0.10 & 0.07      & \snd 0.12   & \trd 0.11   & 0.10      & \fst 0.19 \\
    Ignatius    & \tbestbis 0.51 & \bestbis 0.81  & \sbestbis 0.76 & 0.55 & 0.68      & 0.66        & \trd 0.72   & \fst 0.79 & \snd 0.77 \\
    Meetingroom & \tbestbis 0.18 & \bestbis 0.33  & \sbestbis 0.28 & 0.19 & 0.15      & \trd 0.22   & 0.14        & \snd 0.36 & \fst 0.37 \\
    Truck       & \sbestbis 0.46 & \bestbis 0.63  & \sbestbis 0.46 & 0.41 & 0.39      & \trd 0.45   & 0.43        & \fst 0.50 & \fst 0.50 \\
    \hline
    Mean        & \tbestbis 0.33 & \bestbis 0.52  & \sbestbis 0.43 & 0.31 & 0.32      & \trd 0.35   & 0.34        & \snd 0.45 & \fst 0.48 \\
    \hline
     \multicolumn{10}{l}{\textit{Virtual Scanning Evaluation}} \\  
     \hline
    Barn        & \tbestbis 0.51 & \sbestbis 0.65 & \bestbis 0.67  & 0.49 & 0.49      & 0.47        & \trd 0.52   & \snd 0.62 & \fst 0.67 \\
    Caterpillar & \tbestbis 0.31 & \bestbis 0.50  & \bestbis 0.50  & 0.29 & 0.31      & 0.30        & \trd 0.33   & \fst 0.52 & \snd 0.50 \\
    Courthouse  & \tbestbis 0.10 & \bestbis 0.18  & \sbestbis 0.16 & 0.11 & 0.06      & 0.11        & \trd 0.12   & \snd 0.16 & \fst 0.18 \\
    Ignatius    & \tbestbis 0.61 & \bestbis 0.82  & \sbestbis 0.78 & 0.63 & 0.63      & 0.66        & \trd 0.74   & \fst 0.80 & \snd 0.78 \\
    Meetingroom & \tbestbis 0.19 & \sbestbis 0.35 & \bestbis 0.41  & 0.24 & 0.23      & \trd 0.25   & 0.23        & \fst 0.42 & \snd 0.41 \\
    Truck       & \tbestbis 0.54 & \bestbis 0.66  & \sbestbis 0.63 & 0.50 & 0.48      & \trd 0.57   & 0.56        & \fst 0.65 & \snd 0.63 \\
    \hline
    Mean        & \tbestbis 0.38 & \bestbis 0.53  & \bestbis 0.53  & 0.38 & 0.37      & 0.39        & \trd 0.42   & \fst 0.53 & \fst 0.53 \\
    \hline
    \hline
    Times        & \bestbis 12 m & \tbestbis 45 m & \sbestbis 27 m & 69 m & \sbest 17 m             &\best 12 m& 150 m     & 32 m & \tbest 27 m \\
    \hline
    \end{tabular}
    }
    \vspace*{-4.5mm}
\end{table*}

\vspace{-0.2cm}
\paragraph{Legacy Evaluation.} For backward compatibility, we report the standard vertex-based metric. However, this metric is highly sensitive to tessellation density (e.g., extracting our meshes with 9 pivots instead of 2 artificially inflates scores). Please see the appendix for a detailed discussion of these flaws.

\vspace{-0.1cm}
\subsection{Main Results}

\vspace{-0.1cm}
\paragraph{Geometry on DTU.} We report Chamfer Distance across 15 scenes of the DTU dataset in the appendix. Our method yields highly competitive results, outperforming most prior methods and achieving scores comparable to GGGS~\cite{Zhang2026GeometryGrounded}. Importantly, our method does not exhibit the surface erosion artifacts observed in GGGS, as illustrated in \cref{fig:qualitative}. We note that we outperform foreground only methods such as 2DGS and PGSR on this dataset.

\vspace{-2mm}
\paragraph{Geometry on T\&T.} We evaluate on the T\&T dataset under three complementary protocols, all reported in \cref{tab:tandt_combined}.
\begin{itemize}
    \item \textit{Uniform Sampling Evaluation.} 
    
    Our method sets a new state of the art under this unbiased protocol among the full scene extraction methods. PGSR~\cite{chen2024pgsr} remains an apparent outlier: its use of TSDF fusion and depth filtering yields non-watertight meshes whose holes coincide with those of the ground truth scan, an artifact of the acquisition process rather than genuine geometric quality. We document this in the appendix. We note that Ours (PAM)'s performance takes a toll on this metric. Uniformly sampling points to reconstruct scenes such as Courthouse or Meetingroom is far from optimal. We highlight that the advantage of this meshing procedure is the meshing of regions of interest at arbitrary resolution~(\cref{fig:primal_ablation_meshing_zoom}).
    \item \textit{Virtual Scanning Evaluation.} Simulating the acquisition bias of the ground truth via our virtual scanning procedure, we set a new state of the art alongside GGGS, and see that PGSR has no longer an advantage when this bias is taken into account.
\end{itemize}

\vspace{-3mm}
\paragraph{Mesh-Based Rendering on T\&T and Mip-NeRF~360.} To complement our geometric evaluation, we report Mesh-Based Rendering (MBR) metrics~\cite{guedon2025milo}, which assesses background reconstruction quality and serve as a proxy for mesh completeness. Results are shown in \cref{tab:mbr_metrics}. Despite lacking the dedicated rendering prior of MILo~\cite{guedon2025milo}, our method performs competitively across all datasets and sets a new state of the art on Mip-NeRF~360. This evaluation also exposes a critical weakness of PGSR: its depth-filtering strategy leaves substantial holes in background geometry, resulting in severely degraded rendering quality—confirming that its strong geometric scores are an artifact of the evaluation protocol rather than true reconstruction quality.

\vspace{-0.1cm}
\paragraph{Novel View Synthesis on Mip-NeRF~360.} Although NVS is not the primary objective of our method, we achieve competitive performance on Mip-NeRF~360~\cite{barron22mipnerf360} and set a new state of the art for outdoor scenes, reflecting the effectiveness of our representation in challenging, unbounded settings. Results are reported in the appendix.

\vspace*{-3mm}
\begin{table*}[!ht]
    \centering
    \caption{\textbf{Mesh-Based Novel View synthesis metrics on real-world datasets}. We report PSNR, SSIM, and LPIPS metrics, as well as the total number of Gaussians and vertices (in millions). Our method demonstrates superior performance on these challenging unbounded scenes.}
    \vspace*{-3mm}
    \footnotesize 
    \setlength{\tabcolsep}{4pt} 
    \resizebox{0.99\linewidth}{!}{
    \begin{tabular}{@{}l|ccccc|ccccc}\label{tab:mbr_metrics}
     & \multicolumn{5}{c@{}|}{MipNeRF~360} & \multicolumn{5}{c@{}}{Tanks~\&~Temples} \\
     & PSNR $\uparrow$ & SSIM $\uparrow$ & LPIPS $\downarrow$ & \#Gaussians $\downarrow$ & \#Verts $\downarrow$ & PSNR $\uparrow$ & SSIM $\uparrow$ & LPIPS $\downarrow$ & \#Gaussians $\downarrow$ & \#Verts $\downarrow$ \\
     \hline
     \multicolumn{11}{l}{\textit{Foreground only}} \\
     \hline
     2DGS
     & \sbestbis 15.36 & \sbestbis 0.4987 & \sbestbis 0.4749 & \bestbis 1.88 & \bestbis 4.31
     & \sbestbis 14.23 & \sbestbis 0.5697 & \sbestbis 0.4854 & \bestbis 0.98 & \sbestbis 16.39 \\
     PGSR
     & \bestbis 15.63 & \bestbis 0.5834 & \bestbis 0.4509 & \sbestbis 3.16 & \sbestbis 22.78
     & \bestbis 14.55 & \bestbis 0.5956 & \bestbis 0.4819 & \sbestbis 1.47 & \bestbis 11.79 \\
     \hline
     \multicolumn{11}{l}{\textit{Full scene extraction}} \\
     \hline
     GOF
     & 24.25 & 0.7017 & 0.3454 & \tbest 2.99 & 32.80
     & 20.10 & 0.6475 & 0.4073 & \tbest 1.25 & 11.63 \\
     RaDe-GS
     & \tbest 24.84 & 0.7291 & 0.3128 & \sbest 2.91 & \tbest 30.85
     & \tbest 20.70 & 0.6767 & 0.3876 & \sbest 1.18 & \tbest 10.06 \\
     MILo
     & \sbest 25.075 & \tbest 0.7339 & \tbest 0.3096 & \best 0.46 & \best 6.73
     & \best 21.198 & \tbest 0.6908 & \tbest 0.3782 & \best 0.28 & \best 4.36 \\
     GGGS
     & 24.55 & \sbest 0.7386 & \sbest 0.2986 & 4.08 & 43.46
     & \tbest 20.70 & \sbest 0.6912 & \sbest 0.3734 & 1.73 & 21.81 \\
     Ours
     & \best 25.10 & \best 0.752 & \best 0.289 & 4.0 & \sbest 11.79
     & \sbest 20.98 & \best 0.7004 & \best 0.3674 & 2.07 & \sbest 5.80\\
    \end{tabular}
    }
    \vspace*{-4mm}
\end{table*}

\vspace{-0.3cm}
\paragraph{Qualitative Comparisons.} \cref{fig:qualitative} presents visual comparisons against all baselines. \methodname{} consistently produces surfaces with sharper fine details, fewer topological artifacts, and better coverage of thin structures such as bicycle spokes, corroborating the quantitative findings.

\paragraph{Conclusion.} Our results demonstrate that \methodname{} achieves robust state-of-the-art performance across all evaluation settings. While methods such as PGSR appear competitive under specific protocols by exploiting biases in GT data, the TSDF extraction limits their quality. Our method performs consistently across both legacy and newly proposed metrics, and is further validated by complementary rendering-based and qualitative evidence. This consistency confirms our improvements reflect genuine advances in surface reconstruction.

\subsection{Ablation Studies}

All ablation studies are conducted on T\&T, reporting F1 scores and MBR metrics (SSIM, PSNR, LPIPS). More details are provided in appendix.

\paragraph{Normal Alignment Losses.} \cref{tab:ablation_combined} reports results with and without our normal alignment losses. F1 scores remain stagnant---our rasterizer already places Gaussians near the isosurface---but MBR metrics improve substantially: our losses align Gaussians with the surface, enabling clean mesh extraction with only 2 pivots. Without them, fine details such as thin structures and sharp edges are lost. This also serves as an illustration that the T\&T dataset and ground truth is best to measure coarse geometry alignment, and fails to measure detail loyalty. 

\begin{table*}[t]
    \centering
    \caption{\textbf{Ablation studies on Tanks \& Temples.}
    \textit{Top}: Mean scores across all scenes for our component ablation, reporting F1-score under both evaluation protocols and Mesh-Based Rendering (MBR) metrics.
    \textit{Bottom}: Per-scene Virtual Scan F1-score demonstrating that our contributions generalize as a plug-in regularizer to RaDe-GS~\cite{zhang2024rade}. The normal alignment loss and densification procedure constitute the Gaussian Wrapping approach. 
    }
    \label{tab:ablation_combined}
    \vspace*{-3mm}
    \setlength{\tabcolsep}{3pt} 
    \renewcommand{\arraystretch}{1.1}
    \footnotesize 
    
    \resizebox{\linewidth}{!}{%
    \begin{tabular}{@{}l|cc|ccc@{}}
    \multicolumn{6}{l}{\textit{Component Ablation (Mean across scenes)}} \\
    \hline
    & \multicolumn{2}{c|}{\textbf{F1-Score} $\uparrow$} & \multicolumn{3}{c}{\textbf{MBR Metrics}} \\
    \cline{2-3}\cline{4-6}
    \textbf{Configuration} & \textbf{Virt.\ Scan} & \textbf{Mesh Qual.} & \textbf{PSNR} $\uparrow$ & \textbf{SSIM} $\uparrow$ & \textbf{LPIPS} $\downarrow$ \\
    \hline
    Baseline                 & \fst0.53 & \fst\fst0.48  & \trd 20.74 & \trd0.6958  & \trd0.3718          \\
    + Normal Alignment Loss & \snd 0.52 & \fst0.48  & \snd20.78    & \snd0.6986   & \snd0.3683          \\
    + Normal Alignment Loss + Densification          & \fst0.53 & \fst0.48  & \fst20.98   & \fst0.7004   & \fst0.3674 \\
    \hline
    \end{tabular}}

    \smallskip

    \resizebox{\linewidth}{!}{%
    \begin{tabular}{@{}l|cccccc|c@{}}
    \multicolumn{8}{l}{\textit{Gaussian Wrapping Generalization to RaDe-GS (Virtual Scan F1-Score)}} \\
    \hline
    \textbf{Configuration} & Barn & Caterpillar & Courthouse & Ignatius & Meetingroom & Truck & \textbf{Mean} \\
    \hline
    RaDe-GS                              & \snd0.47          & \snd0.30          & \snd0.11          & \snd0.66          & \snd0.25          & \snd0.57          & \snd0.39          \\
    RaDe-GS + Gaussian Wrapping & \fst0.63 & \fst0.46 & \fst0.13 & \fst0.71 & \fst0.34 & \fst0.60 & \fst0.48 \\
    \hline
    \end{tabular}}
    \vspace*{-3mm}
\end{table*}

\vspace{-0.1cm}
\paragraph{Generalizability of Gaussian Wrapping.} To confirm that the Gaussian Wrapping approach is architecture-agnostic, we plug it into RaDe-GS~\cite{zhang2024rade}. The consistent improvement across all T\&T scenes (\cref{tab:ablation_combined}; and figure in the appendix) demonstrates that our method operates as an effective drop-in regularizer well beyond our specific pipeline. Note that the difference between the RaDe-GS + Gaussian Wrapping combination and Ours in previous tables is the use of our depth rasterization, which precisely queries the 0.5-isosurface via binary search.

\vspace{-0.5cm}
\section{Conclusion}\label{sec:conclusion}

\vspace{-0.2cm}
We presented \methodname{}, a principled method grounded on the OaV framework. Our method reconstructs complete scenes including extremely thin structures where prior methods fail, without surface erosion artifacts, and at a fraction of the mesh weight of competitors. We further expose fundamental biases in standard evaluation protocols and propose two alternatives. Finally, Primal Adaptive Meshing enables extraction of meshes with better topology and controllable resolution.

\vspace{-0.2cm}
\paragraph{Limitations and Future Work.}
Our Primal Adaptive Meshing currently relies on uniform sampling of an initial MTet mesh, which can be suboptimal for highly detailed scenes. Developing more principled sampling strategies---for instance, guided by the Gaussian Vector Field or local surface curvature---would extend the approach to finer and more intricate geometry. Moreover, our attenuation-based formulation is not specific to 3DGS: extending it to more accurate volumetric rendering models such as EVER~\cite{mai2025ever} is a natural and promising direction.

\section{Acknowledgements}

We are grateful to George Drettakis for insightful discussions. Parts of this work were supported by the ERC Consolidator Grant ``VEGA'' (No. 101087347), the ERC Advanced Grant ``explorer'' (No. 101097259), as well as gifts from Ansys Inc., and Adobe Research.

%
%
\clearpage
\bibliographystyle{splncs04}
\bibliography{main}

\clearpage
\appendix
\section{Overview of the Appendix}\label{sec:overview_appendix}

This appendix serves three purposes: \textbf{(i)} provide proofs for the theoretical claims made in the main paper; \textbf{(ii)} provide the implementation details required to reproduce our results; and \textbf{(iii)} showcase additional results, including quantitative tables, qualitative figures, and an illustration of the evaluation bias we expose.
Throughout the theory sections, key results are highlighted in \colorbox{dblue}{\strut blue boxes} and their corresponding proofs in \colorbox{dorange}{\strut orange boxes}.

\paragraph{Section~\ref{sec:derivations}} provides details on the theoretical results presented in the paper.
We give further proof and motivation for our choice of attenuation and normal field. We also illustrate how the transmittance introduced in GGGS~\cite{Zhang2026GeometryGrounded} leads to surface erosion.

\paragraph{Section~\ref{sec:mesh_extraction_details}} presents the derivation of the lower-bound for estimating vacancy (\cref{sec:mesh_extraction_details}), and the Newton projection step used in our Primal Adaptive Meshing procedure (\cref{prop:newton}).

\paragraph{Section~\ref{sec:experiments_details}} provides additional quantitative results (DTU, legacy evaluation on T\&T, Novel View Synthesis on Mip-NeRF~360), examples of the evaluation bias discussed in the main paper with a concrete code snippet from the RaDe-GS repository (\cref{fig:code_flaw}).

\section{Deriving an Equivalent Ray-Marching Formulation}\label{sec:derivations}

Objects as Volumes (OaV)~\cite{miller2024objectsasvolumes} establishes a strong theoretical framework exploring the relationship between surface extraction and volumetric rendering. Specifically, the results introduced in OaV only hold for volumetric rendering models relying on ray-marching through an attenuation field $\sigma$ (also called \textit{density} field), such as Neural Radiance Fields~\cite{mildenhall2020nerf}.

Even though 3D Gaussian Splatting relies on volumetric primitives, its image formation model is based on rasterization: volumetric particles are first projected onto the image plane as flat 2D Gaussians, then blended together to obtain pixel colors.
This formulation is not strictly equivalent to volumetric ray-marching, and theoretical results from OaV cannot be directly applied to Gaussian Splatting. 

While rasterization allows for incredibly fast rendering, it is less accurate than methods relying on ray-marching through attenuation. Still, in practice, it can be shown~\cite{celarek2025does} that given enough Gaussians, the lack of accuracy has minor impact on the novel view synthesis quality of 3DGS representations.

The goal of this section is to propose, motivate and justify a volumetric ray-marching formulation equivalent to the image formation model of 3D Gaussian Splatting under the right set of assumptions. 
Building on this equivalence, we then apply the results of OaV to 3D Gaussian Splatting, enhancing surface extraction from Gaussians.

\subsection{Deriving Attenuation}\label{sec:derivation_attenuation}
Let us start by deriving the attenuation in this formulation.

\hypbox[]{
\begin{lemma}[Transmittance of a single Gaussian.]\label{lemma:transmittance_single_gaussian}
    The image formation model of a Gaussian Splatting representation composed of a single Gaussian, which relies on alpha-blending with opacity, is equivalent to a ray-marching volumetric rendering model with the following continuous transmittance function: for any ray origin $\rayo\in\IR^3$ in empty space, direction $\rayw\in\calS^2$, and flight distance $t\in [0,+\infty)$:
    \begin{equation}
        \trans(t) = 1 - G(\rayo + \min(t, \tstar)\rayw) \> ,
    \end{equation}
    where $\tstar := \arg \max_{t\geq 0} \{G(\rayo + t\rayw)\}$.
    The corresponding attenuation coefficient is view-dependent and is given by:
    \begin{equation}
        \sigma(\rayx, \rayw) = \max\left(
            0,
            -\rayw \cdot \nabla \log (1-G(\rayx))
        \right) \> .
    \label{eq:gaussian_splatting_attenuation_single}
    \end{equation}
\end{lemma}
}
\begin{figure}[t]
\centering
\begin{subfigure}[b]{\linewidth}
    \centering
    \includegraphics[width=0.75\linewidth]{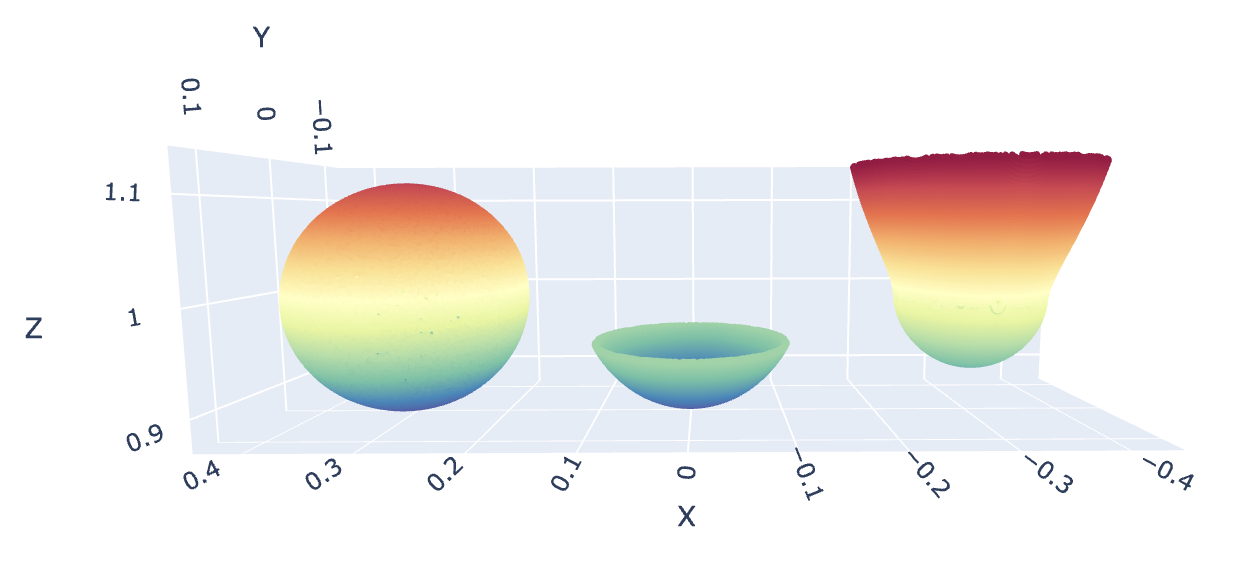}
    \caption*{\quad \quad \quad\quad (a) Gaussian  \quad \quad \quad  (b) Ours \quad  \quad  \quad  \quad (c) GGGS~\cite{Zhang2026GeometryGrounded}}
\end{subfigure}

\begin{subfigure}[b]{0.19\linewidth}
    \centering
    \includegraphics[width=\linewidth]{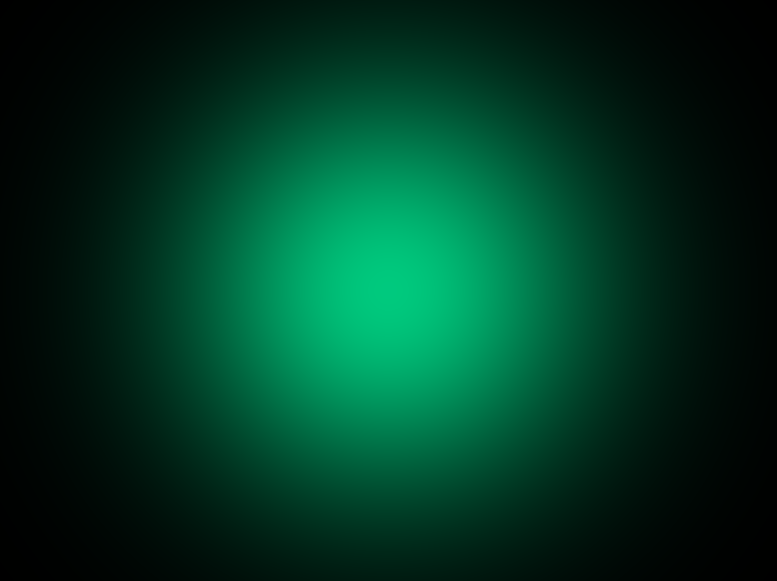}
    \caption*{(d) RGB render}
\end{subfigure}
\begin{subfigure}[b]{0.19\linewidth}
    \centering
    \includegraphics[width=\linewidth]{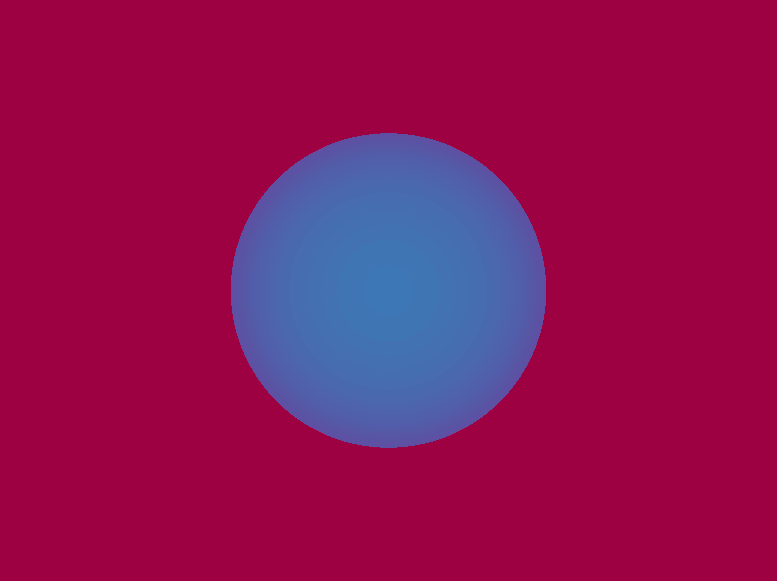}
    \caption*{(e) Depth (Ours)}
\end{subfigure}
\begin{subfigure}[b]{0.19\linewidth}
    \centering
    \includegraphics[width=\linewidth]{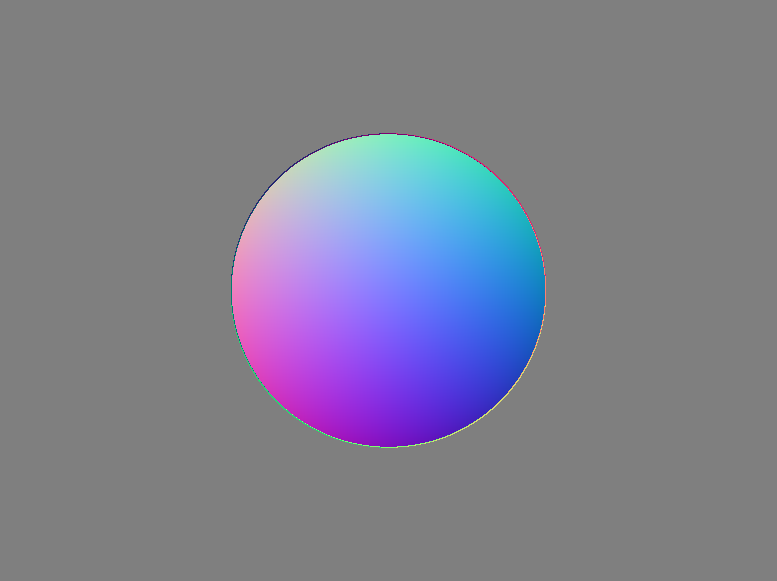}
    \caption*{(f) Normal (Ours)}
\end{subfigure}
\begin{subfigure}[b]{0.19\linewidth}
    \centering
    \includegraphics[width=\linewidth]{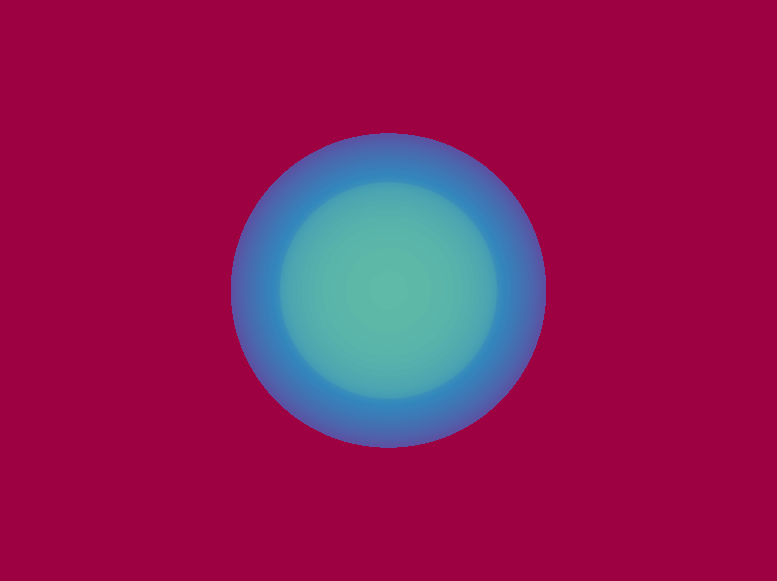}
    \caption*{(g) Depth~\cite{Zhang2026GeometryGrounded}}
\end{subfigure}
\begin{subfigure}[b]{0.19\linewidth}
    \centering
    \includegraphics[width=\linewidth]{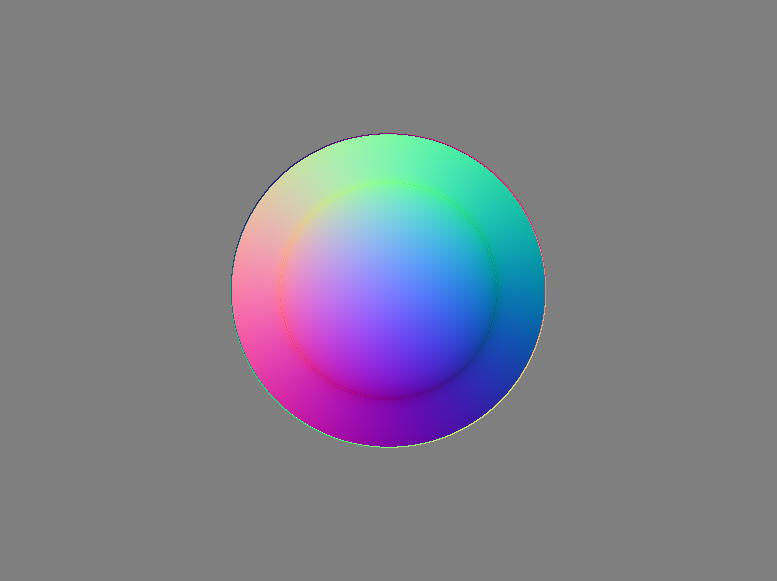}
    \caption*{(h) Normal~\cite{Zhang2026GeometryGrounded}}
\end{subfigure}
\caption{\textbf{Comparison between our continuous attenuation model, and the one proposed in GGGS~\cite{Zhang2026GeometryGrounded}.} \textbf{Top figure:} The sphere illustrated on the left (a) represents the $0.5$-isosurface of a single isotropic Gaussian. Given a camera located below the Gaussian and looking up, we represent the corresponding $0.5$-isosurface of the transmittance for both our continuous attenuation model (b) and GGGS~\cite{Zhang2026GeometryGrounded} (c). Our method produces unbiased, spherical isosurfaces aligned with the visible part of the actual Gaussian isosurface. On the contrary, GGGS produces non-multi-view consistent surfaces misaligned with the Gaussian isosurface, leading to erosion. \textbf{Bottom row} illustrates the actual RGB rendering, depth maps and normal maps obtained for both methods.
}
\label{fig:comparison_transmittance}
\end{figure}

\paragraph{Sketch of proof.} Our model is uniquely determined by enforcing three conditions on the attenuation coefficient:
\begin{enumerate}
    \item First, applying ray-marching (as in NeRF~\cite{mildenhall2020nerf}) with the attenuation-based model should produce the same rendered image as Gaussian Splatting.
    \item Second, the transmittance function should reach its maximum value at the point of maximum contribution along the ray $\rayx = \rayo + \tstar \rayw$.
    \item Third, following the definition of transmittance from Objects as Volumes~\cite{miller2024objectsasvolumes}, we interpret transmittance as visibility probability along a ray. Therefore, we assume an isosurface of the transmittance function for a single Gaussian should be a portion of an ellipsoid centered at the Gaussian center.
\end{enumerate}
We note that, transmittance having ellipsoidal isosurfaces is not a simplification, and uniquely determines the form of the attenuation coefficient in Eq.~\ref{eq:attenuation_gaussian_splatting}. 
Removing this assumption, another natural and simpler possible attenuation can be proposed: $\sigma(\rayx, \rayw) = \frac{1}{2}\left|\rayw \cdot \nabla \log (1-G(\rayx))\right|$.
This attenuation factor corresponds to the continuous transmittance introduced in the concurrent work~\cite{Zhang2026GeometryGrounded}. It does not enforce the aforementioned geometric constraint on isosurfaces.
\cref{fig:comparison_transmittance} illustrates the difference between our transmittance and~\cite{Zhang2026GeometryGrounded}: For a single Gaussian, we rendered the depth as the $0.5$-isosurface of the transmittance, as proposed in~\cite{Zhang2026GeometryGrounded}. The simpler attenuation from~\cite{Zhang2026GeometryGrounded} produces non-spherical surfaces for a single Gaussian. This results in (a) non-multi-view consistent depth maps that are largely misaligned with the actual Gaussian; and (b) eroded surfaces. The qualitative comparison presented in the main paper and in~\cref{fig:comparison_transmittance} illustrates how this misalignment erodes the extracted surfaces.

\proofbox[]{
\begin{proof}
We first derive a continuous transmittance expression for a single Gaussian by imposing consistency with alpha blending while following the reasoning of~\cite{miller2024objectsasvolumes}. We then verify equivalence in the rendering result between alpha blending and the proposed density-based model.

Following~\cite{yu2024gaussian,zhang2024rade}, the alpha-blending formula models the contribution of a single Gaussian to the transmittance as being constant after the point of maximum Gaussian contribution along the ray. Hence, for a scene composed of a single Gaussian $G$:
\begin{equation}
    \forall t\geq\tstar \> , \> \trans(t) = \trans(\tstar) = 1 - G(\rayo + \tstar \rayw) \> .
\end{equation}
We now determine how $\trans$ interpolates from its initial value to $\trans(\tstar)$ before reaching the maximum-contribution point.

Following~\cite{miller2024objectsasvolumes}, we interpret transmittance as the visibility probability along a ray, \ie, the probability of a point not being occluded by the scene along the ray. Therefore, for a fixed ray origin $\rayo$ and scalar $\lambda>0$, the $\lambda$-isosurface of $(\rayw, t) \mapsto \trans(t)$ corresponds to the visible part of a candidate scene surface at decision level $\lambda$.

To simplify the analysis, we temporarily assume that the Gaussian $G$ is isotropic. In this case, the scene surface is assumed to be spherical around $\mean$, such that each isosurface of the transmittance function $(\rayw, t) \mapsto \trans(t)$ contains the visible part of a sphere centered at $\mean$.

For any ray direction $\rayw$ and flight distance $t<\tstar$, the point $\rayo + t\rayw$ is located on the $\lambda$-isosurface of $(\rayw, s) \mapsto \trans(s)$ with $\lambda = \trans(t)$. There exists at least one direction $\rayw'$ and flight distance $t'$ such that $\rayo + t'\rayw'$ is located on the same $\lambda$-isosurface and $\rayo + t'\rayw'$ is the point with maximum contribution along the ray with direction $\rayw'$.

It follows that $\trans(t) = T_{\rayo, \rayw'}(t') = 1 - G(\rayo + t'\rayw')$. Since the Gaussian is isotropic and the spherical $\lambda$-isosurface is centered at $\mean$, we have $||\rayo + t\rayw - \mean|| = ||\rayo + t'\rayw' - \mean||$, such that $G(\rayo + t\rayw) = G(\rayo + t'\rayw')$.

We conclude that, for any ray direction $\rayw$ and flight distance $t$, the transmittance satisfies
\begin{equation}
    \trans(t) =
    \begin{cases}
        1 - G(\rayo + t\rayw) \> , & t < \tstar \\
        1 - G(\rayo + \tstar\rayw) \> , & t \geq \tstar \\
    \end{cases} \> ,
\end{equation}
such that:
\begin{equation}
    \trans(t) = 1 - G(\rayo + \min(t, \tstar)\rayw) \> .
\end{equation}
Differentiating yields
\begin{equation}
    -\ddt \log \trans(t) = \sigma(\rayo + t\rayw, \rayw) \> ,
\end{equation}
with 
\begin{equation}
    \sigma(\rayx, \rayw) = \max\left(0, -\rayw \cdot \nabla \log (1-G(\rayx))\right) \> .
\end{equation}
The same derivation extends to anisotropic Gaussians by replacing spherical level sets with ellipsoidal ones centered at $\mean$.

We now verify equivalence between alpha blending and the proposed density-based model by showing that both formulations produce the same rendered color. Let the Gaussian $G$ be equipped with a directional color field $\colorfield:\calS^2 \rightarrow [0,1]^3$. The rendered color $C_{\rayo,\rayw}$ along a ray from an empty point $\rayo$ in direction $\rayw$ under the density-based model is
\begin{equation}
    \begin{split}
        C_{\rayo,\rayw} 
        & = \int_0^{+\infty} \sigma(\rayo + s\rayw, \rayw) \trans(s) \colorfield(\rayw) \diff s \\
        & = \colorfield(\rayw) \int_0^{+\infty} -\ddt \trans(s) \diff s \\
        & = \colorfield(\rayw) \left[ \trans(0) - \trans(+\infty) \right] \\
        & = \colorfield(\rayw) \left[G(\rayo + \tstar \rayw) - G(\rayo)\right] \\
        & = \colorfield(\rayw) G(\rayo + \tstar \rayw) \> , \\
    \end{split} 
\end{equation}
where $G(\rayo)\simeq 0$ since $\rayo$ lies in empty space. The right-hand side matches the alpha-blending expression for a single Gaussian~\cite{yu2024gaussian,zhang2024rade}, which proves equivalence.

\end{proof}
}
\begin{figure}[t]
\centering

\begin{subfigure}[b]{0.512\linewidth}
    \centering
    \includegraphics[width=\linewidth]{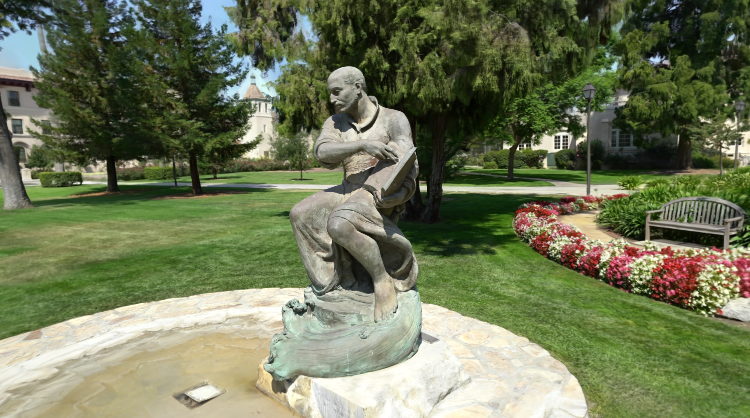}
    \caption{Gaussian Splatting rendering~\cite{zhang2024rade}}
\end{subfigure}
\begin{subfigure}[b]{0.468\linewidth} 
    \centering
    \begin{subfigure}[b]{\linewidth}
        \includegraphics[width=0.48\linewidth]{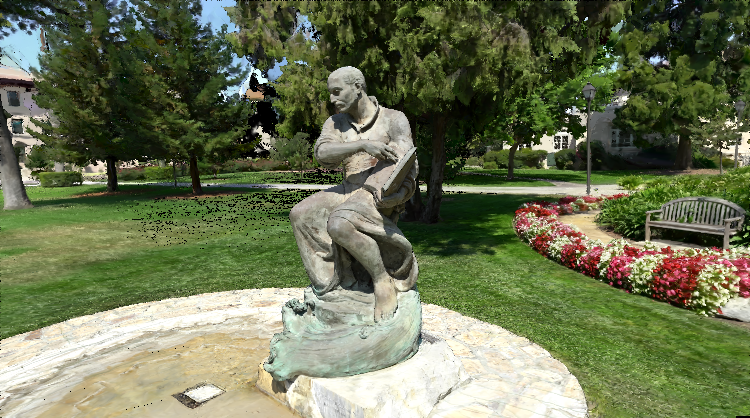}
        \includegraphics[width=0.48\linewidth]{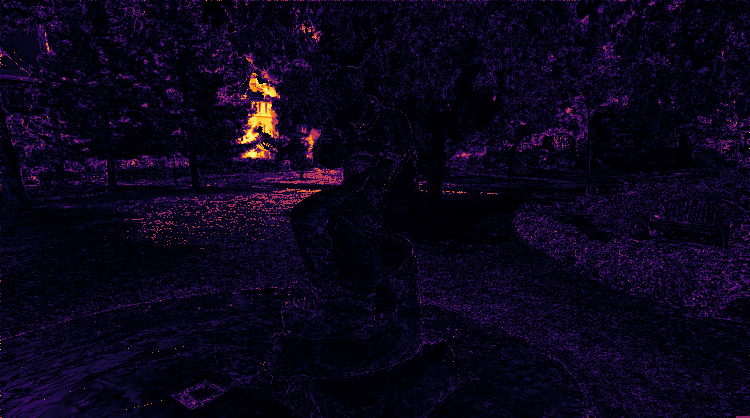}
        \caption{Ray-marching with attenuation}
    \end{subfigure}
    \begin{subfigure}[b]{\linewidth}
        \includegraphics[width=0.48\linewidth]{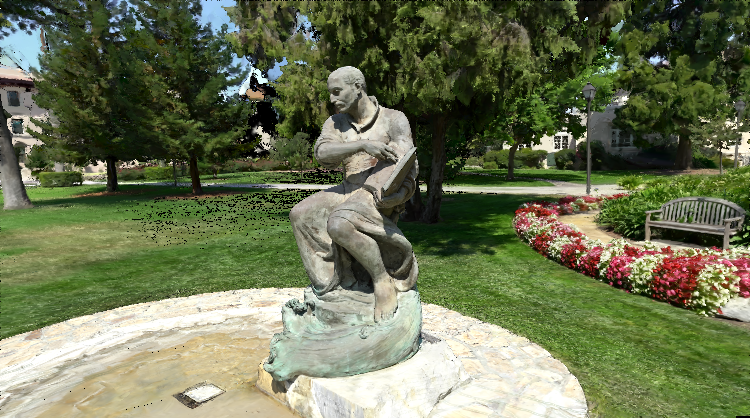}
        \includegraphics[width=0.48\linewidth]{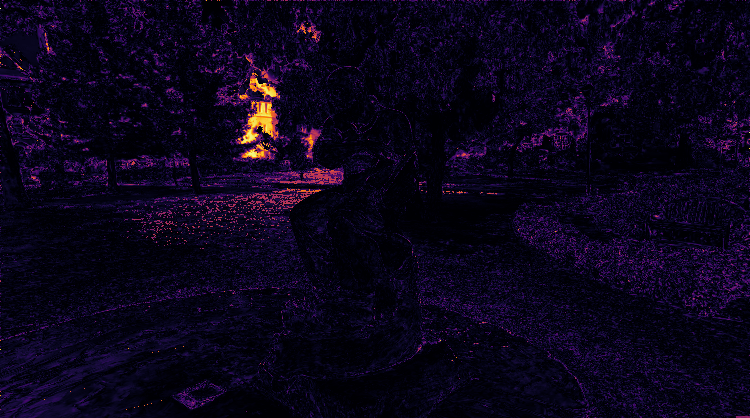}
        \caption{Ray-marching with oriented attenuation}
    \end{subfigure}
\end{subfigure}

\caption{\textbf{Equivalence between Gaussian Splatting rendering and our attenuation-based formulations.} (a) Standard Gaussian Splatting rasterization. (b) Ray-marching our derived attenuation field from Equation~\ref{eq:attenuation_gaussian_splatting}, with corresponding error map. (c) Ray-marching the oriented attenuation after optimizing Gaussians with our oriented normal regularization~\ref{eq:normal_field_regularization}. Both attenuation-based models produce visually similar results, confirming the formal equivalence established in \cref{sec:derivation_attenuation} and the approximation proposed in \cref{sec:oriented_gaussians}. The discrepancies occur precisely when the non-overlapping conditions aren't met, such as the poorly supervised background.}
\label{fig:attenuation_vs_splatting}
\end{figure}

We now extend the lemma to a set of $N$ Gaussians $\{G_i\}_{i=1}^N$ under two assumptions: (a) statistical independence, following RayGauss~\cite{blanc2025raygauss}, and (b) non-overlapping Gaussian primitives~\cite{celarek2025does}. Although the second assumption is approximate, alpha blending makes a similar approximation by sorting Gaussians front-to-back and compositing them as if overlap were negligible in 3D. In practice, this approximation is effective.

\hypbox[]{
\begin{proposition}[Transmittance of Gaussian Splatting]\label{theorem:transmittance_gaussian}
    Assuming statistical independence and non-overlapping Gaussian primitives, the image formation model of a Gaussian Splatting representation---relying on alpha-blending with opacity---is equivalent to a ray-marching volumetric rendering model with the following continuous transmittance function, for any ray origin $\rayo\in\IR^3$ in empty space, $\rayw\in\calS^2$, and $t\in [0,+\infty)$:
    \begin{equation}
        \trans(t) = \prod_{i=1}^N \left(1 - \Gstari(t)\right) \> ,
        \label{eq:transmittance_gaussian_splatting}
    \end{equation}
    where $\Gstari(t) = G_i(\rayo + \min(t, \tstari)\rayw)$.
    The corresponding attenuation coefficient is view-dependent and given by:
    \begin{equation}\label{eq:attenuation_gaussian_splatting}
        \sigma(\rayx, \rayw) = \sum_{i=1}^N \max\left(
            0,
            -\rayw \cdot \nabla \log (1-G_i(\rayx))
        \right) \> .
    \end{equation}
\end{proposition}
}

\paragraph{Sketch of proof.} Extending the previous proposition to multiple Gaussians is straightforward, as statistical independence between Gaussians implies that the total transmittance is the product of the individual transmittances. Finally, assuming Gaussians do not overlap in the 3D domain allows for proving the equivalence between the two image formation models.
Figure~\ref{fig:attenuation_vs_splatting} compares images rendered by ray-marching our attenuation field and by the standard Gaussian Splatting rasterizer, confirming visual consistency.

\proofbox[]{
\begin{proof}
    We first derive the transmittance expression from statistical independence. We then verify equivalence between alpha blending and the proposed density-based model.

    We assume the primitives to contribute independently to attenuation, such that the total coefficient is the sum of the individual coefficients: $\sigma(\rayx, \rayw) = \sum_i \sigma_i(\rayx, \rayw)$.
    Equivalently, under exponential transport, the total transmittance is the product of the individual transmittances:
    \begin{align}
        \trans(t) &= \exp\left( - \int_0^t \sum_i \sigma_i(\rayo + s\rayw, \rayw) \diff s \right) \\
        &= \prod_i \exp\left( - \int_0^t \sigma_i(\rayo + s\rayw, \rayw) \diff s \right) \\
        &= \prod_i \transi(t) \> ,
    \end{align}
    where $\transi(t)$ denotes the transmittance of the $i$-th Gaussian. Using \cref{lemma:transmittance_single_gaussian}, we obtain
    \begin{equation}
        \trans(t) = \prod_i \left(1 - \Gstari(t)\right) \> .
    \end{equation}

    We now verify equivalence between alpha blending and the proposed density-based model by comparing rendered color values.
    
    Consider an empty ray origin $\rayo$ and direction $\rayw$. Under the non-overlap assumption, Gaussians can easily be ordered along the ray, e.g., by projecting their centers onto the ray. We retain the $M$ Gaussians with positive projection and order them by increasing projected distance from $\rayo$.
    As Gaussians do not overlap, there exist $t_0 < \cdots < t_M$ such that
    \begin{align}
        & t_0  = 0 \\
        & t_M  = +\infty \\
        & \forall t\in[t_{i-1}, t_i] \> , \> \sigma(\rayo + t\rayw, \rayw) \simeq \sigma_i(\rayo + t\rayw, \rayw) \\
        & \forall i\in[1,M] \> , \> \Gstari(t_{i-1}) \simeq 0 \\
        & \forall i\in[1,M] \> , \> \Gstari(t_i) = \Gstari(\tstari) \> \> \> .
    \end{align}
    Similarly, we equip each Gaussian $G_i$ with a directional color field $\colorfield_i:\calS^2 \rightarrow [0,1]^3$ and assume the total color field $\colorfield:\IR^3\times\calS^2\rightarrow[0,1]^3$ to verify $\colorfield(\rayo + t\rayw, \rayw) \simeq \colorfield_i(\rayw)$ for all $t\in[t_{i-1}, t_i]$.
    As a result, for any $i\leq M$ and $t\in[t_{i-1}, t_i]$, we have:
    \begin{equation}
        \Gstarj(t) = \begin{cases}
            G_j(\rayo + \tstarj\rayw) & \text{if } j < i \\
            G_i(\rayo + t\rayw) & \text{if } j = i \\
            0 & \text{if } j > i 
        \end{cases} \> \> \> ,
    \end{equation}
    such that for all $t\in[t_{i-1}, t_i]$:
    \begin{equation}
        \transj(t) = \begin{cases}
            1 - G_j(\rayo + \tstarj\rayw) & \text{if } j < i \\
            1 - G_i(\rayo + t\rayw) & \text{if } j = i \\
            1 & \text{if } j > i 
        \end{cases} \> \> \> ,
    \end{equation}
    and, still for all $t\in[t_{i-1}, t_i]$:
    \begin{equation}
        \begin{split}
            \trans(t) 
            & = \prod_{j=1}^M \transj(t) \\ 
            & = \transi(t) \prod_{j=1}^{i-1} \transj(\tstarj) \> \> \> .
        \end{split}
    \end{equation}
    The rendered color value $C_{\rayo,\rayw}$ along the ray is then given by:
    \begin{equation}
        \begin{split}
            C_{\rayo,\rayw}
            & = \int_0^{+\infty} \sigma(\rayo + s\rayw, \rayw) \trans(s) \colorfield(\rayo + s\rayw, \rayw) \diff s \\ 
            & = \sum_{i=1}^M \int_{t_{i-1}}^{t_i} \sigma(\rayo + s\rayw, \rayw) \trans(s) \colorfield(\rayo + s\rayw, \rayw) \diff s \\
            & = \sum_{i=1}^M \int_{t_{i-1}}^{t_i} 
            \sigma_i(\rayo + s\rayw, \rayw) 
            \left(\transi(s) \prod_{j=1}^{i-1} \transj(\tstarj) \right)
            \colorfield_i(\rayw) \diff s \\
            & = \sum_{i=1}^M 
            \colorfield_i(\rayw)
            \left(\int_{t_{i-1}}^{t_i} \sigma_i(\rayo + s\rayw, \rayw) \transi(s) \diff s\right)
            \prod_{j=1}^{i-1} \transj(\tstarj) \\
            & = \sum_{i=1}^M 
            \colorfield_i(\rayw)
            \left[\transi(t_i) - \transi(t_{i-1})\right]
            \prod_{j=1}^{i-1} \transj(\tstarj) \\
            & = \sum_{i=1}^M 
            \colorfield_i(\rayw)
            \Gstari(\tstari)
            \prod_{j=1}^{i-1} \left(1-\Gstarj(\tstarj)\right) \> \> \> .
        \end{split}
    \end{equation}
    The right-hand side is exactly the alpha-blending color for multiple Gaussians~\cite{yu2024gaussian,zhang2024rade}, proving equivalence.
\end{proof}
}

\subsection{Oriented Gaussians}\label{sec:oriented_gaussians}

The attenuation factor from the previous section is not reciprocal: $\sigma(\rayx,\rayw)\neq\sigma(\rayx,-\rayw)$ in general. Each Gaussian $G_i$ contributes to attenuation only before the ray reaches its point of maximum contribution $\rayo+\tstari\rayw$, so reversing the ray changes where transmittance varies, breaking reciprocity. This asymmetry primarily manifests when a ray \emph{exits} an object (see \cref{fig:vacancy}).

Reciprocity is a necessary condition for applying the OaV results~(see main paper), which links attenuation to the direction-independent $\nabla \log v$. Without it, back-facing Gaussians introduce a systematic bias that prevents deriving a closed-form vacancy expression.

\begin{figure}
\centering
\begin{subfigure}{0.3\textwidth}
    \includegraphics[width=\textwidth]{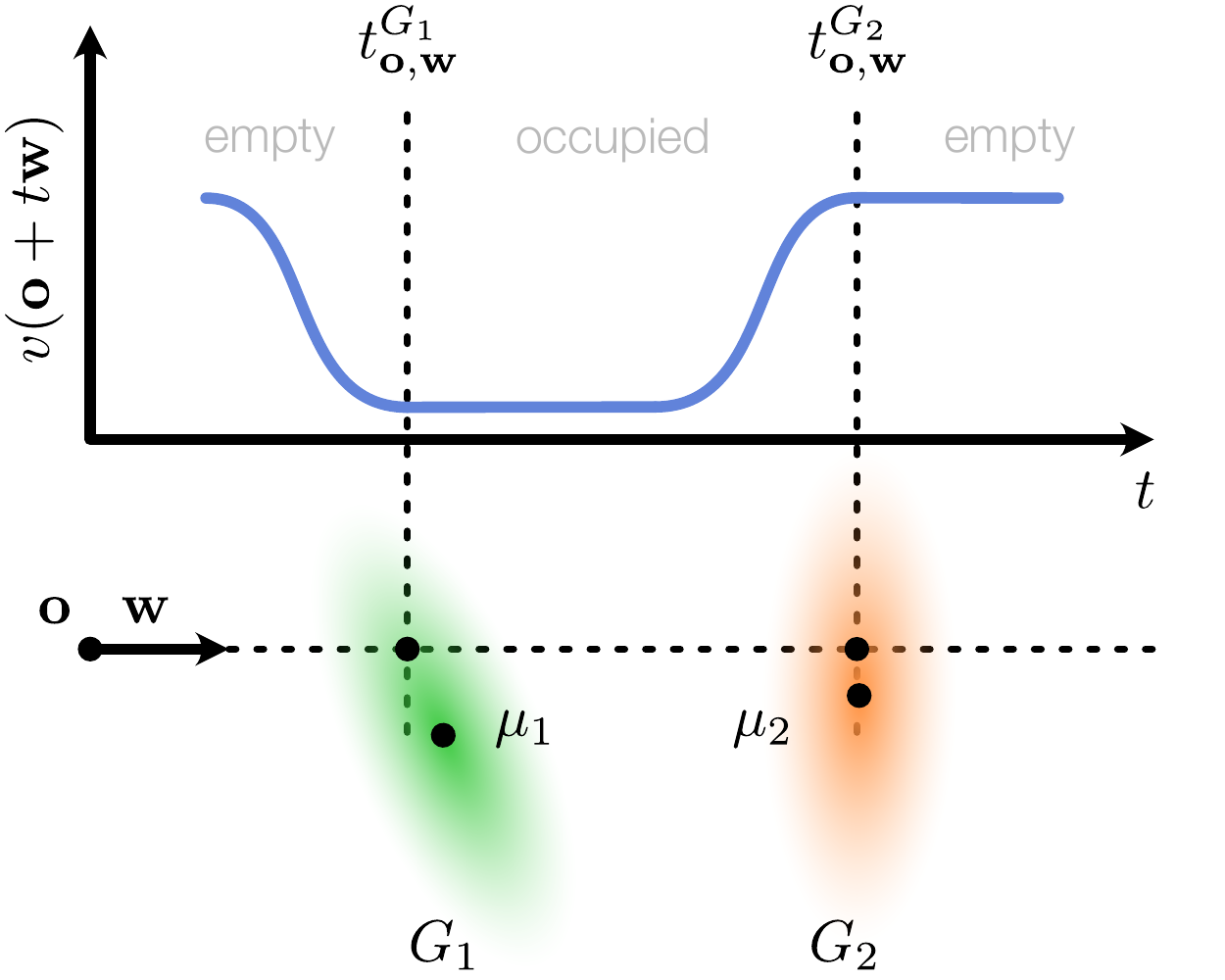}
    \caption{Initial attenuation.}
    \label{fig:vacancy_tstar}
\end{subfigure}
\hspace{1cm}
\begin{subfigure}{0.3\textwidth}
    \includegraphics[width=\textwidth]{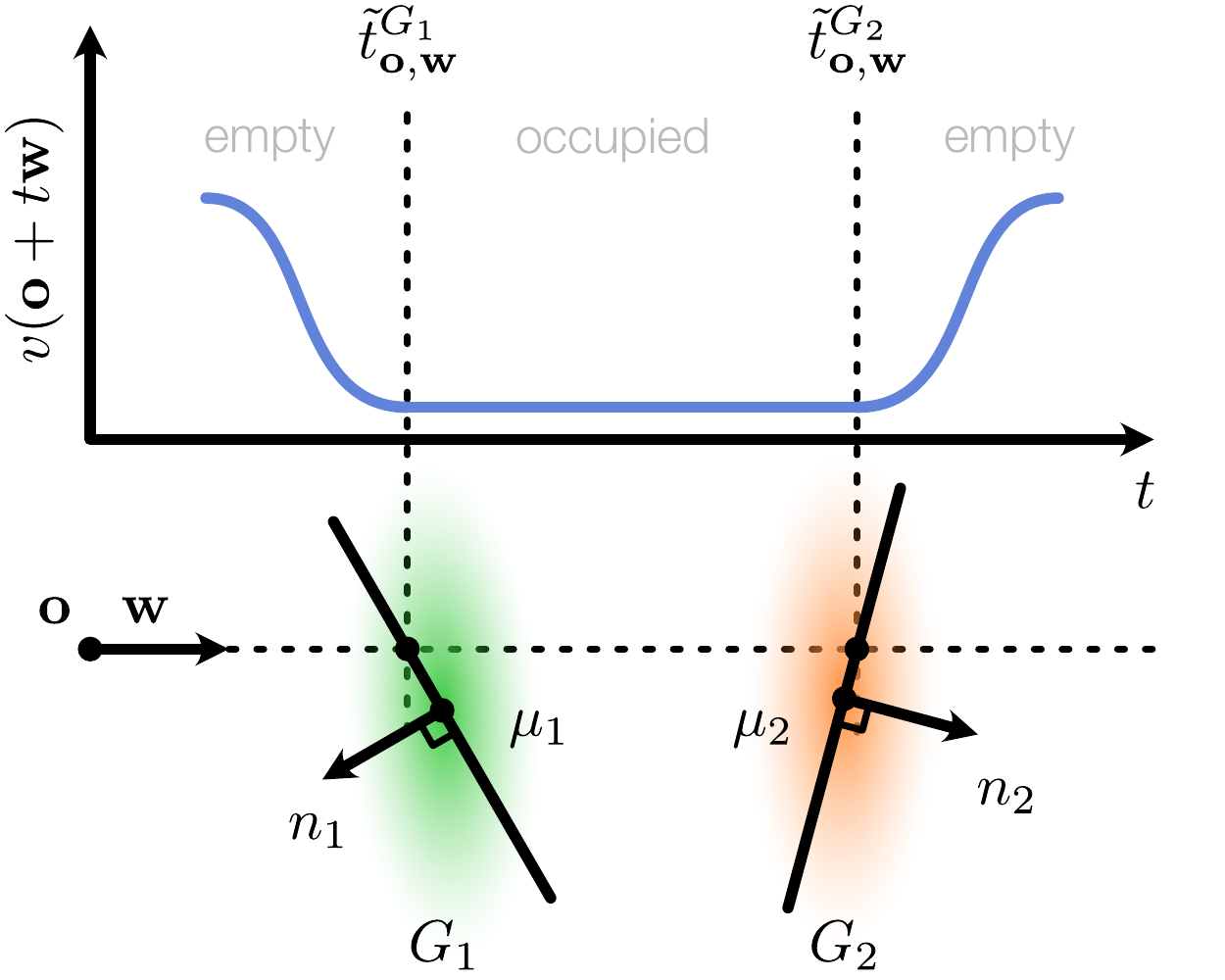}
    \caption{Oriented attenuation.}
    \label{fig:vacancy_tplane}
\end{subfigure}
    \caption{\textbf{Evolution of the vacancy along a ray.} In contrast to the initial attenuation we derived for Gaussian Splatting (a), our oriented Gaussian formulation (b) allows for a reciprocal formulation of the vacancy. This figure shows how our formulation corrects the bias present in (a). Indeed, in (a) vacancy starts changing just before each Gaussian. This poses a problem for $G_2$, since vacancy starts evolving from occupied to empty before crossing $t_{o,w}^{G_2}$. You can see in that in our formulation (b) this change happens after crossing this critical point. This renders the evaluation of vacancy reciprocal}
    \label{fig:vacancy}
\end{figure}

To restore reciprocity, we \emph{orient} each Gaussian with a normal vector and modify the attenuation accordingly, assuming the scene geometry is \emph{wrapped} by oriented Gaussians each covering a local surface patch.

\hypbox[]{
\begin{definition}[Oriented Gaussian]\label{definition:oriented_gaussians_appendix}
    Consider a parameter $n_i\in\calS^2$ associated to each $G_i$ --- we call it the normal vector of the Gaussian. We define the oriented attenuation coefficient of the Gaussian $G_i$ as
    \begin{equation}
        \orientedsigma_i(\rayx, \rayw) = \charx |\rayw \cdot \nabla\log (1-G_i(\rayx))| \> .
        \label{eq:oriented_attenuation_appendix}
    \end{equation}
\end{definition}
}

By introducing $n_i$, we replace the direction-dependent split at $\tstari$ with a fixed half-space: $\orientedsigma_i$ is non-zero where $n_i^T(\rayx - \meani) \geq 0$ and zero otherwise. Since this half-space depends only on the Gaussian's position and normal---not the ray direction---the resulting attenuation is reciprocal by construction, enabling the application of OaV's results, and thus the derivation of a closed-form equation for vacancy (see~\cref{vector_field}).

This introduces a consistency requirement: the region where transmittance is constant must lie inside the object (see \cref{fig:vacancy}), which holds when Gaussians properly \textit{wrap} the surface. Concretely, both formulations agree for non-occluded, front-facing Gaussians when the point of maximum contribution $\rayo+t^{G_1}_{\rayo,\rayw}\rayw$ aligns with the plane intersection $\rayo+\tilde{t}^{G_1}_{\rayo,\rayw}\rayw$.

\paragraph{Normal alignment loss.} We promote Gaussians to ``wrap'' the scene through the following normal alignment loss.
\begin{equation}
    \mathcal{L}_{\text{N}} = \sum_{p}
    1 - N(p) \cdot \nabla D(p) \> ,
    \label{eq:normal_field_regularization_appendix}
\end{equation}
where $N(p)$ is the rendered oriented normal at pixel $p$ and $\nabla D(p)$ is the image-space gradient of the rendered depth. Following~\cite{miller2024objectsasvolumes}, rendering depth and normals $n_i$ with the Splatting rasterizer~\cite{zhang2024rade,yu2024gaussian} yields respectively the expected depth of the point of maximum contribution $\rayo+t^{G_1}_{\rayo,\rayw}\rayw$, and the expected orientation of the Gaussians along the ray. As a result, $\mathcal{L}_{\text{N}}$ enforces $\rayo+t^{G_1}_{\rayo,\rayw}\rayw$ and the plane intersections $\rayo+\tilde{t}^{G_1}_{\rayo,\rayw}\rayw$ to be similar on average, encouraging the two formulations to agree.

\subsection{Deriving Surface Normals}

The previous section established that a Gaussian Splatting representation is approximately equivalent to an attenuation-based model. When Gaussians are aligned on a surface and properly wrap the scene geometry, orienting each Gaussian with a normal vector yields a reciprocal attenuation coefficient $\sigma$. Gaussian Splatting then satisfies the reciprocal exponential transport assumption of OaV~\cite{miller2024objectsasvolumes}, and their result provides an explicit expression of $\nabla\log v$.

This section leverages this gradient field to derive an explicit normal field on the stochastic surface. This normal field serves two purposes: supervising oriented Gaussians during training to enforce proper surface wrapping, and extracting a mesh from the trained representation.

\hypbox[]{
\begin{definition}[Gaussian vector field]
    Consider a Gaussian Splatting representation of $N$ oriented Gaussians.
    The Gaussian vector field $V$ is defined as the gradient of the log-vacancy. Following OaV:
    \begin{equation}
        \begin{split}
            V(x) = \nabla\log v(x) 
            & = \sum_{i=1}^N \charx \nabla\log (1-G_i(x)) \> .
        \end{split}
        \label{vector_field}
    \end{equation}
\end{definition}
}

\proofbox[]{
\begin{proof}
    Combining the result of OaV and~\cref{eq:oriented_attenuation_appendix}, for any $\rayx\in\IR^3$ and $\rayw\in\calS^2$, yields:
    \begin{equation}
        \begin{split}
            |\rayw \cdot \nabla\log v(\rayx)| 
            & = \sum_i \charx |\rayw \cdot \nabla\log (1-G_i(\rayx))| \> .
        \end{split}
    \end{equation}
    Under the wrapping assumption, each gradient term $\nabla\log (1-G_i(\rayx))$ points from occupied space ($v=0$) toward empty space ($v=1$) when $\charx > 0$. As a result, we assume that, locally, each term $\charx \rayw \cdot \nabla\log (1-G_i(\rayx))$ either has the same sign as $\rayw \cdot \nabla\log v(\rayx)$ or is equal to zero.
    Hence, we can disambiguate the absolute values within the sum and write:
    \begin{equation}
        \begin{split}
            \rayw \cdot \nabla\log v(\rayx) 
            & = \sum_i \rayw \cdot \charx \nabla\log (1-G_i(\rayx)) \> .
        \end{split}
        \label{eq:vector_field_sign}
    \end{equation}
    Since Eq.~\ref{eq:vector_field_sign} holds for arbitrary direction $\rayw\in\calS^2$, Eq.~\ref{vector_field} follows.
\end{proof}
}

\hypbox[]{
\begin{proposition}[Gaussian normal field]
    Consider a Gaussian Splatting representation of $N$ oriented Gaussians.
    The Gaussian normal field $\calN:\IR^3 \to \calS^2$ is defined as the normalized Gaussian vector field $V:\IR^3 \to \IR^3$:
    \begin{equation}
        \calN(x) = 
        \begin{cases}
            \frac{V(x)}{\|V(x)\|} & \text{if } \|V(x)\| > 0 \\
            0 & \text{otherwise}
        \end{cases} \> \> \> .
        \label{eq:normal_field_appendix}
    \end{equation}
    In a neighborhood of the surface, $\|V\|>0$ and $\calN$ coincides with the true normal field of the expected stochastic surface of $M$.
\end{proposition}
}

\proofbox[]{
\begin{proof}
    The normal field of a scene $E\subset\IR^3$ with surface $\partial E \subset E$ is defined as the gradient of a signed distance function (SDF) $f:\IR^3 \to \IR$ satisfying
    \begin{align}
        f(x) & \leq 0 \quad \text{if } x \in E, \\
        f(x) & > 0 \quad \text{if } x \notin E, \\
        f(x) & = 0 \quad \text{if } x \in \partial E,\\
        \|\nabla f(x)\| & = 1 \quad \text{for all } x \in \IR^3 \quad \text{(Eikonal equation).}
        \label{eq:signed_distance_field}
    \end{align}
    Following \textit{Objects as Volumes}~\cite{miller2024objectsasvolumes}, we assume the SDF $f$ representing the expected stochastic surface $\partial M$ is locally a scalar function of the vacancy $v$. Specifically, we assume there exists an increasing sigmoidal function $\Psi:\IR \to (0,1)$ and a scalar $s>0$ such that, for all $x$ in a neighborhood of the surface:
    \begin{equation}
        f(x) = \frac{1}{s} \Psi^{-1}(v(x)) \> .
    \end{equation}
    Since $\Psi^{-1}$ and $\log$ are strictly increasing, $\nabla f$ and $\nabla \log v$ are collinear with a positive scalar factor in a neighborhood of the surface.
    The Eikonal constraint $\|\nabla f\|=1$ then implies $\nabla f = \nabla\log v \,/\, \|\nabla\log v\|$, so $\calN$ coincides with the true surface normal field near the surface.
\end{proof}
}

In summary, under our oriented Gaussian assumption, the normal field of the stochastic surface represented by Gaussian Splatting admits an explicit closed-form expression through the Gaussian vector field $V$ (Eq.~\ref{vector_field}). This expression enables both surface-aware regularization during training and mesh extraction at inference time. We can see in~\cref{fig:vector_field_alignment}, that the normalized vector field (i.e, the normal field) queried at the median depth (i.e, 0.5-isosurface our transmittance) has been regularized to match the rendered normals only through rendering losses. We show that this is the case in theory too in~\cref{lemma:normal_field_supervision_appendix}.

\begin{figure}[ht!]
\centering
\setlength{\tabcolsep}{1pt}
\renewcommand{\arraystretch}{0.5}

\begin{tabular}{ccc}
    \small Ground Truth & \small Rendered Normals & \small Normal Field $\mathcal{N}$ \\[2pt]
    \includegraphics[width=0.323\linewidth]{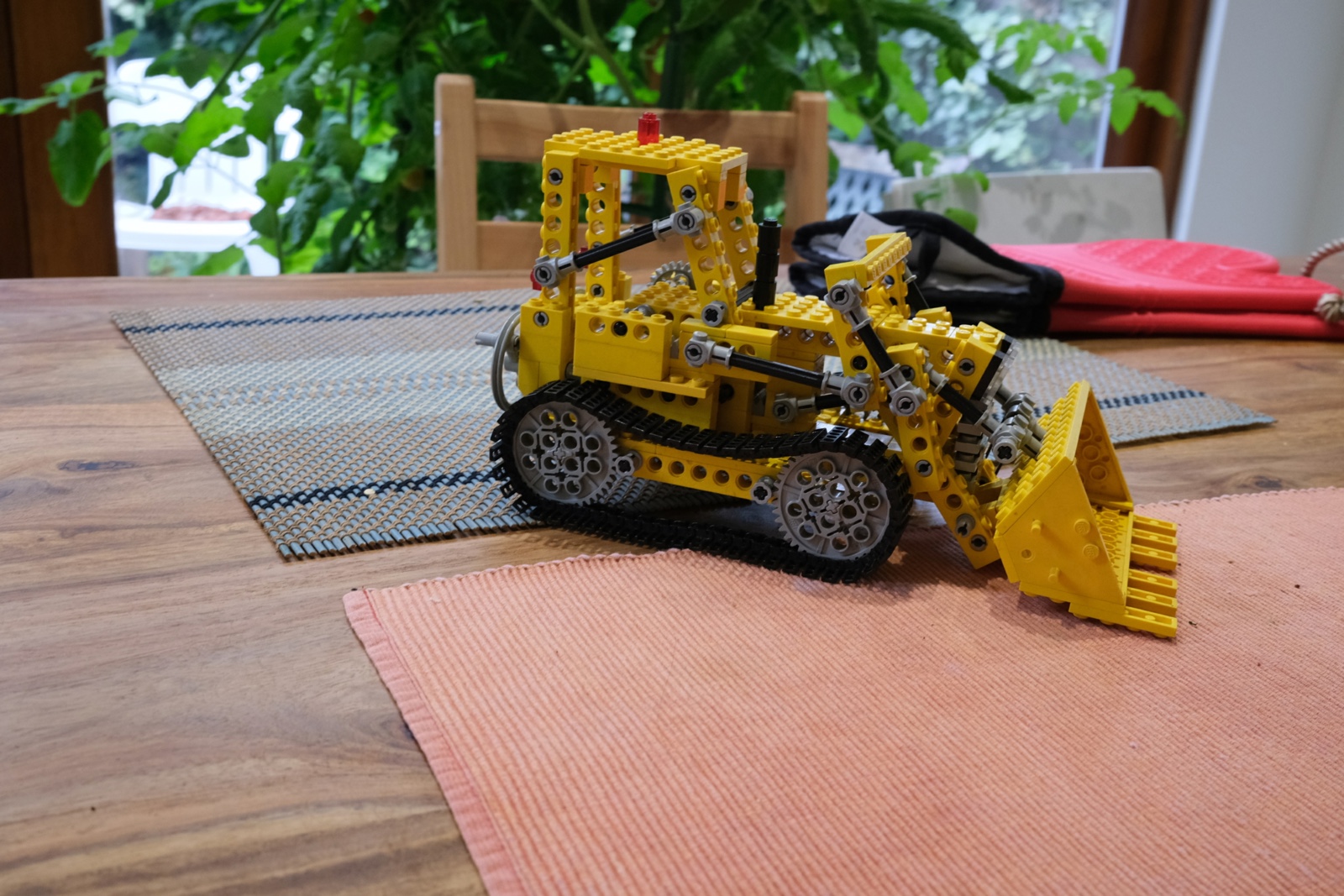} &
    \includegraphics[width=0.323\linewidth]{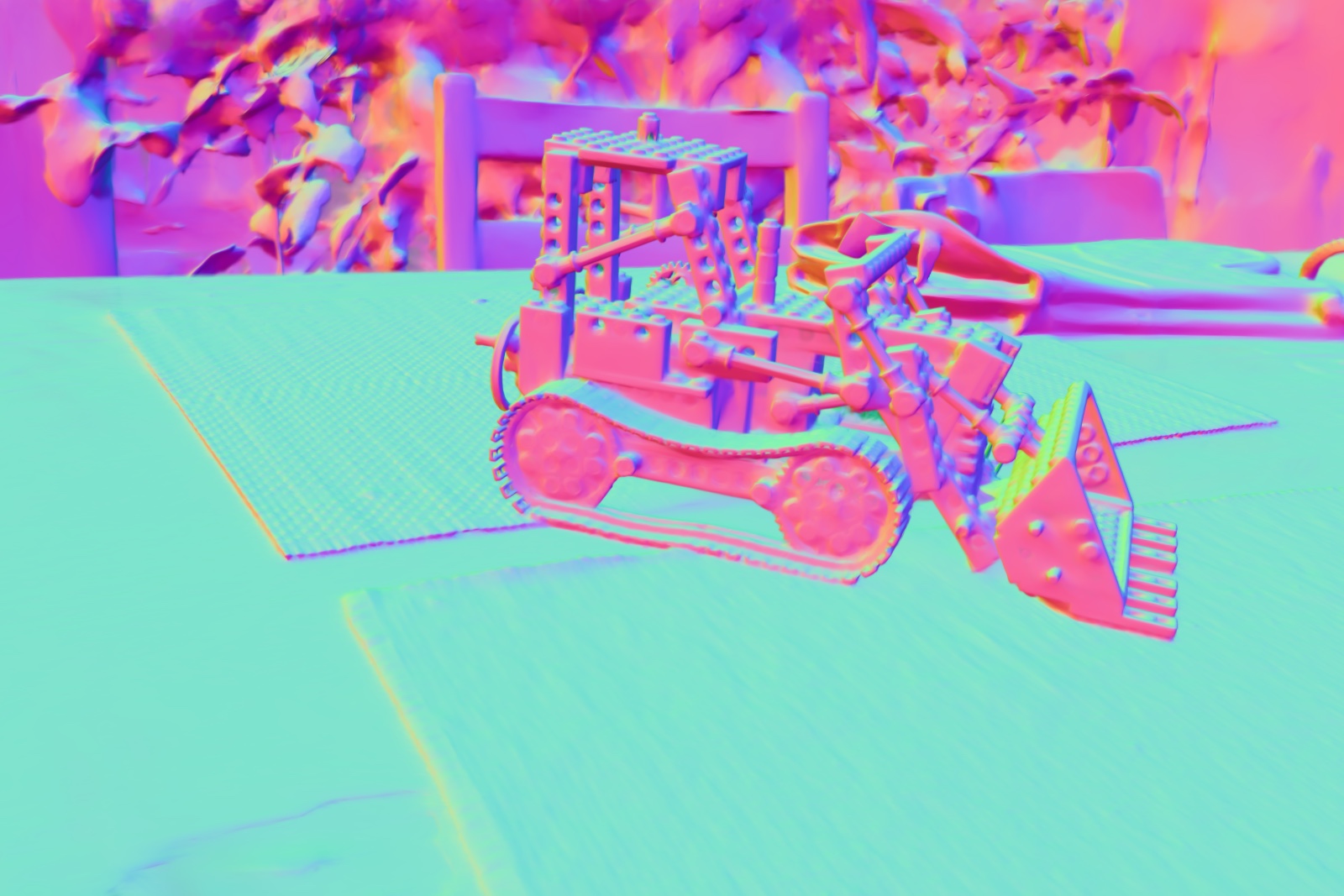} &
    \includegraphics[width=0.323\linewidth]{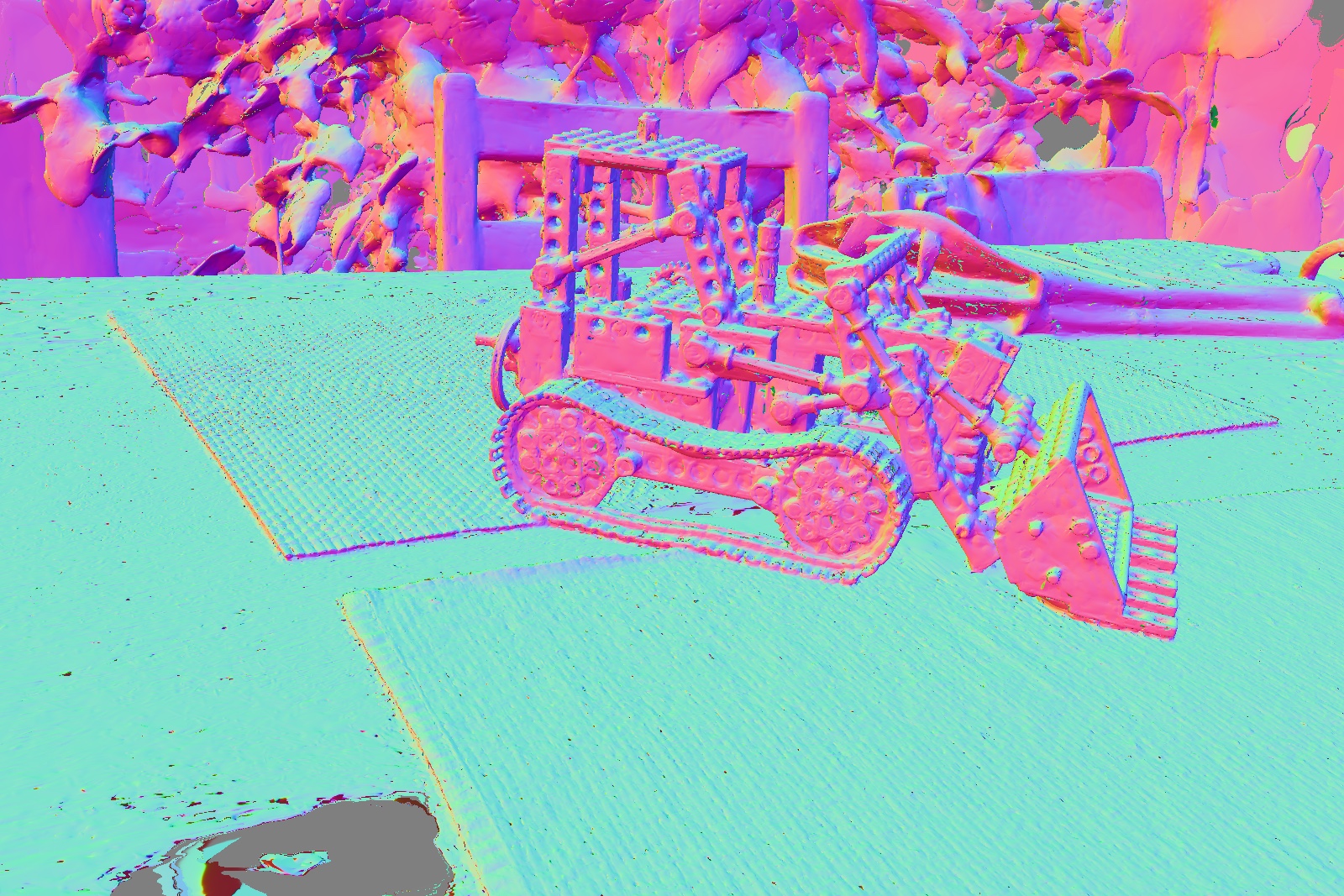} \\[1pt]
    \includegraphics[width=0.323\linewidth]{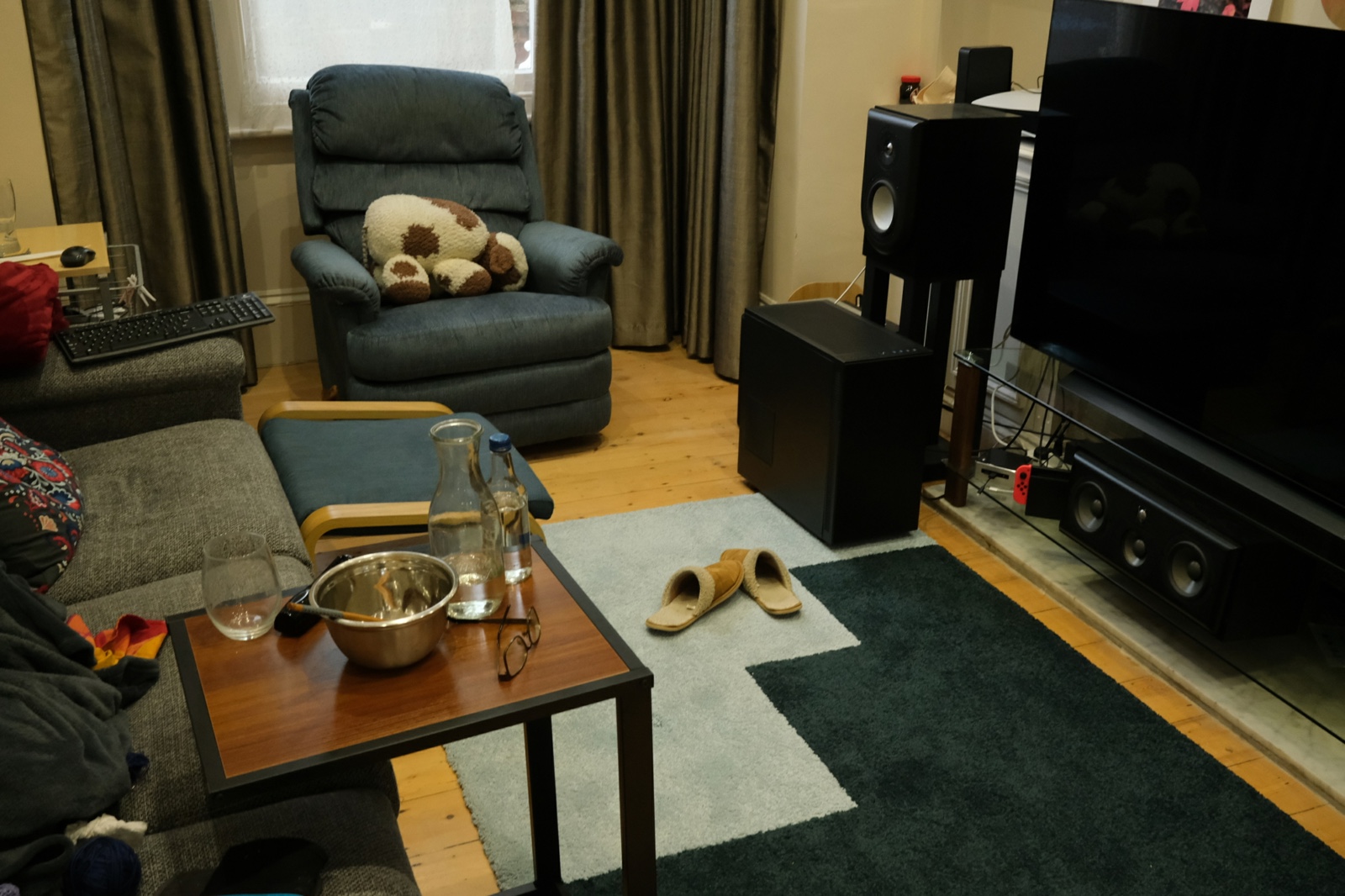} &
    \includegraphics[width=0.323\linewidth]{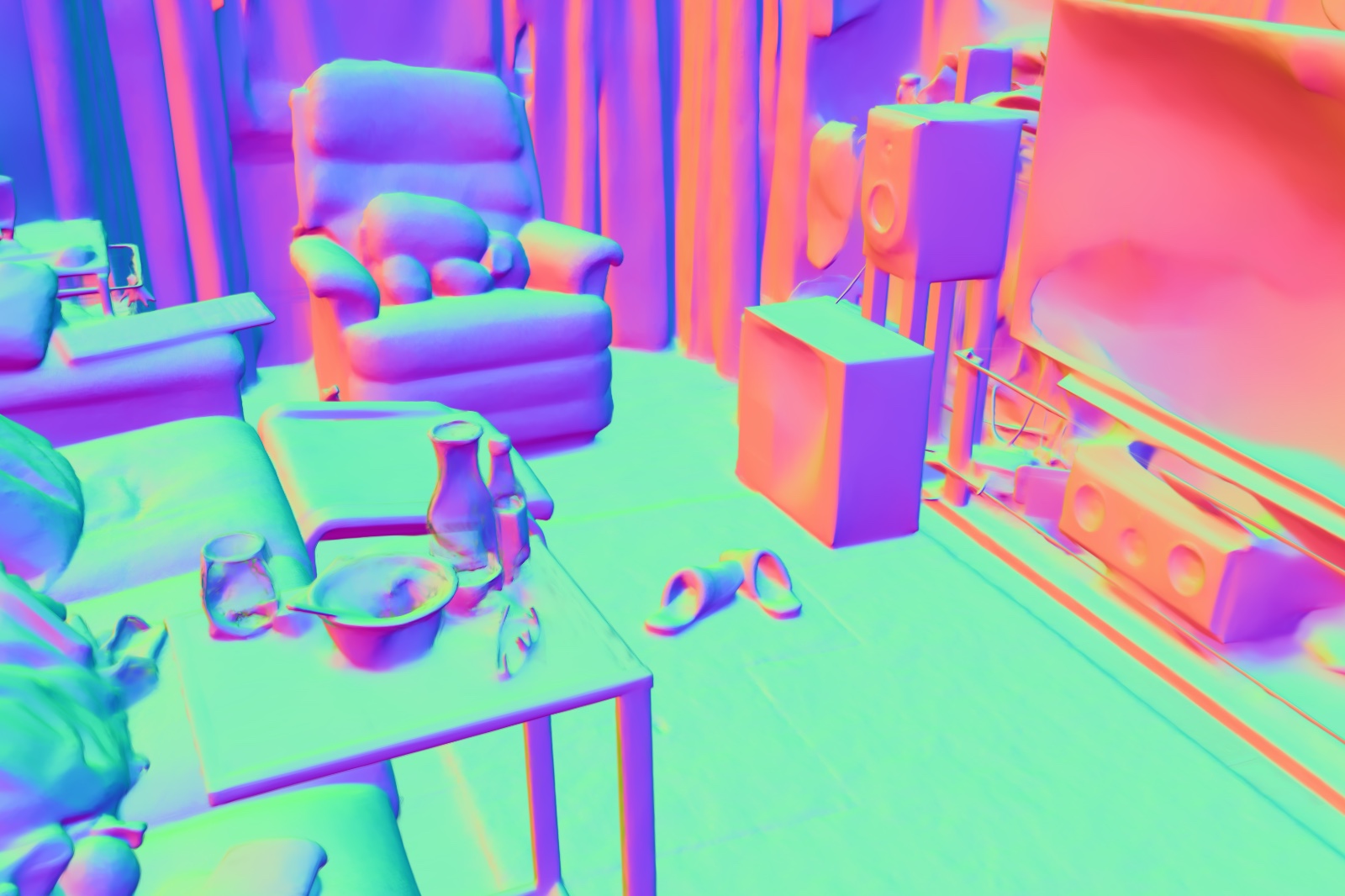} &
    \includegraphics[width=0.323\linewidth]{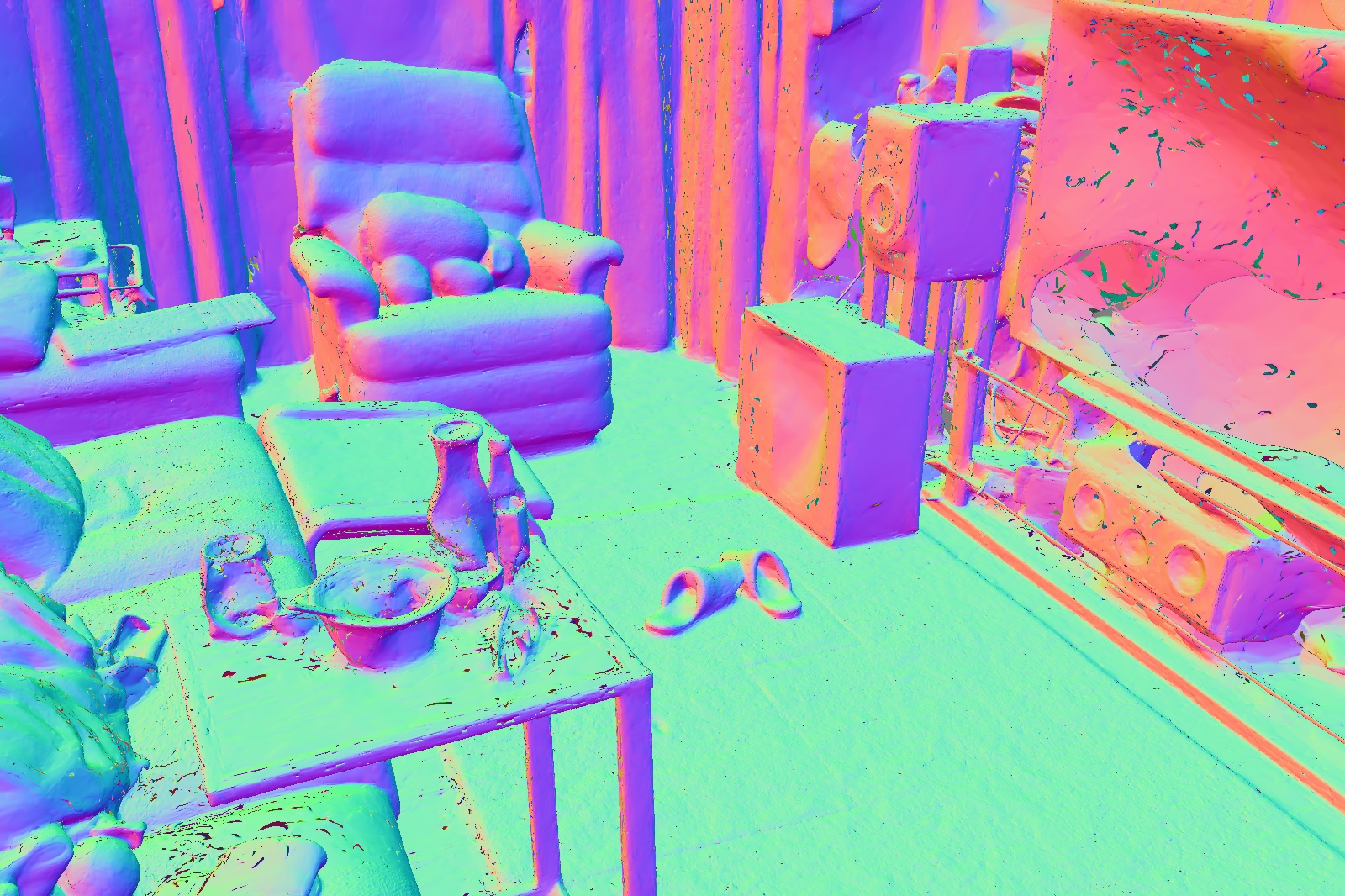} \\[1pt]
    \includegraphics[width=0.323\linewidth]{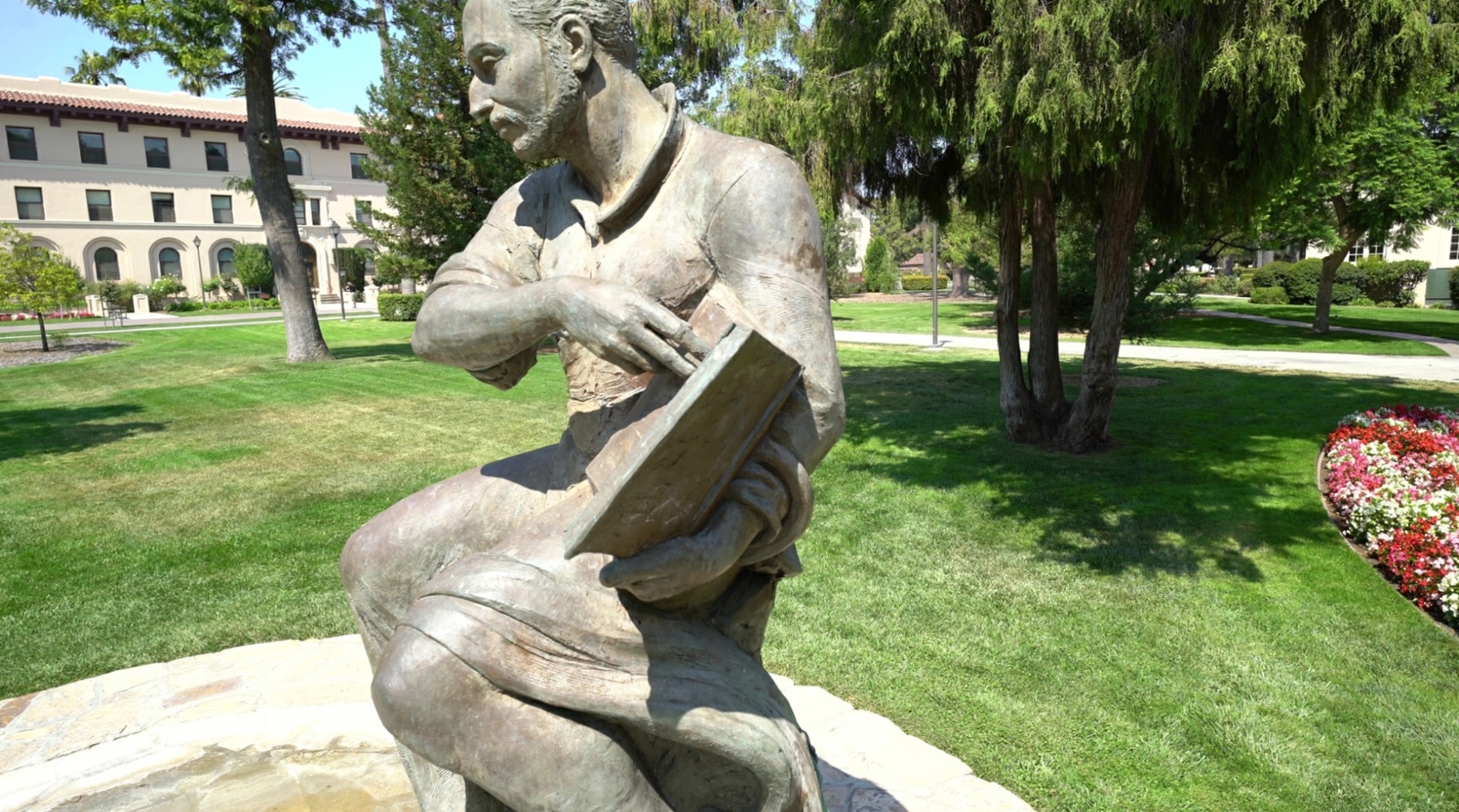} &
    \includegraphics[width=0.323\linewidth]{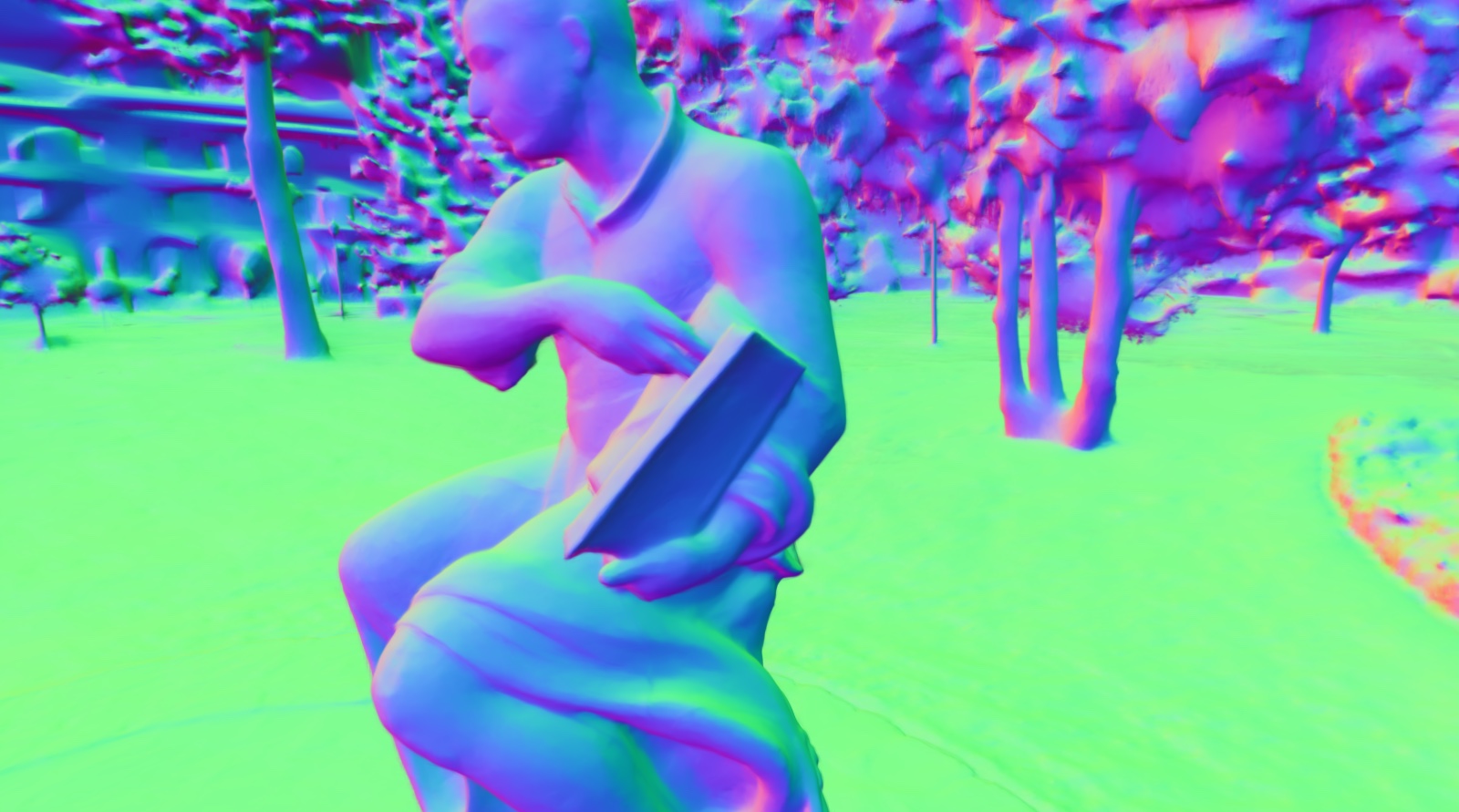} &
    \includegraphics[width=0.323\linewidth]{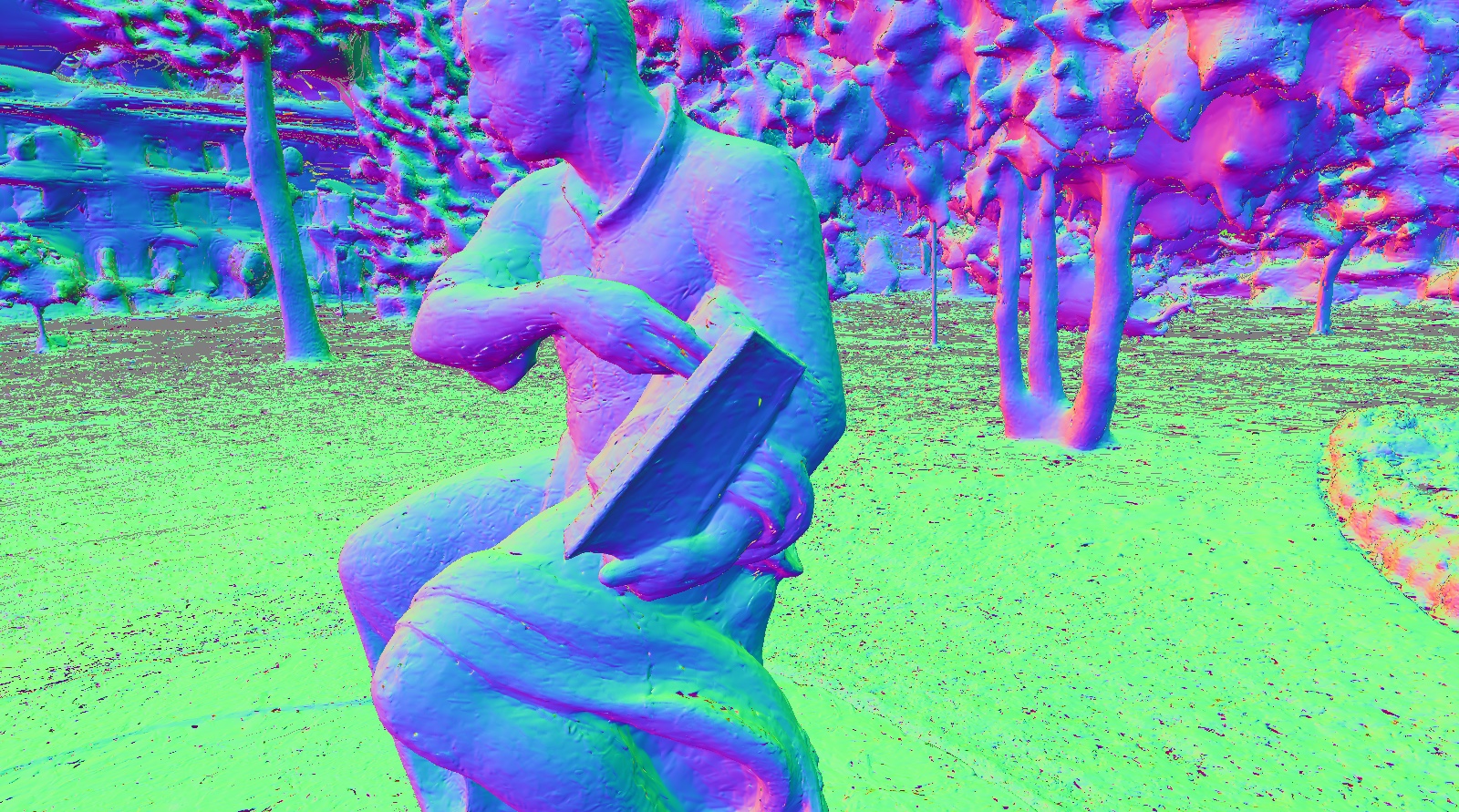} \\[1pt]
    \includegraphics[width=0.323\linewidth]{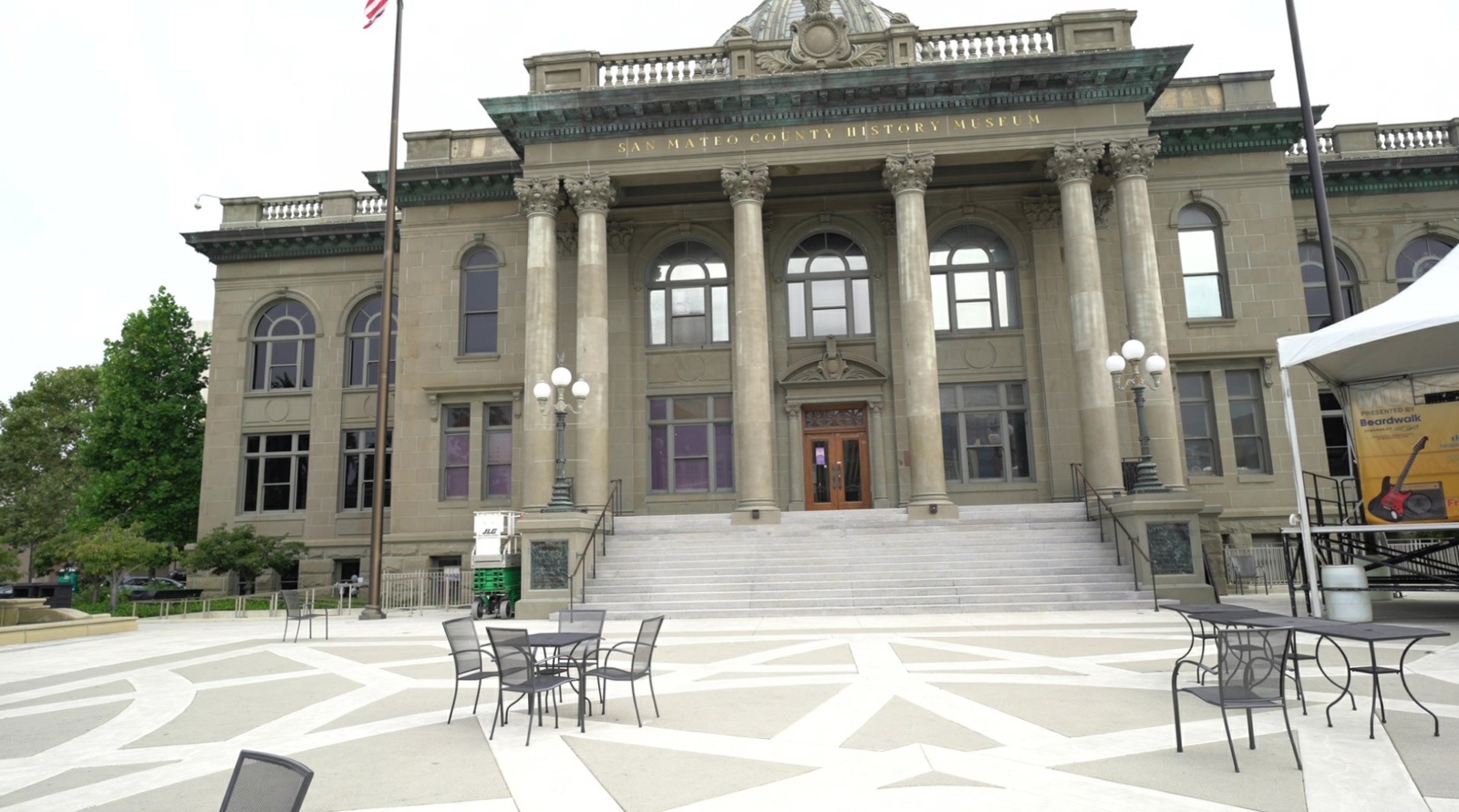} &
    \includegraphics[width=0.323\linewidth]{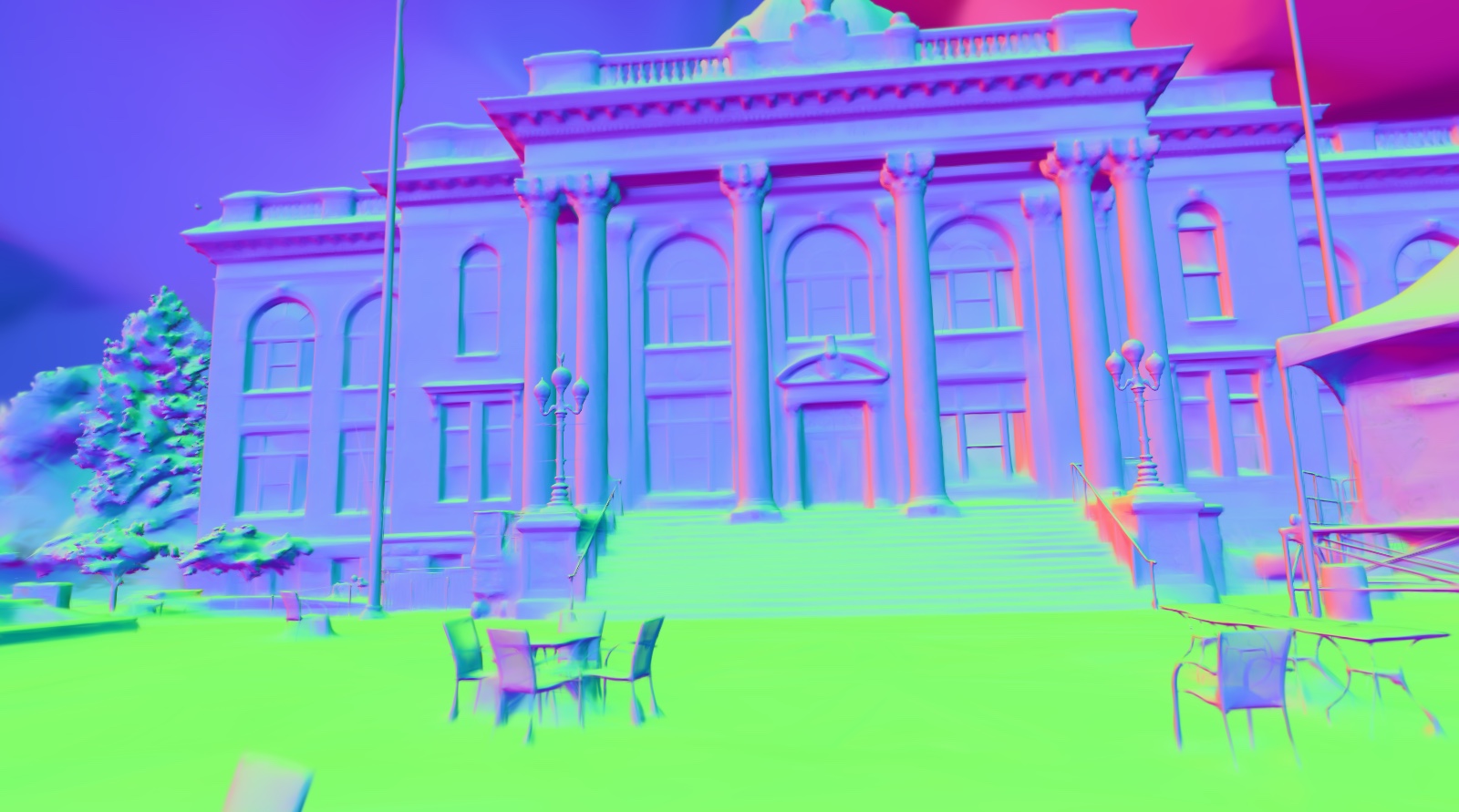} &
    \includegraphics[width=0.323\linewidth]{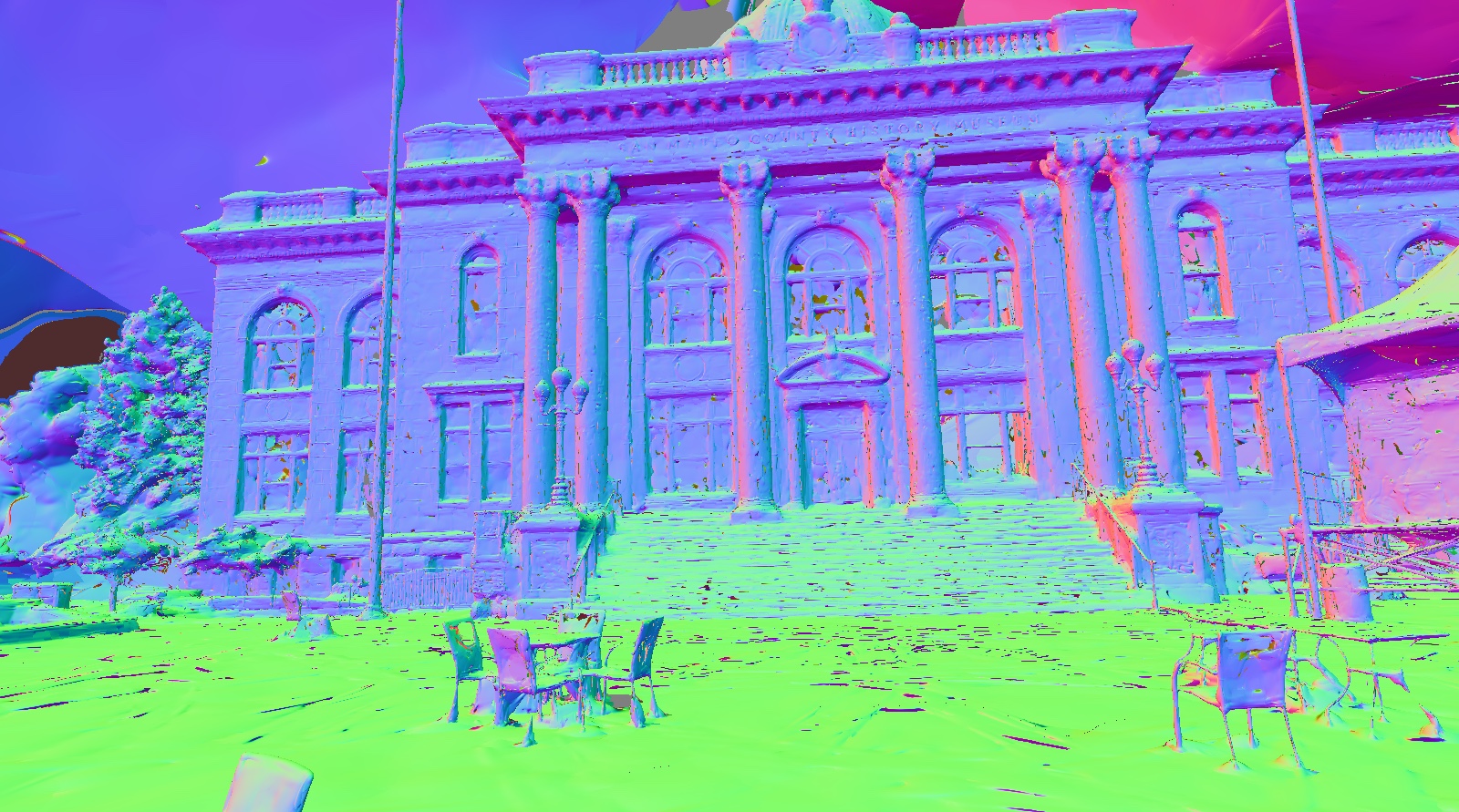} \\
\end{tabular}

\caption{
\textbf{Gaussian vector field aligned with surface normals.}
Visualization of the Gaussian normal field $\mathcal{N}$, \ie the normalized vector field $V = \nabla\log v$, across four scenes (Kitchen, Room, Ignatius, Courthouse).
Each row shows, from left to right: the ground truth RGB image, the normals rendered by alpha-blending the oriented Gaussian normals $n_i$, and the normal field $\mathcal{N}(x)$ queried at points $x$ lying on the 0.5-isosurface of the transmittance (median depth).
In practice, $\mathcal{N}(x)$ is evaluated by querying the $K$-nearest Gaussians to $x$ ($K=32$).
The close agreement between rendered normals and $\mathcal{N}$ confirms that the Gaussian vector field is well-aligned with the surface. We note that artifacts in the normal field $
\mathcal{N}$ correspond to under-reconstructed parts of the scene (e.g specular spot on Kitchen table, TV in room, background in Ignatius scene, and sky in Courthouse scene).
}\label{fig:vector_field_alignment}
\end{figure}

\hypbox[]{
\begin{lemma}[Supervision of Normal Field]
    We supervise the normal field $\mathcal{N}$ by enforcing consistency between:
    \begin{enumerate}
        \item The rendered normal maps created by rasterizing the 3D Gaussian Splatting representation with normals $n_i$;
        \item The image-derivative of the depth.
    \end{enumerate}
    Doing so promotes the value of our normal field at the 0.5-isosurface of the transmittance to match the true normal of the surface.
\end{lemma}\label{lemma:normal_field_supervision_appendix}
}

\proofbox[]{
\begin{proof}
    
First, let us start by showing why supervising the rendered normal maps---produced by splatting Gaussians with their learnable normals $n_i$---also supervises the normal field $\mathcal{N}$.

For a given ray $\ray$, the expectancy of the normal field value at the observed surface is equal to
\begin{equation}
    \IE_M(\mathcal{N}(\rayo+\tff\rayw)) = \int_0^{+\infty} \pdft(s) N(\rayo+s\rayw)\diff s
\end{equation}
where $\tff$ is the free-flight distance and $\pdft$ is the pdf of $\tff$. 
Following OaV, $\pdft(t) = \sigma(\rayo+t\rayw, \rayw) \trans(t)$ and the expectancy is equal to
\begin{equation}
    \IE_M(\mathcal{N}(\rayo+\tff\rayw)) = \int_0^{+\infty} 
    \sigma(\rayo+s\rayw, \rayw) \trans(s)
    N(\rayo+s\rayw)\diff s \> ,
\end{equation}
which is precisely equal to the rendering equation of radiance fields.

In the context of Gaussian splatting, this integral is discretized and approximated with alpha-blending. As such, for rendering the normal field, we need to attribute to every Gaussian a single value $\mathcal{N}_i$ per Gaussian to approximate the above integral. Let's start with the vector field $V$. We need to select a single value $V_i$ for rendering.
A natural choice is to use the average value of the contribution of the Gaussian to $V$, \ie, the integral of the $i$-th term of the sum multiplied by the normalized density of the Gaussian:
\begin{equation}
    V_i \propto \int_{\IR^3} G_i(x)\cdot \charx \nabla\log (1-G_i(x)) \diff x
\end{equation}
This integral has a symmetry around the vector $n_i$ (this can be easily seen with a change of variable to reduce and center the Gaussian; the resulting integral is over a hemisphere). As such, the result $V_i \propto n_i$, and similarly the value of the normalized vector field $\mathcal{N}_i \sim n_i$ as both have norm $1$.

As a consequence, \textit{one can simply render the plane normals $n_i$ with Gaussian splatting to estimate the average value of the Normal field $\mathcal{N}$ on the observed surface}, and supervise $\calN$ through a rendering loss.

The question then arises of how to supervise the rendered normal maps. We propose to compare this value to the derivative of the rendered depth. Indeed, according to~\cite{miller2024objectsasvolumes}, rendering the depth directly estimates the expectancy of the free-flight distance $\tff$ for each pixel, \ie, the average position of the visible surface point.
By taking the derivative of the depth (\ie, the normal of the expected visible surface), we can then compute an estimate of the surface normals that does not depend on our \textit{oriented Gaussians} assumption and normals $n_i$. We can finally use this estimate to supervise our plane normals.

More precisely, we choose as the depth to use, the median value of $\tff$ (as opposed to the expectancy). Indeed, the median provides an estimate of $\tff$ which is more robust to statistical outliers than the expectancy. The median of $\tff$ can be computed as
\begin{equation}
    \begin{split}
        \median(\tff) & = \inf \left\{
        t\in[0,+\infty) : 
        \int_0^t \pdft(s) \diff s \geq 0.5
        \right\} \\
        & = \inf \left\{
        t\in[0,+\infty) : 
        \int_0^t -\ddt\trans(s) \diff s \geq 0.5
        \right\} \\
        & = \inf \left\{
        t\in[0,+\infty) : 
        \trans(0) - \trans(t) \geq 0.5 
        \right\}\\
        & = \inf \left\{
        t\in[0,+\infty) : 
        1 - \trans(t) \geq 0.5 
        \right\}\\
        & = \inf \left\{
        t\in[0,+\infty) : 
        \trans(t) \leq 0.5 
        \right\} \> ,
    \end{split}
\end{equation}
such that in the context of point-based rendering, the median of $\tff$ can be estimated as the depth $\tstar$ of the first Gaussian along the ray for which the transmittance reaches the 0.5 threshold. This definition is consistent with the median depth as implemented in previous works~\cite{yu2024gaussian,zhang2024rade,radl2025sof}
\end{proof}
}

\section{Mesh Extraction Details}\label{sec:mesh_extraction_details}

We now proceed to leverage the closed-form equation for the vacancy for mesh extraction. In this section, we provide additional details on the procedures introduced in the main paper.

\subsection{Estimating Vacancy} 

As stated, we leverage the ability to compute the vacancy $v$---the probability of a point to be outside the opaque volume---for arbitrary points of space. In this section, we provide further motivation and proof for the formula introduced in the main document.

The explicit value of the log-vacancy $\rayx$ can be computed by integrating the vector field along any camera ray $\ray$ such that $\rayx = o+t\rayw$ for some $t \in [0,+\infty)$:
\begin{equation}
    \begin{split}
        \int_0^{t} V(\rayo+s\rayw)\cdot \rayw \diff s 
        & = \int_0^{t} \nabla\log v(\rayo+s\rayw)\cdot \rayw \diff s \\
        & = \log v(\rayo+t\rayw) - \log v(\rayo) \\
        & = \log v(\rayx)
    \end{split} \> \> ,
\end{equation}
since $\log v(\rayo) = \log 1 = 0$ as the ray origin $\rayo$ is located in empty space.

However, directly integrating the vector field might produce noisy results as it assumes all Gaussians follow our \textit{oriented Gaussian} assumption. In practice, some Gaussians not participating in the rendering might be trapped inside the volume and perturb the computation of this integral. Moreover, even though our optimization and densification greatly help to enforce our assumption and wrap the geometry with oriented Gaussians, it is possible for \textit{holes} to remain.
Consequently, we propose another computation which proves to be more robust to such scenarios.

\hypbox[]{
\begin{proposition}
    Let $v:\IR^3 \to [0,1]$ be the vacancy function, and let $\mathcal{T}_c$ denote the set of all training camera rays. For any point $\rayx \in \IR^3$, the vacancy admits the following lower bound:
    \begin{equation}
        v(\rayx) \geq \max_{(\rayo,\rayw)\in \mathcal{T}_c}\left\{
             \prod_{i=1}^N \left(1 - \Gstari(t)\right) : \rayx = \rayo + t\rayw, t>0
        \right\} \> ,
        \label{eq:lower_bound_vacancy_appendix}
    \end{equation} 
    where $\Gstari(t) = G_i(\rayo + \min(t, \tstari)\rayw)$; and $\tstari := \arg\max_{t\geq 0} G_i(\rayo + t\rayw)$.
\end{proposition}
}

\proofbox[]{
\begin{proof}

Following OaV, we establish the following relation between the log-vacancy and the log-transmittance for any $\rayo,\rayw, t$:
\begin{equation}
    |w\cdot\nabla\log v(\rayo+t\rayw)| = -\ddt\log\trans(t) \> ,
\end{equation}
which implies, as $\ddt\log\trans(t)$ is always non-positive:
\begin{equation}
    w\cdot\nabla\log v(\rayo+t\rayw) \geq \ddt\log\trans(t) \> .
\end{equation}
By integrating the formula from $0$ to $t$ along the ray $\ray$,
\begin{equation}
    \log v(\rayo+t\rayw) - \log v(\rayo) \geq \log\trans(t) - \log\trans(0) \> ,
\end{equation}
with $\log v(\rayo) = \log\trans(0) = \log1 = 0$, resulting in:
\begin{equation}
\begin{split}
    \log v(\rayo+t\rayw) & \geq \log\trans(t) \\
    v(\rayo+t\rayw) & \geq \trans(t) \> \> .
\end{split}    
\end{equation}
Consequently, for any point $\rayx$, we have:
\begin{equation}
    v(\rayx) \geq \max\left\{
        \trans(t) : (\rayo,\rayw)\in \IR^3\times\calS^2, t>0, \text{ s.t. } \rayx = \rayo + t\rayw
    \right\} \> .
    \label{lower_bound_vacancy}
\end{equation}
Equation~\ref{lower_bound_vacancy} shows that we can compute a lower bound of the vacancy at $\rayx$ as the maximum of transmittance values along camera rays passing through $\rayx$. This particularly holds for the set $\mathcal{T}_c$ of all training camera rays.
\end{proof}
}

In practice, for any 3D point $\rayx$, we iterate over training cameras and compute the transmittance along the corresponding ray for each camera. \textbf{We finally select the maximum transmittance over the camera rays as an estimate of the vacancy.}

The computation of the lower bound is more robust to floaters and Gaussians hidden inside the geometry, as transmittance is non-decreasing and mainly depends on the visible Gaussians participating in the rendering. Additionally, we use the transmittance from the non-oriented scenario in~\cref{lower_bound_vacancy}. This choice of transmittance is more robust to outlier Gaussians not ``properly'' wrapping some parts of the scene. Moreover, our losses enforce un-oriented and oriented formulations of the transmittance formula to be consistent for visible Gaussians.

\subsection{Newton Projection Step}

\begin{figure}[ht!]
\centering
\includegraphics[width=\linewidth]{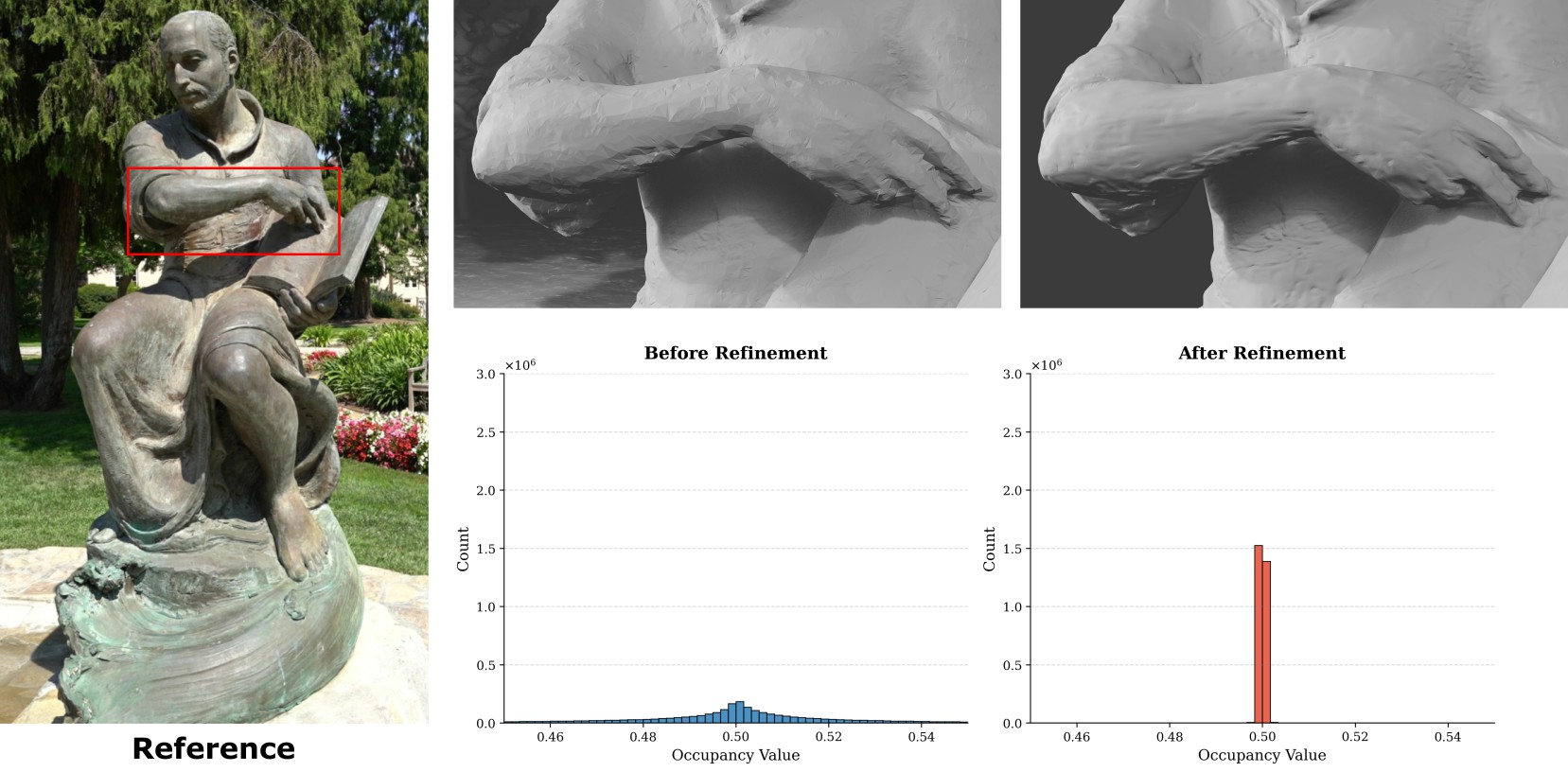}

\caption{
\textbf{Effect of the Newton projection step on vacancy values at sampled points.}
\textit{Left:} Reference image with a mark on zoomed area.
\textit{Top:} mesh renders of the Baseline and our proposed method on the Truck scene of T\&T.
\textit{Bottom:} distribution of vacancy values $v(x)$ evaluated at the points to refine for PAM~(zoomed in on the $[0.45, 0.55]$ range around the target isosurface $v=0.5$). This plot was obtained performing PAM on the Ignatius scene of T\&T, sampling 3 Million points and performing 10 refinement steps.
\textit{Bottom Left (blue):} Sampled points on mesh without Newton refinement;
\textit{Bottom Right (red):} after applying our Newton projection step (Proposition~\ref{prop:newton}), points collapse onto the isosurface $v=0.5$, validating the efficiency of our refinement. Qualitatively we can see that disentangling the vertex distribution from the Gaussians distributions allows to retrieve extremely fine-grained details.
}\label{fig:newton_refinement}
\end{figure}

Here, we motivate the choice of the Newton step we use for refining the point cloud during \textit{Primal Adaptive Meshing}. \cref{fig:newton_refinement} illustrates the effectiveness of this procedure. The Newton step aims to collapse points to the 0.5-isosurface of the vacancy.

\hypbox[]{
\begin{proposition}\label{prop:newton}
    Let $v\colon\IR^3\to[0,1]$ be the vacancy function and note that the normal field $\mathcal{N}(x) = \nabla v(x)/\|\nabla v(x)\|^2$. The first-order Newton step projecting $x_i$ onto the isosurface $v^{-1}(0.5)$ and restricting updates to the gradient direction is:
    \begin{equation}
        x_{i+1} = x_i + \frac{0.5 - v(x_i)}{2}\,\mathcal{N}(x_i).
    \end{equation}
\end{proposition}
}

\proofbox[]{
\begin{proof}
    Introduce the auxiliary function $f(x) = {(0.5 - v(x))}^2$, whose roots coincide with the isosurface. Linearising $f$ around $x_i$ and setting to zero gives the Newton condition:
    \begin{equation}
        f(x_i) + \nabla f(x_i)^\top(x_{i+1} - x_i) = 0.
    \end{equation}
    We restrict to gradient-direction updates by writing $x_{i+1} = x_i + t\,\nabla v(x_i)$ for a scalar step $t$. Substituting $\nabla f(x) = -2\bigl(0.5-v(x)\bigr)\nabla v(x)$ and solving for $t$:
    \begin{equation}
        t
        = \frac{-f(x_i)}{\nabla f(x_i)^\top \nabla v(x_i)}
        = \frac{-(0.5-v(x_i))^2}{-2(0.5-v(x_i))\|\nabla v(x_i)\|^2}
        = \frac{0.5 - v(x_i)}{2\|\nabla v(x_i)\|^2}.
    \end{equation}
    Therefore:
    \begin{equation}
        x_{i+1}
        = x_i + t\,\nabla v(x_i)
        = x_i + \frac{0.5-v(x_i)}{2\|\nabla v(x_i)\|^2}\,\nabla v(x_i)
        = x_i + \frac{0.5 - v(x_i)}{2}\,\mathcal{N}(x_i). \qedhere
    \end{equation}
\end{proof}
}

\section{Experiments Details}\label{sec:experiments_details}

\paragraph{Loss Details.} At training time, we use our normal alignment loss $\calL_{\text{N}}$ in conjunction with the following set of losses: The standard photometric~\cite{kerbl3Dgaussians} loss $\calL_{\text{RGB}}$, depth-normal consistency $\calL_{\text{DN}}$~\cite{yu2024gaussian,zhang2024rade} and the multi-view loss $\mathcal{L}_{\text{MV}} = \lambda_{\text{pc}} \mathcal{L}_{\text{pc}} +  \lambda_{\text{gc}} \mathcal{L}_{\text{gc}}$, where the terms control the strength of the photometric and geometric consistency as defined in~\cite{Zhang2026GeometryGrounded,chen2024pgsr}. These have the following weights $\lambda_{\text{DN}}=0.05,\lambda_{\text{N}}=0.05, \lambda_{\text{pc}}=0.6, \lambda_{\text{gc}}=0.02$.
This results in the total loss 
\[
\calL=\calL_{\text{RGB}} +  \lambda_{\text{DN}}\calL_{\text{DN}} + \lambda_{\text{N}}\mathcal{L}_{\text{N}} + \mathcal{L}_{\text{MV}}
\]

\paragraph{Learnable Normals Parameterization.} We observed that parameterizing the learnable normals $n_i$ as simple, normalized vectors of 3 coordinates leads to unstable optimization. Indeed, flipping the direction of the normal under this naive parameterization requires the vector to perform a rotation of $180$ degrees, leading to convergence toward inaccurate local minima and poor efficiency.
Instead, we parameterize the normal $n_i$ of the Gaussian $i$ with 4 parameters, including a sign parameter $s_i\in\IR$ and a direction $d_i \in \IR^3$. The normal is computed as:
\[
n_i = \tanh(s_i) \cdot \frac{d_i}{\|d_i\|} \> 
\]
This decomposition into direction and sign enforces the norm of the normal to be lower than 1, while allowing for ‘‘flipping’’ the direction simply by changing the sign parameter $s_i$. Overall, it stabilizes gradients and prevents optimization from getting stuck in local minima.

\paragraph{CUDA implementation of forward and backward pass.} As stated in the main paper, we provide a modified version of the efficient rasterizer proposed by the work Geometry Grounded Gaussian Splatting~\cite{Zhang2026GeometryGrounded}. This modification consists of using our definition of transmittance~(\cref{theorem:transmittance_gaussian}), instead of the original one that leads to the issues presented in~\cref{fig:comparison_transmittance}. In practice, the differences are simple and involve removing the square root of the original formulation. We will release the code upon acceptance.

\paragraph{T\&T Eval Bias.} In the main paper we mention that there is a bias towards dense meshes in the current evaluation used by the rest of the literature on the T\&T dataset. We give as simple support this snippet of code available in the RaDe-GS~\cite{zhang2024rade} repository (see~\cref{fig:code_flaw}). The code creates the point cloud to be compared to the ground truth by simply concatenating the available vertices of the mesh and the center of each face---as opposed to uniformly sampling the mesh surface.

\subsection{DTU Qualitative Examples}

In this section we briefly present some renders of the meshes extracted from DTU with our method. They can be found in \cref{fig:dtu_qualitative}.

\begin{figure}[ht!]
\centering
\setlength{\tabcolsep}{1pt}

\begin{tabular}{ccc}
    \includegraphics[width=0.323\linewidth]{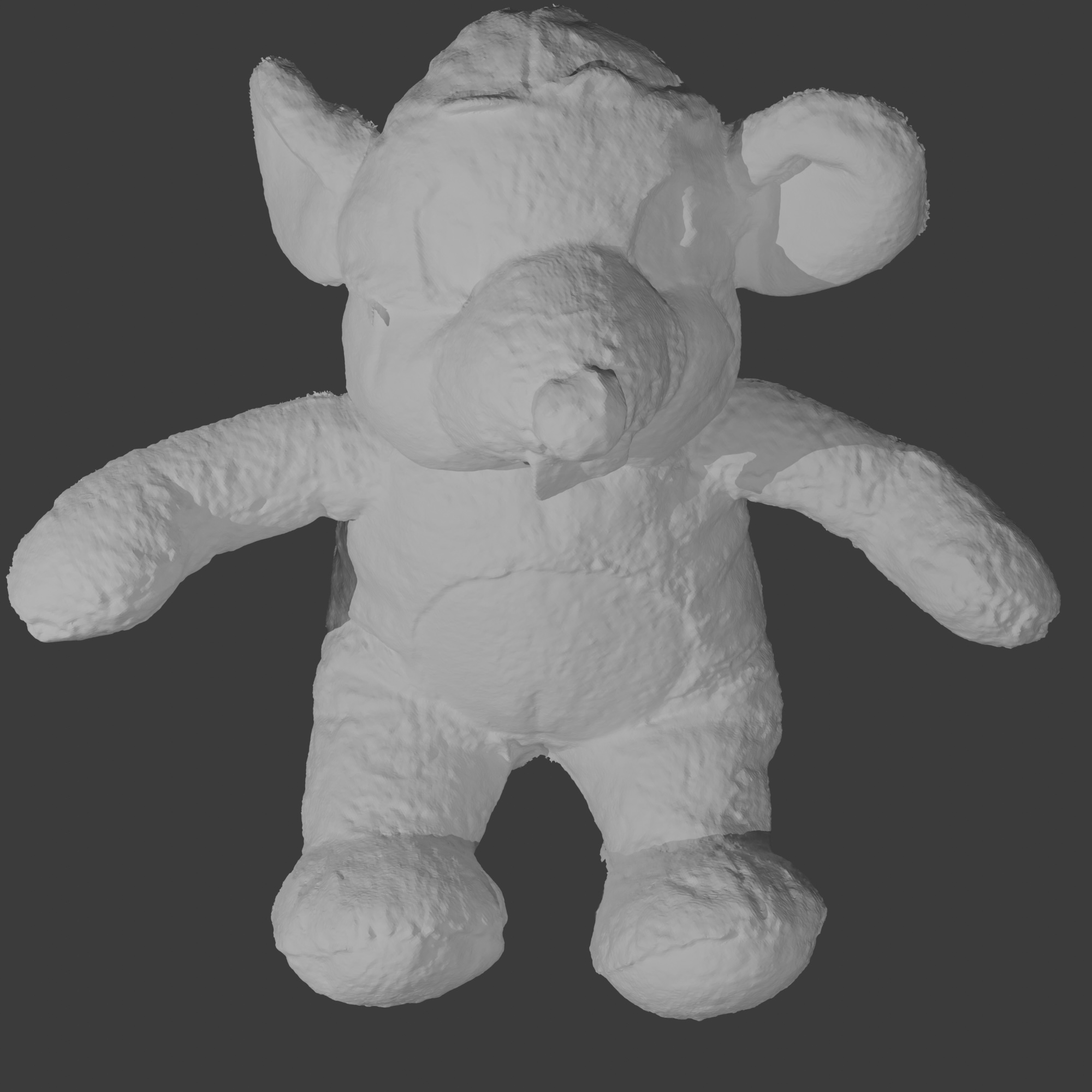} &
    \includegraphics[width=0.323\linewidth]{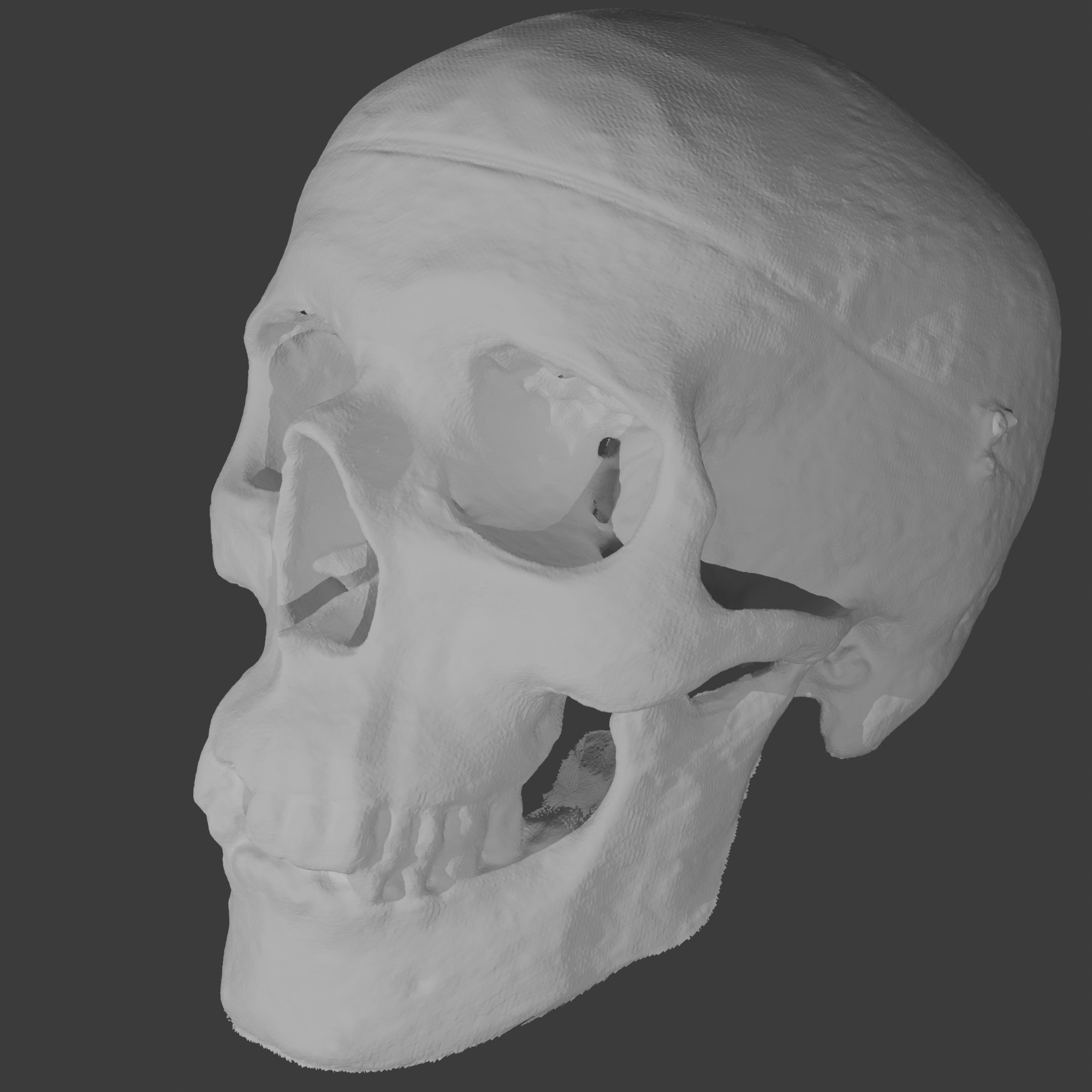} &
    \includegraphics[width=0.323\linewidth]{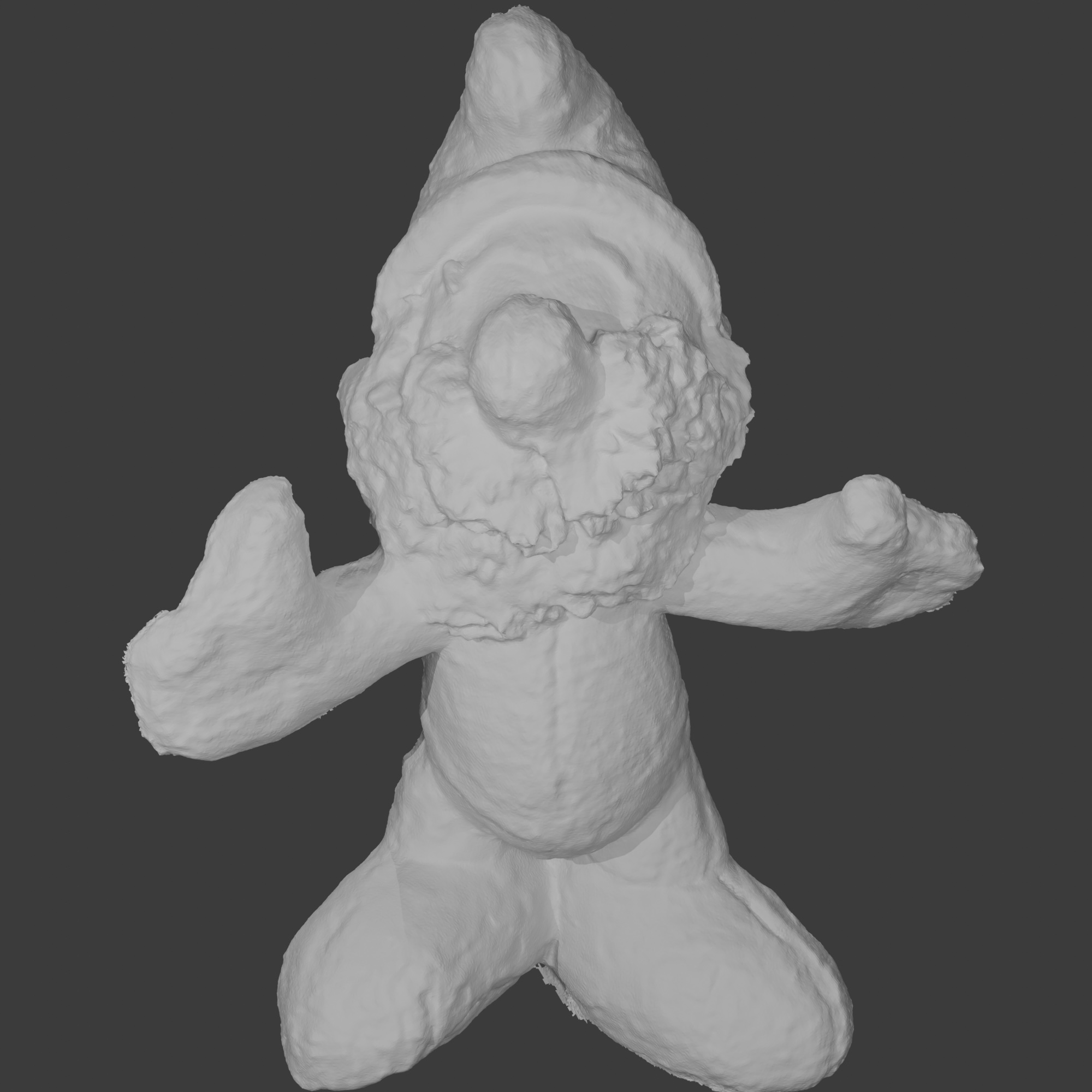} \\
\end{tabular}

\caption{
\textbf{Qualitative mesh renders on DTU.} Mesh renders of our method on three DTU scenes.
}\label{fig:dtu_qualitative}
\end{figure}

\subsection{T\&T Shortcomings Visualizations.}

We show in \cref{fig:tandt_shortcomings} that the T\&T ground-truth point clouds contain holes due to occlusion and unfavorable acquisition angles. These gaps are not a property of the reconstructed mesh but of the acquisition process itself. Methods such as PGSR tend to produce meshes with holes that align with these GT gaps---due to their use of depth filtering~\cite{chen2024pgsr} that relies on removing points viewed at grazing angles---which inflates their scores under the standard F1 metric. Our method reconstructs a more complete surface shell, which can be penalized in these under-observed regions.

\begin{figure}[ht!]
\centering
\setlength{\tabcolsep}{1pt}
\renewcommand{\arraystretch}{0.5}

\begin{tabular}{ccc}
    \small GT Point Cloud & \small PGSR & \small Ours \\[2pt]
    \includegraphics[width=0.323\linewidth]{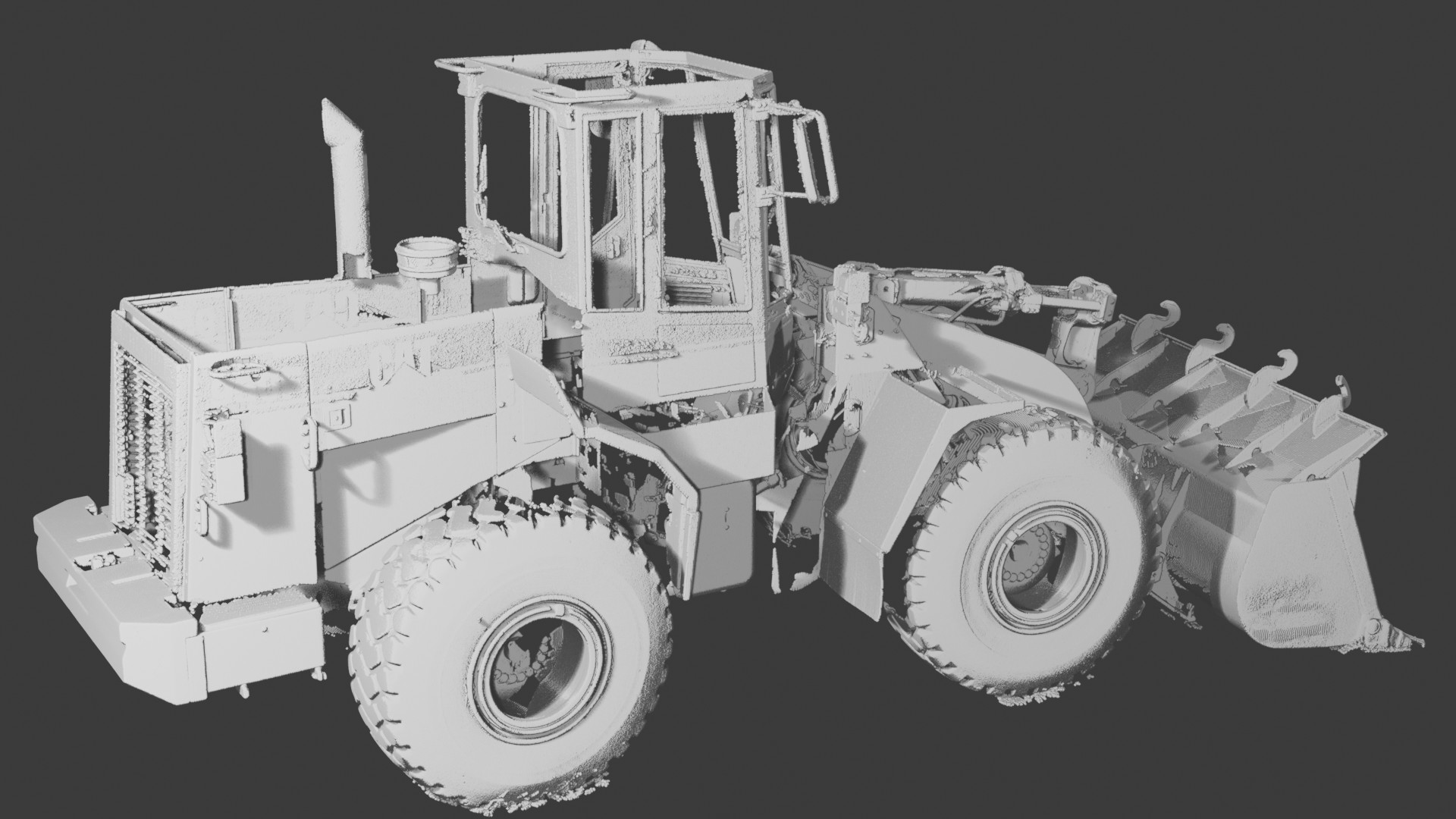} &
    \includegraphics[width=0.323\linewidth]{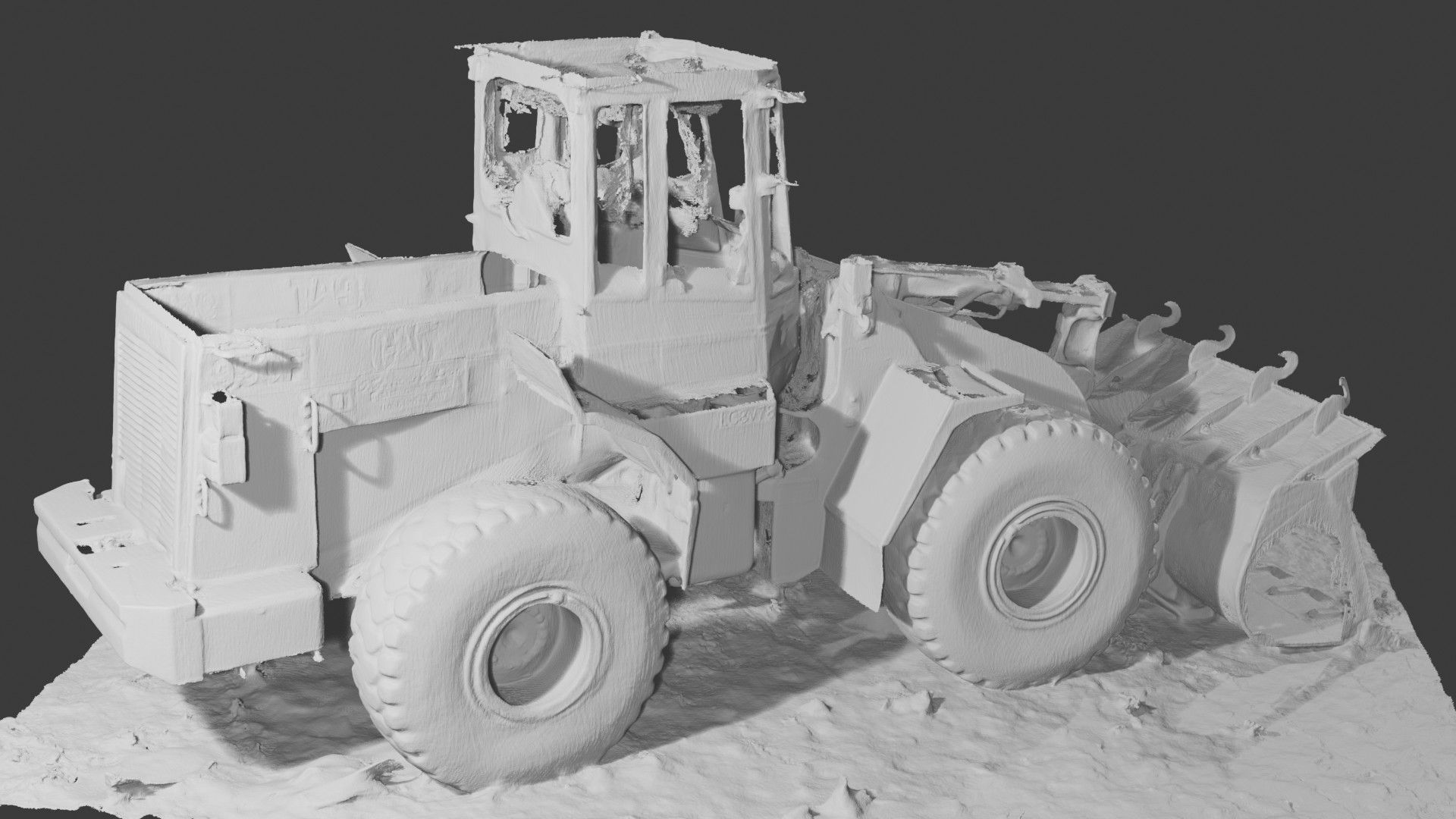} &
    \includegraphics[width=0.323\linewidth]{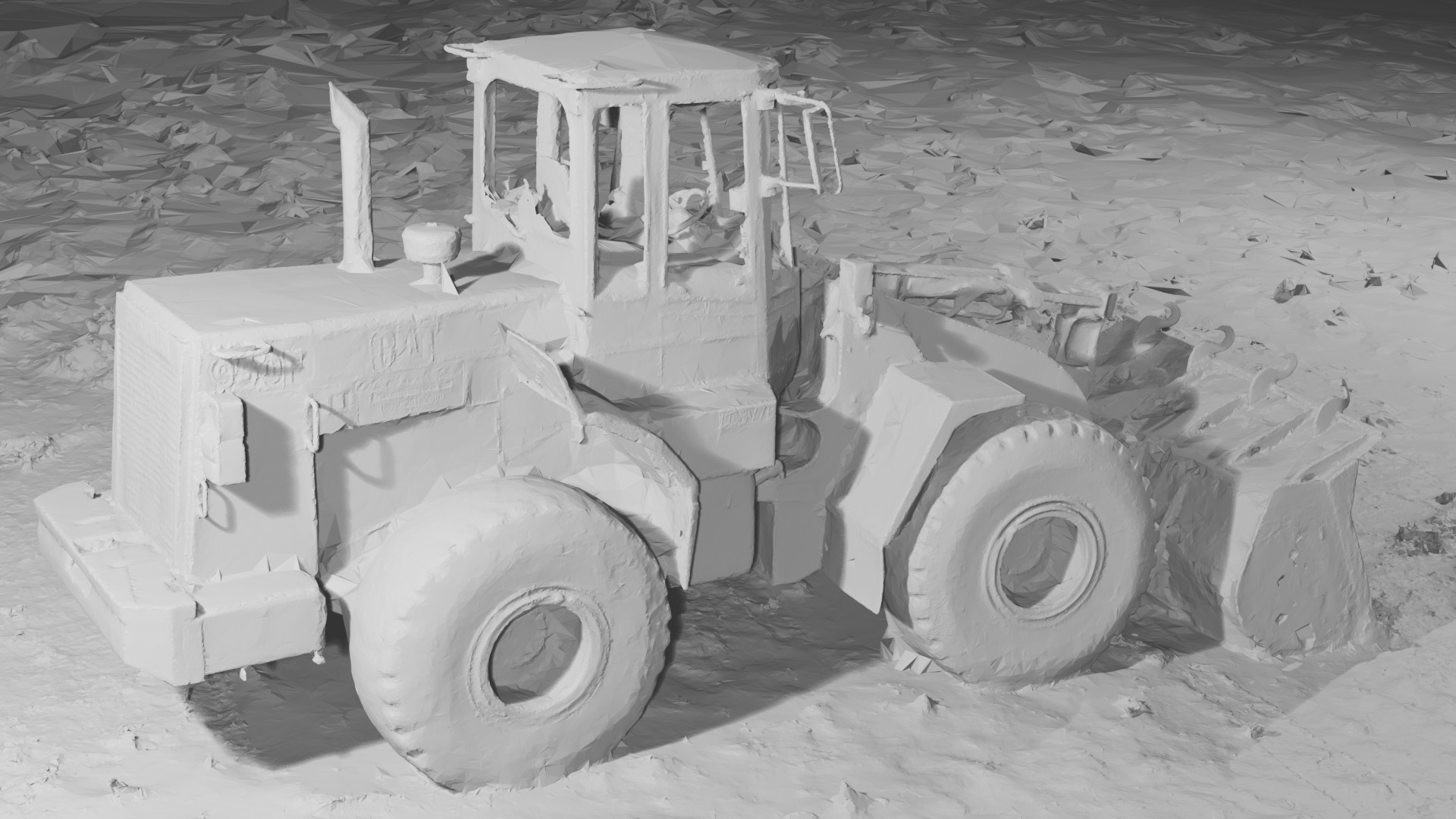} \\[4pt]
    \includegraphics[width=0.323\linewidth]{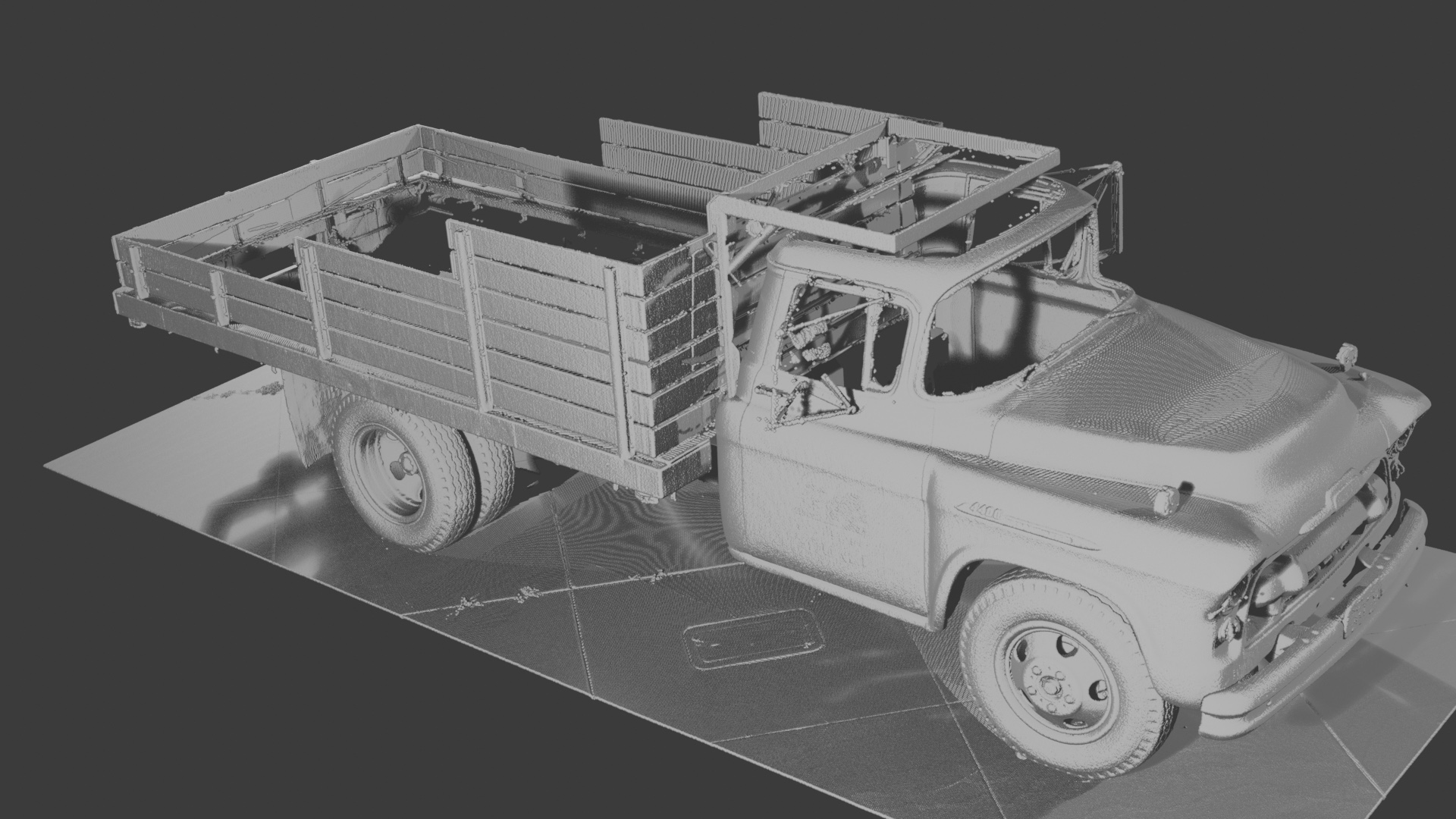} &
    \includegraphics[width=0.323\linewidth]{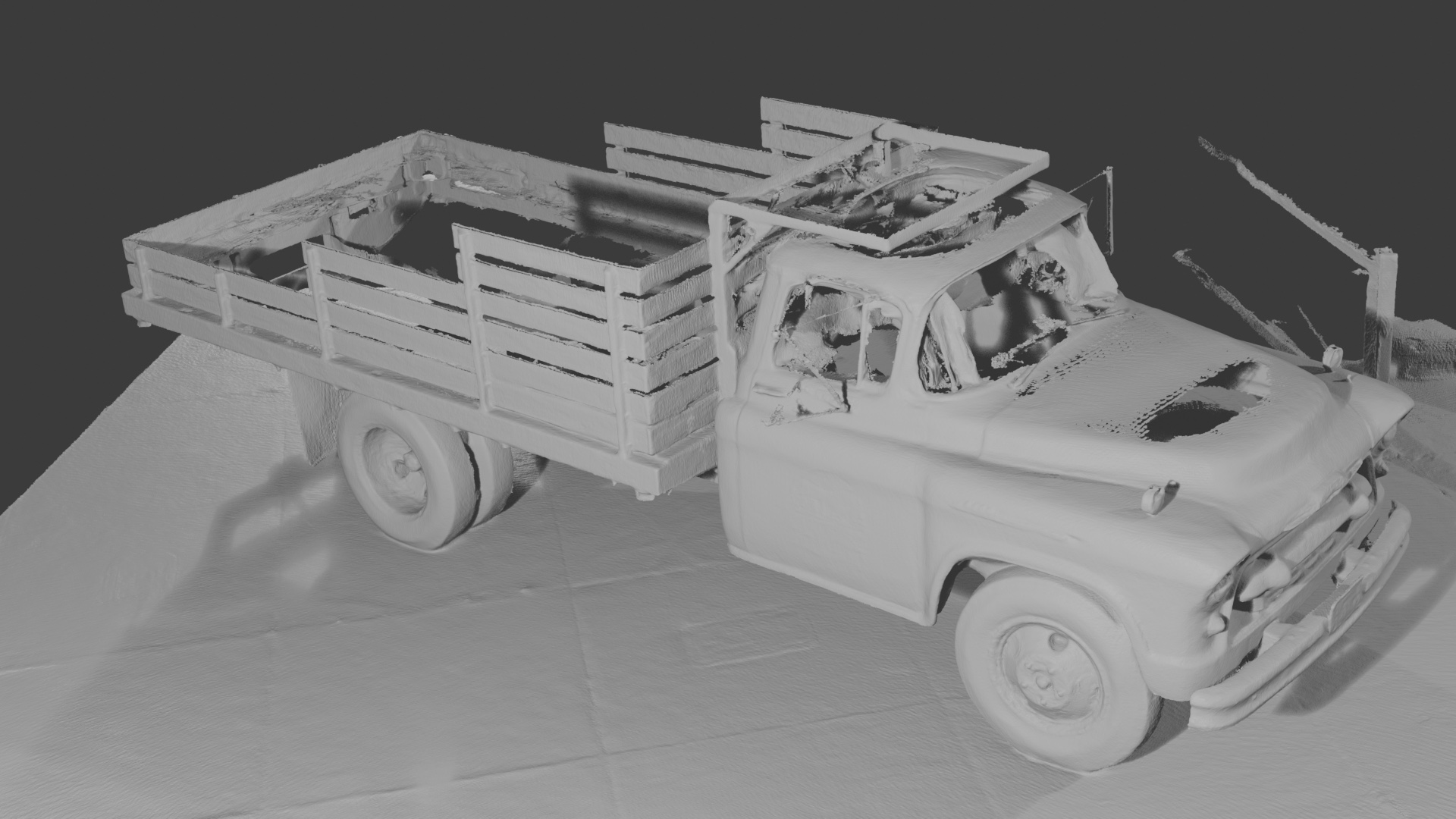} &
    \includegraphics[width=0.323\linewidth]{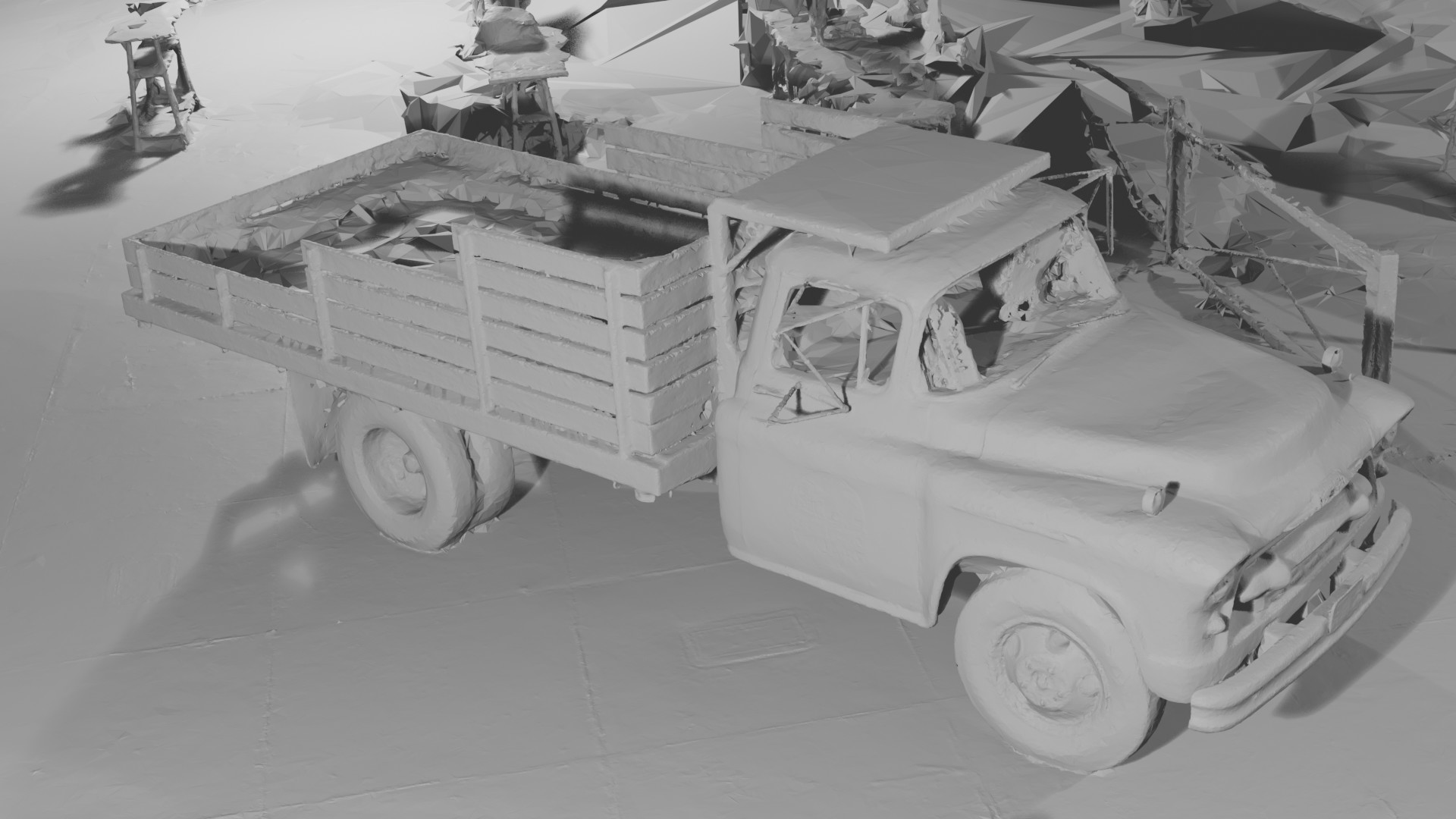} \\
\end{tabular}

\caption{
\textbf{T\&T evaluation shortcomings on the Caterpillar scene.}
The standard Tanks \& Temples ground-truth point cloud (left) contains holes due to occlusion and unfavorable acquisition angles---these are limitations of the LiDAR acquisition, not of the reconstruction.
PGSR (center) produces a mesh with holes that align with GT gaps, which inflates its F1 score.
Our method (right) reconstructs a more complete surface shell, which can be penalized by the evaluation metric in these under-observed regions. Note that there is no GT mesh, instead for visualization purposes we use the Points to Mesh Blender functionality to render the GT point cloud. 
}\label{fig:tandt_shortcomings}
\end{figure}

\begin{figure}[ht!]
\centering
\includegraphics[width=0.75\linewidth]{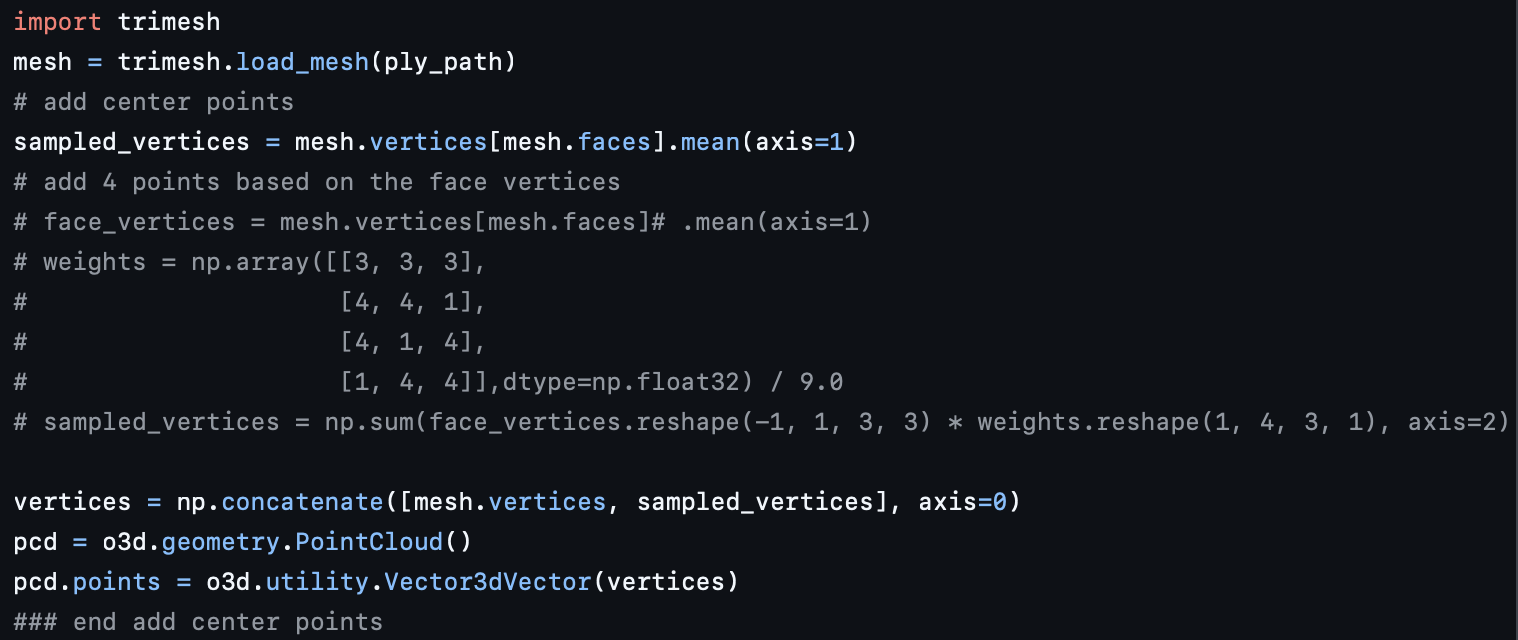}
\caption{
\textbf{Evaluation Code of T\&T used by the literature~\cite{yu2024gaussian,zhang2024rade,chen2024pgsr,jiang2025geometry}}
The traditional evaluation script uses the vertices and face center points from the mesh and uses them to build a point cloud. This way of constructing a prediction point cloud biases the metric towards dense meshes.
}\label{fig:code_flaw}
\end{figure}

\subsection{Additional Quantitative Results.}

In this section we showcase some of the mentioned and discussed experiments and results in the main paper, such as the Legacy Table for the T\&T dataset~(\cref{tab:old_mesh_quality}); the DTU evaluation~(\cref{tab:surface_metrics_dtu}); and the Mip~NeRF360 Novel View Synthesis Evaluation~(\cref{tab:nvs_metrics}).

\begin{table*}[!ht]
    \centering
    \caption{\textbf{Legacy quantitative comparison on the Tanks \& Temples dataset~\cite{Knapitsch2017} using the standard vertex-based metric.} We report F1-score and average optimization time. All results besides ours are taken from the MILo paper. When a method fails due to out-of-memory errors, we report the mean over successful scenes, denoted by a superscript asterisk ($x^*$). Despite the known bias of this metric (see \cref{fig:code_flaw})---vertex-based sampling non-uniformly queries the isosurface in proportion to mesh density---our method achieves the best F1 score among all representations. The sensitivity of this metric to mesh density is further evidenced by the significant score increase obtained simply by adding more pivot points.
    }
    \vspace{-0.2cm}
    \resizebox{0.98\linewidth}{!}{
    \begin{tabular}{@{}l|cccccc|cccccccc|cc}\label{tab:old_mesh_quality}
     & \multicolumn{6}{c@{}|}{Implicit} & \multicolumn{10}{c@{}}{Explicit} \\ 
     & Gaussian UDF & GS-Pull & GSDF & NeuS & Geo-Neus & Neuralangelo & SuGaR & 3DGS & 2DGS & GOF & RaDe-GS & QGS & VCR-GauS & MILo (GOF base) & \textbf{Ours (2p)} & \textbf{Ours (9p)}\\ 
     \hline
    Barn & 0.27 & 0.60 & 0.16 &  0.29 &  0.33 &  0.70  & 0.14 & 0.13 & 0.36 & 0.51 & 0.43 & 0.43 & 0.62 & 0.59 & 0.58 & 0.65 \\
    Caterpillar & 0.17 & 0.37 & 0.11 & 0.29 & 0.26 & 0.36 & 0.16 & 0.08 & 0.23 & 0.41 & 0.32 & 0.31 & 0.26 & 0.39 & 0.48 & 0.50 \\
    Courthouse & 0.03 & 0.16 & OOM & 0.17 & 0.12 & 0.28 & 0.08 & 0.09 & 0.13 & 0.28 & 0.21 & 0.26 & 0.19 & 0.29 & 0.30 & 0.35 \\
    Ignatius & 0.38 & 0.71 & 0.34 & 0.83 & 0.72 & 0.89 & 0.33 & 0.04 & 0.44 & 0.68 & 0.69 & 0.79 & 0.61 & 0.78 & 0.70 & 0.75 \\
    Meetingroom & 0.17 & 0.22 & 0.01 & 0.24 & 0.20 & 0.32 & 0.15 & 0.01 & 0.16 & 0.28 & 0.25 & 0.25 & 0.19 & 0.26 & 0.33 & 0.39 \\
    Truck & 0.30 & 0.52 & 0.25 & 0.45 & 0.45 & 0.48 & 0.26 & 0.19 & 0.26 & 0.59 & 0.51 & 0.60 & 0.52 & 0.62 & 0.58 & 0.61 \\
    \hline
    Mean & 0.22 & 0.43 & $0.18^*$ & 0.38 & 0.35 & 0.50 & 0.19 & 0.09 & 0.30 & 0.46 & 0.40 & 0.44 & 0.40 & 0.49 & 0.50 & 0.54 \\
    Time & 90m & 37.6m & 70m & >24h & >24h & >24h & 73~m & 7.9~m & 12.3~m & 69m & 11.5~m & 120 m & 53 m & 150 m & 27 m & 27m \\ 
    \end{tabular}
    }
    \vspace{-0.2cm}
\end{table*}

\setlength\tabcolsep{0.5em}

\begin{table*}[!ht]
\centering
\caption{\textbf{Surface reconstruction metrics on the DTU dataset}. We report the Chamfer Distance across 15 scenes. Our method achieves competitive performance on object-centric scenes, while being able to recover very fine details in unbounded scenes.
}
\vspace{-0.2cm}
\resizebox{.98\textwidth}{!}{
\begin{tabular}{@{}llcccccccccccccccclc}
\hline
 \multicolumn{3}{c}{} & 24 & 37 & 40 & 55 & 63 & 65 & 69 & 83 & 97 & 105 & 106 & 110 & 114 & 118 & 122 & & Mean \\ \cline{4-18} \cline{20-20}
\multirow{4}{*}{\rotatebox[origin=c]{90}{implicit}} & NeRF~\cite{mildenhall2020nerf} & & 1.90 & 1.60 & 1.85 & 0.58 & 2.28 & 1.27 & 1.47 & 1.67 & 2.05 & 1.07 & 0.88 & 2.53 & 1.06 & 1.15 & 0.96 & & 1.49 \\
& VolSDF~\cite{yariv2021volume_volsdf} & & 1.14 & 1.26 & 0.81 & 0.49 & 1.25 & 0.70 & 0.72 & 1.29 & 1.18 & 0.70 & 0.66 & 1.08 & 0.42 & 0.61 & 0.55 & & 0.86 \\
& NeuS~\cite{wang2021neus} & & 1.00 & 1.37 & 0.93 & 0.43 & 1.10 & 0.65 & 0.57 & 1.48 & 1.09 & 0.83 & 0.52 & 1.20 & \snd{0.35} & 0.49 & 0.54 & & 0.84 \\
& Neuralangelo~\cite{li2023neuralangelo} & & \snd{0.37} & 0.72 & \trd{0.35} & \trd{0.35} & 0.87 & \trd{0.54} & 0.53 & 1.29 & 0.97 & 0.73 & 0.47 & 0.74 & \snd{0.32} & 0.41 & 0.43 & & 0.61 \\ 
\cline{2-2} \cline{4-18} \cline{20-20}
\multirow{10}{*}{\rotatebox[origin=c]{90}{explicit}} 
& 3D GS~\cite{kerbl3Dgaussians} & & 2.14 & 1.53 & 2.08 & 1.68 & 3.49 & 2.21 & 1.43 & 2.07 & 2.22 & 1.75 & 1.79 & 2.55 & 1.53 & 1.52 & 1.50 & & 1.96 \\
& SuGaR~\cite{guedon2024sugar} & & 1.47 & 1.33 & 1.13 & 0.61 & 2.25 & 1.71 & 1.15 & 1.63 & 1.62 & 1.07 & 0.79 & 2.45 & 0.98 & 0.88 & 0.79 & & 1.33  \\
& 2D GS~\cite{huang20242d} && 0.48 & 0.91 & 0.39 & 0.39 & 1.01 & 0.83 & 0.81 & 1.36 & 1.27 & 0.76 & 0.70 & 1.40 & 0.40 & 0.76 & 0.52 & & 0.80 \\
& GOF~\cite{yu2024gaussian} & & 0.50 & 0.82 & 0.37 & 0.37 & 1.12 & 0.74 & 0.73 & 1.18 & 1.29 & 0.68 & 0.77 & 0.90 & 0.42 & 0.66 & 0.49 & & 0.74\\
& RaDe-GS~\cite{zhang2024rade} & & 0.46 & 0.73 & 0.33 & 0.38 & \snd{0.79} & 0.75 & 0.76 & 1.19 & 1.22 & \trd{0.62} & 0.70 & 0.78 & \trd{0.36} & 0.68 & 0.47 & & 0.68 \\
& MILo \cite{guedon2025milo} & & 0.43 & 0.74 & 0.34 & 0.37 & \trd{0.80} & 0.74 & 0.70 & 1.21 & 1.22 & 0.66 & 0.62 & 0.80 & 0.37 & 0.76 & 0.48 & & 0.68 \\
& PGSR \cite{chen2024pgsr} & & \fst{0.34} & \trd{0.54} & 0.44 & 0.37 & \fst{0.78} & 0.57 & 0.49 & \trd{1.06} & \fst{0.63} & \snd{0.59} & 0.47 & \fst{0.50} & \fst{0.30} & \trd{0.37} & \trd{0.34} & & 0.52  \\
& GGGS (20k iter) & & \trd{0.38} & \fst{0.50} & \snd{0.27} & \fst{0.31} & \trd{0.80} & \fst{0.43} & \fst{0.42} & \fst{1.04} & \snd{0.64} & \fst{0.52} & \fst{0.31} & \snd{0.56} & \fst{0.30} & \fst{0.31} & \fst{0.33} & & \fst{0.47}\\
& GGGS (30k iter) & & \snd{0.37} & \snd{0.51} & \snd{0.27} & \fst{0.31} & 0.81 & \fst{0.43} & \fst{0.42} & \snd{1.05} & \snd{0.64} & \fst{0.52} & \snd{0.32} & \trd{0.58} & \fst{0.30} & \fst{0.31} & \fst{0.33} & & \snd{0.48}  \\
& \textbf{Ours (20k iter)} & & 0.39 & \fst{0.50} & \fst{0.26} & \snd{0.32} & \trd{0.80} & \snd{0.46} & \snd{0.44} & \snd{1.05} & \trd{0.65} & \fst{0.52} & \trd{0.34} & \trd{0.58} & \snd{0.32} & \snd{0.33} & \snd{0.34} & & \trd{0.49} \\
& \textbf{Ours (30k iter)} & & 0.40 & \snd{0.51} & \snd{0.27} & \snd{0.32} & 0.81 & \snd{0.46} & \trd{0.45} & \fst{1.04} & \trd{0.65} & \fst{0.52} & \trd{0.33} & 0.59 & \snd{0.32} & \snd{0.33} & \snd{0.34} & & \trd{0.49} \\
 \hline
\end{tabular}
}
\label{tab:surface_metrics_dtu}
\vspace{-0.1cm}
\end{table*}

\begin{table}[t]
    \centering
    \caption{\textbf{Quantitative results for novel view synthesis on MipNeRF~360~\cite{barron22mipnerf360}}. We report PSNR, SSIM, and LPIPS. Our method maintains competitive rendering quality while significantly improving surface reconstruction.}
    \vspace{-0.2cm}
    \resizebox{0.99\linewidth}{!}{
    \begin{tabular}{@{}l|ccc|ccc}
     & \multicolumn{3}{c@{}|}{Indoor Scenes} & \multicolumn{3}{c@{}}{Outdoor Scenes}\\ 
     & PSNR $\uparrow$ & SSIM $\uparrow$ & LPIPS $\downarrow$
     & PSNR $\uparrow$ & SSIM $\uparrow$ & LPIPS $\downarrow$ 
     \\ 
     \hline
    3DGS \cite{kerbl3Dgaussians}
    & 30.41 & 0.920 & 0.189
    & 24.64 & 0.731 & 0.234
    \\
    Mip-Splatting \cite{yu2024mip}
    & \best 30.90 & 0.921 & 0.194 
    & 24.65 & 0.729 & 0.245 
    \\
    BakedSDF \cite{yariv2023bakedsdf}
    & 27.06 & 0.836 & 0.258
    & 22.47 & 0.585 & 0.349 
    \\
    SuGaR \cite{guedon2024sugar}
    & 29.43 & 0.906 & 0.225
    & 22.93 & 0.629 & 0.356 
    \\
    2DGS \cite{ren20252dgs}
    & 30.40 & 0.916 & 0.195 
    & 24.34 & 0.717 & 0.246 
    \\
    GOF \cite{yu2024gaussian}
    & \sbest 30.79 & 0.924 & 0.184 
    & \tbest 24.82 & \tbest 0.750 & \tbest 0.202 
    \\
    RaDe-GS \cite{zhang2024rade}
    & 30.74 & \tbest 0.928 & \tbest 0.165 
    & \sbest 25.17 & \sbest 0.764 & \sbest 0.199 
    \\
    MILo (dense) \cite{guedon2025milo}
    & \tbest 30.76 & \best 0.934 & \best 0.155 
    & 24.81 & 0.744 & 0.229 
    \\
    Ours
    & 30.43 & \sbest 0.932 & \sbest 0.162 
    & \best 26.49 & \best 0.815 & \best 0.174 
    \\
    \end{tabular}
    }
    \label{tab:nvs_metrics}
    \vspace{-0.2cm}
\end{table}

\subsection{Additional Qualitative Results}

In this section we present some additional qualitative results. In~\cref{fig:ablation_nf_barn}, we show qualitative results of our ablation study. In~\cref{fig:ablation_radegs}, we show qualitative results of adding Gaussian Wrapping to RaDe-GS.

\begin{figure}[ht!]
\centering
\setlength{\tabcolsep}{1pt}
\renewcommand{\arraystretch}{0.5}

\begin{tabular}{cc}
    \small Baseline & \small +$\mathcal{L}_{\text{N}}$ \\[2pt]
    \includegraphics[width=0.49\linewidth]{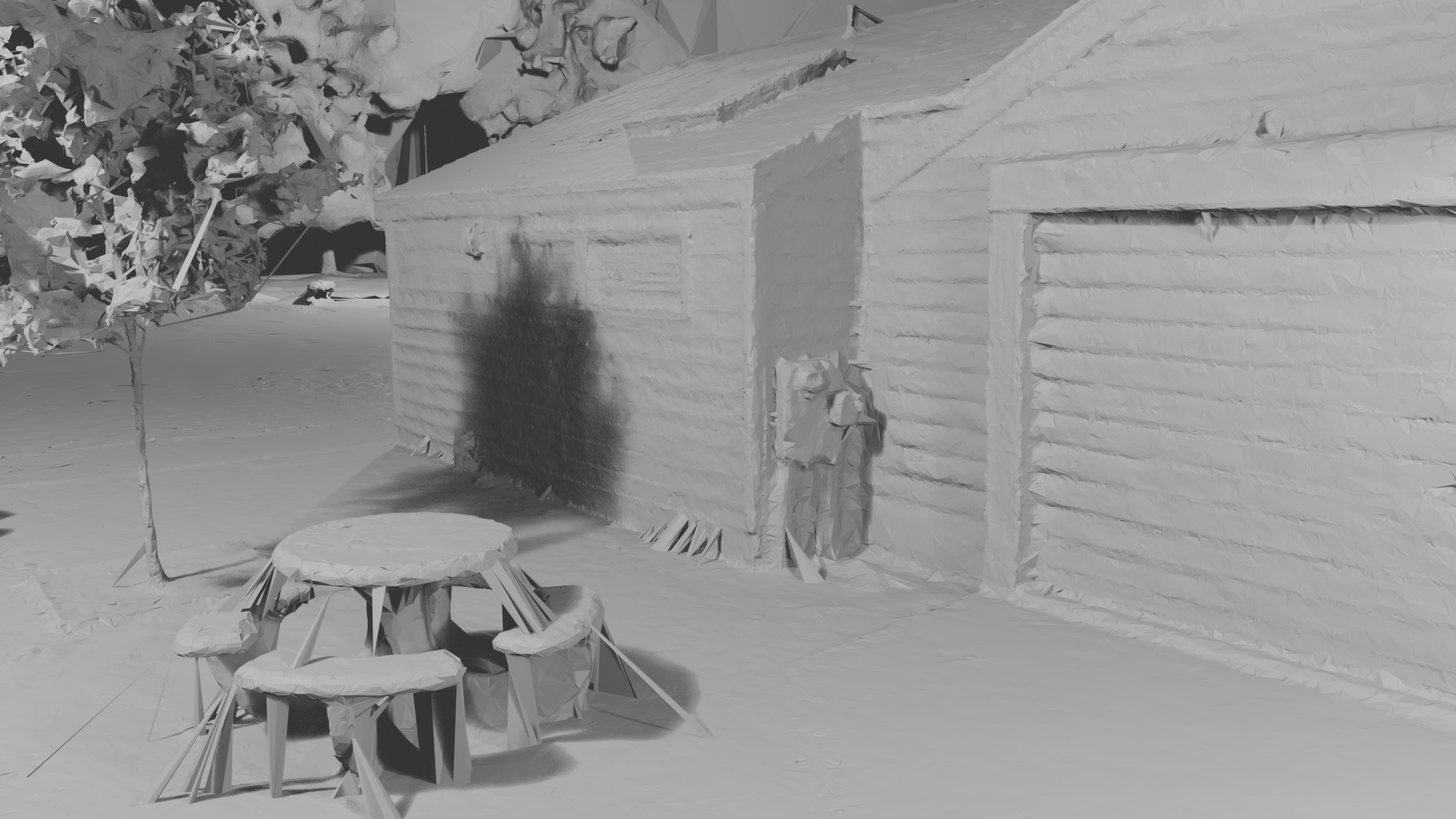} &
    \includegraphics[width=0.49\linewidth]{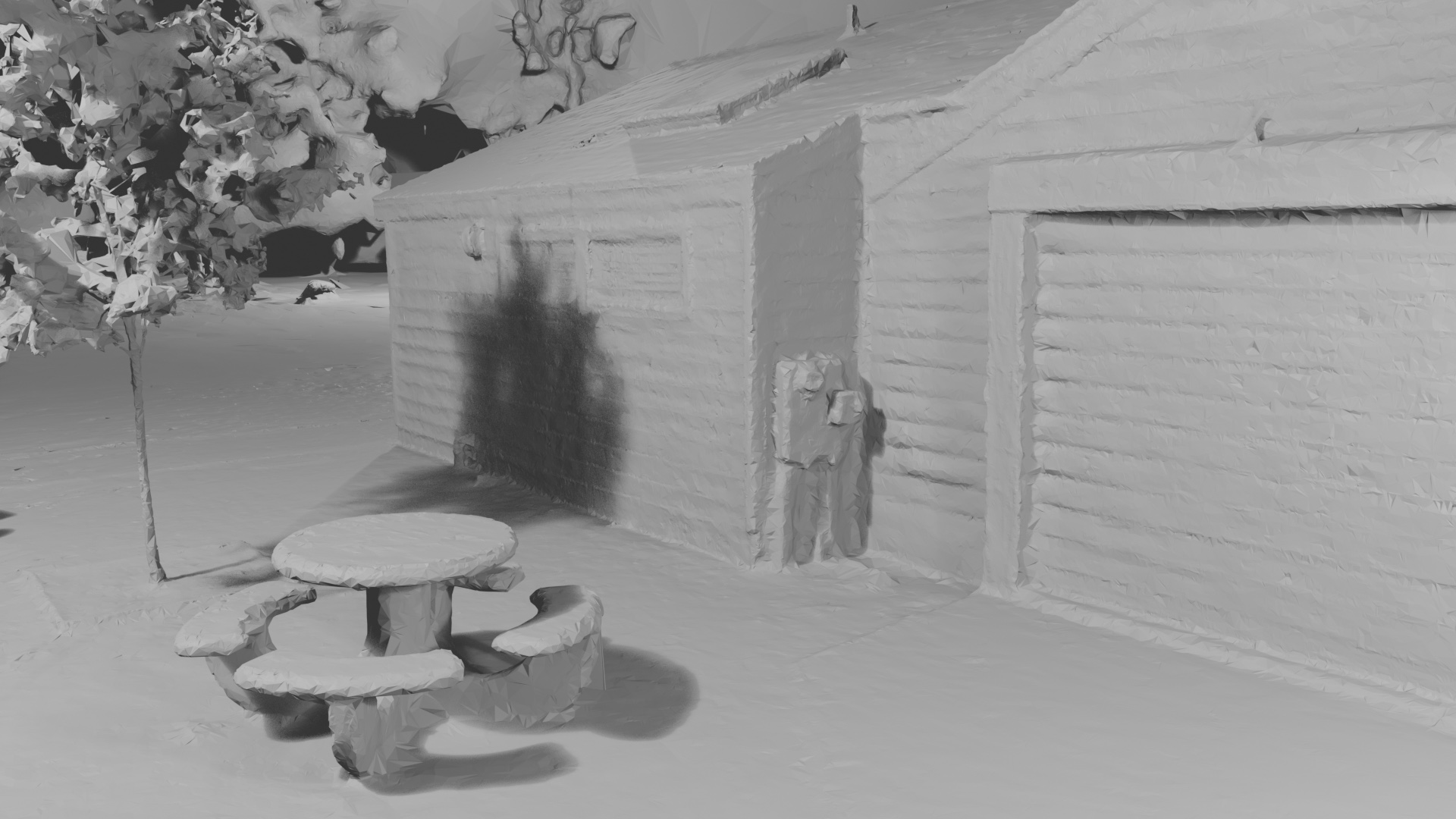} \\[6pt]
    \small +$\mathcal{L}_{\text{N}}$ & \small Ours \\[2pt]
    \includegraphics[width=0.49\linewidth]{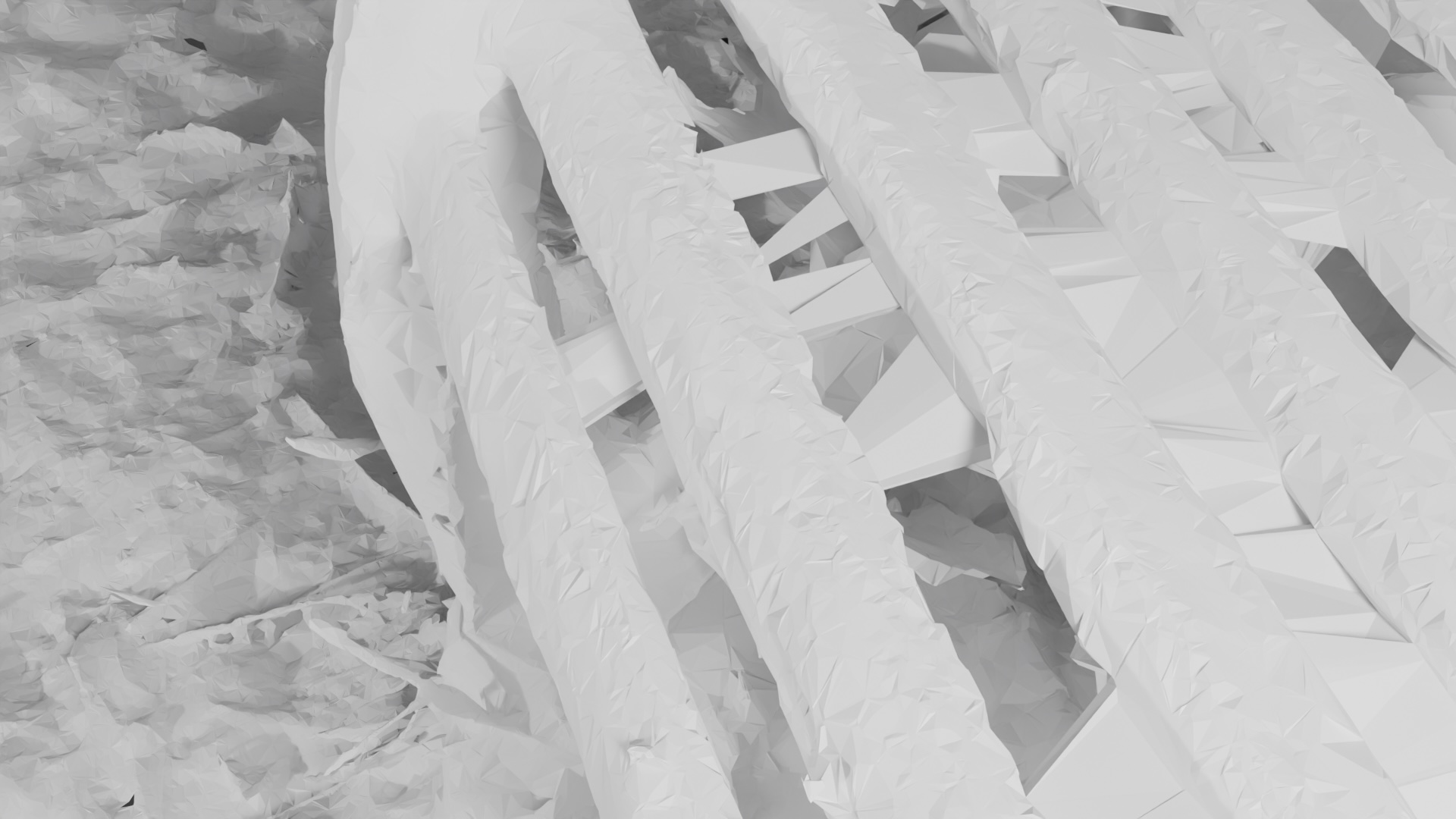} &
    \includegraphics[width=0.49\linewidth]{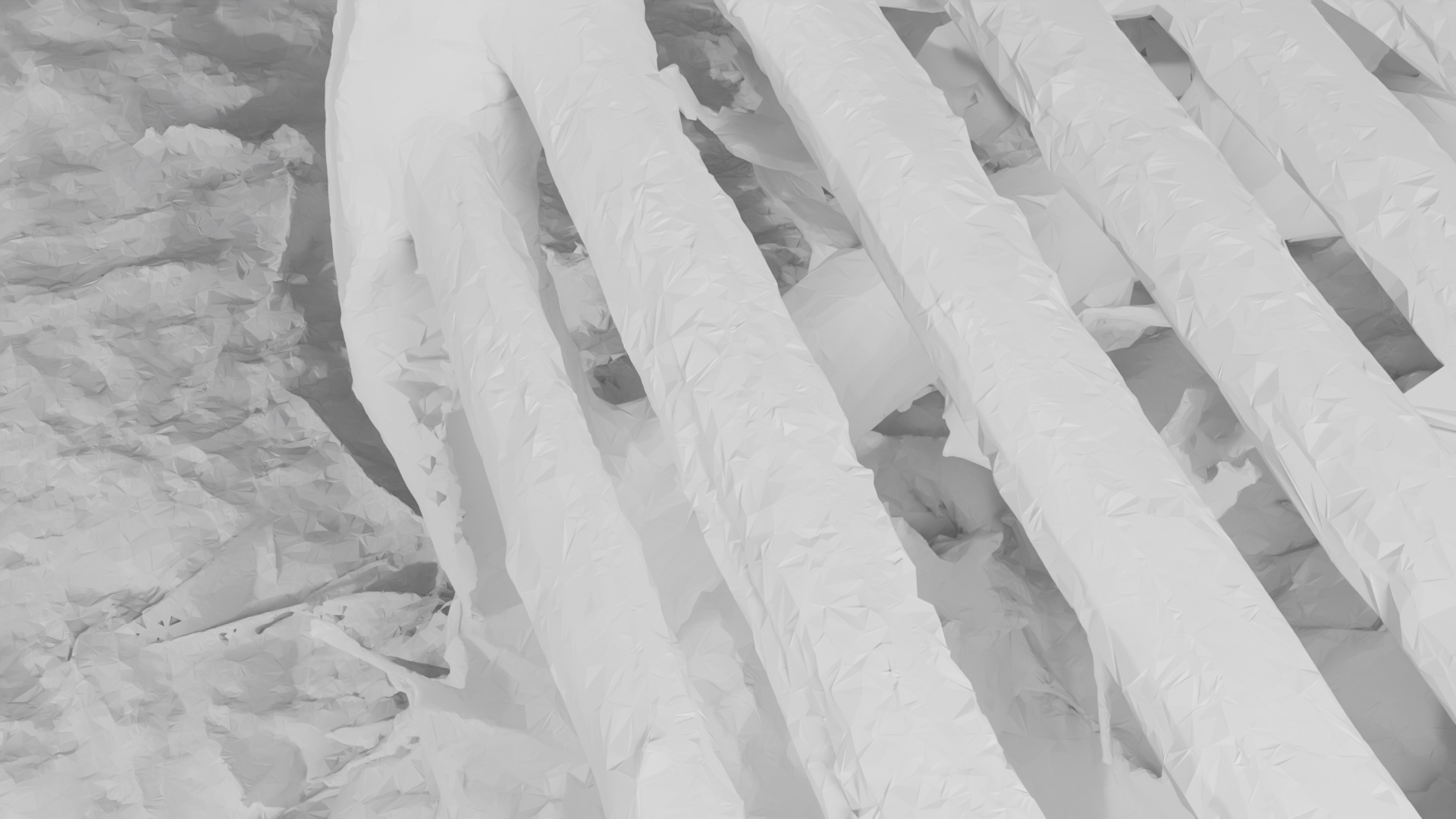} \\
\end{tabular}

\caption{
\textbf{Ablation of the normal field loss $\mathcal{L}_{\text{N}}$ and densification on the Barn and Truck scenes.}
On the Barn scene, we see that by extracting a mesh with 2 pivots without our alignment loss $\mathcal{L}_N$, several artifacts are present. These come in the form of large triangles that stem from a bad orientation of the Gaussians which result in bad placement for the pivots. Then below, we ablate the need for our densification (Ours = Baseline + $\mathcal{L}_N$ + Densification), we can see that it allows to solve for complex and poorly supervised geometry such as in between the bench slats in the bicyle scene.
}\label{fig:ablation_nf_barn}
\end{figure}

\begin{figure}[t]
\centering
\setlength{\tabcolsep}{1pt}
\begin{tabular}{cc}
    \small RaDe-GS & \small RaDe-GS + NF (Ours) \\[2pt]
    \includegraphics[width=0.49\linewidth]{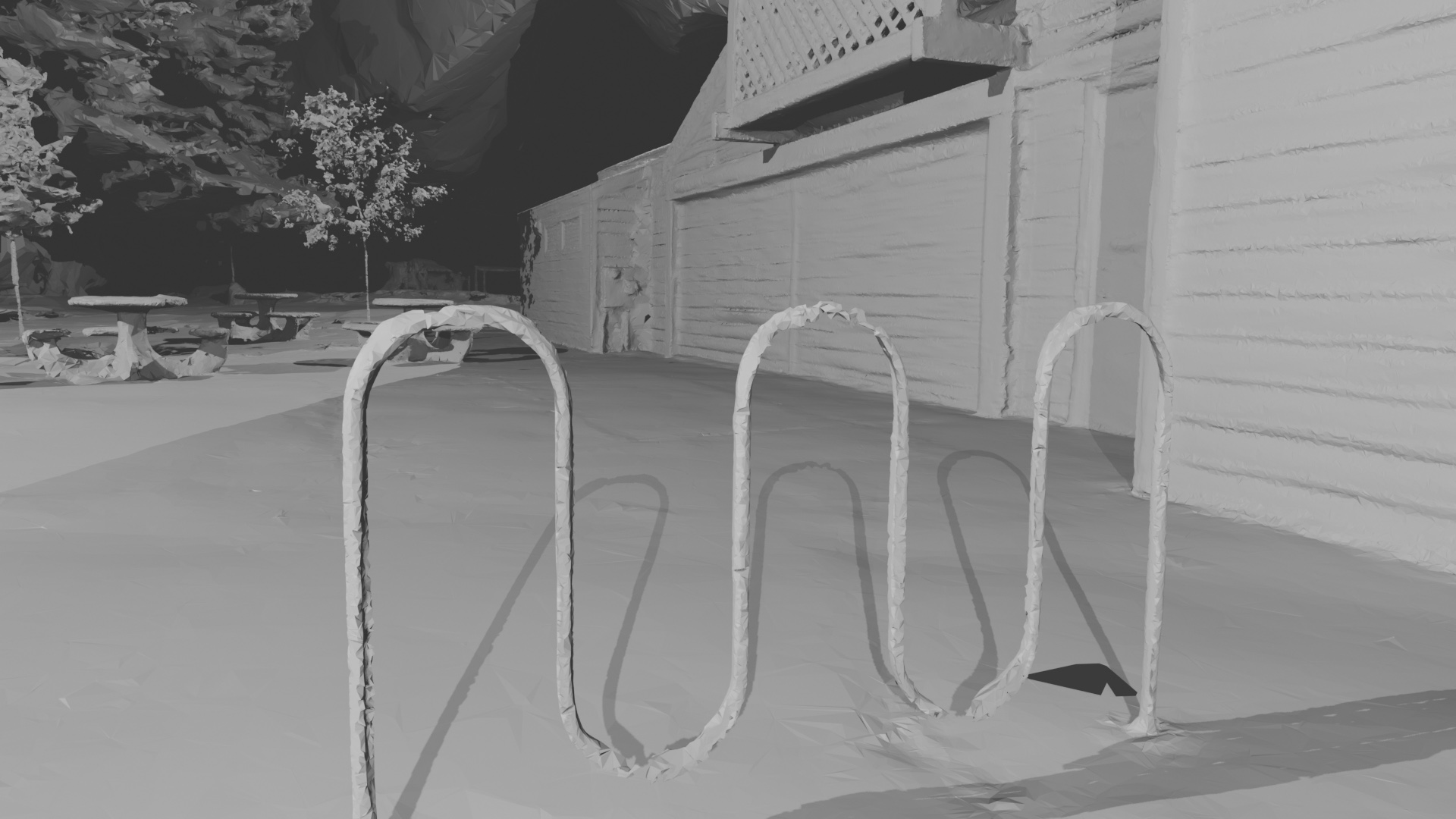} &
    \includegraphics[width=0.49\linewidth]{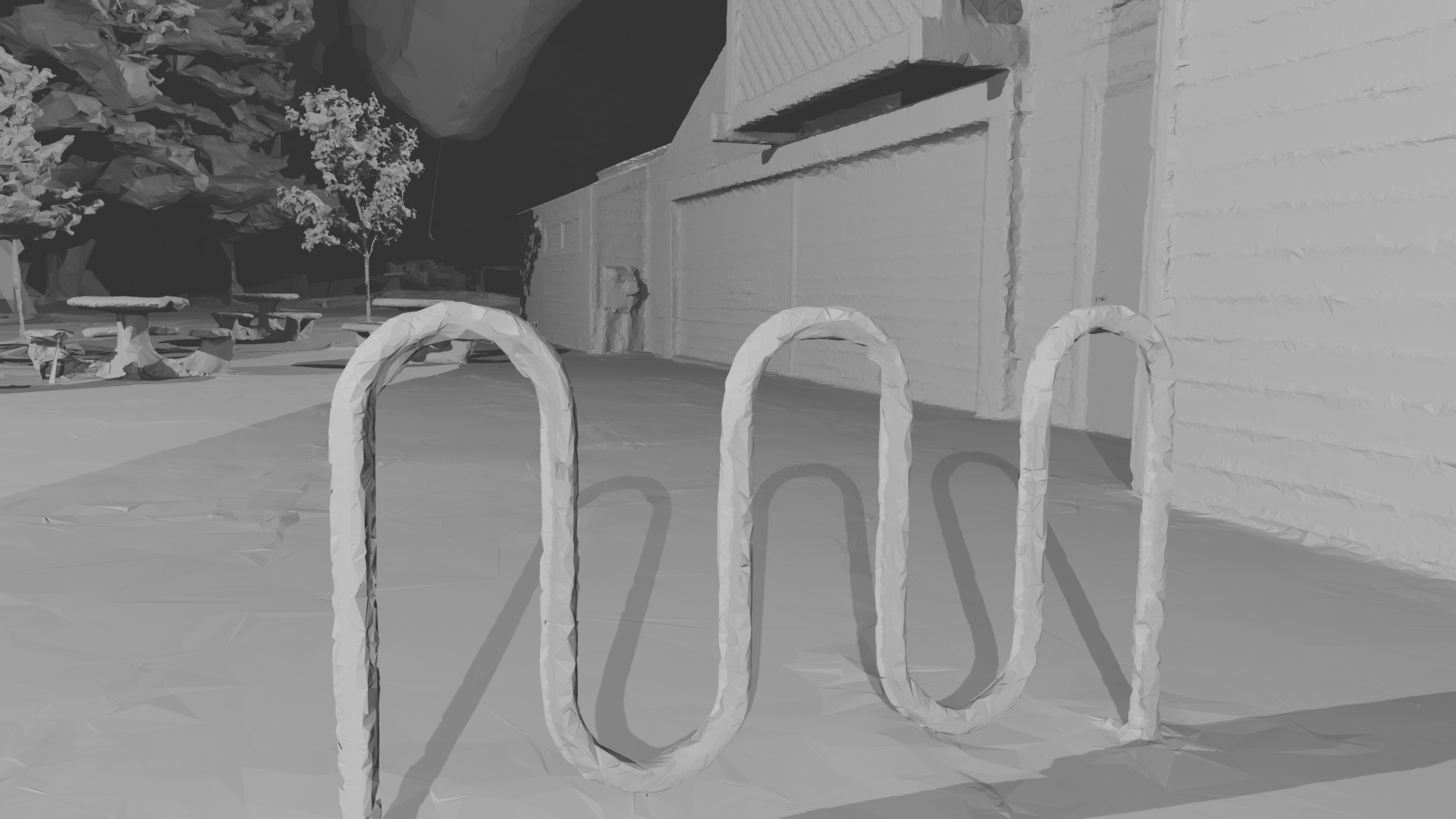} \\[4pt]
    \includegraphics[width=0.49\linewidth]{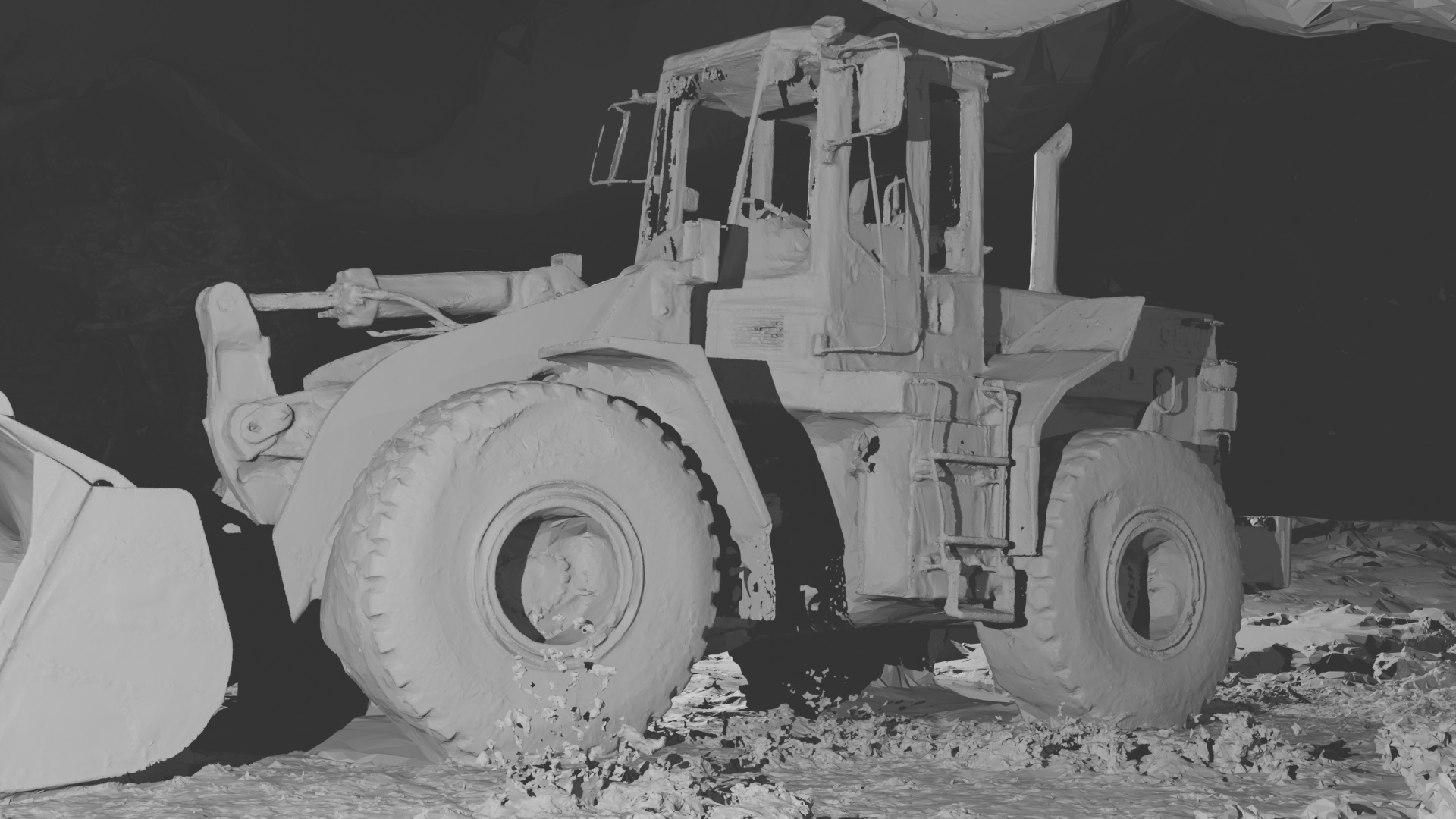} &
    \includegraphics[width=0.49\linewidth]{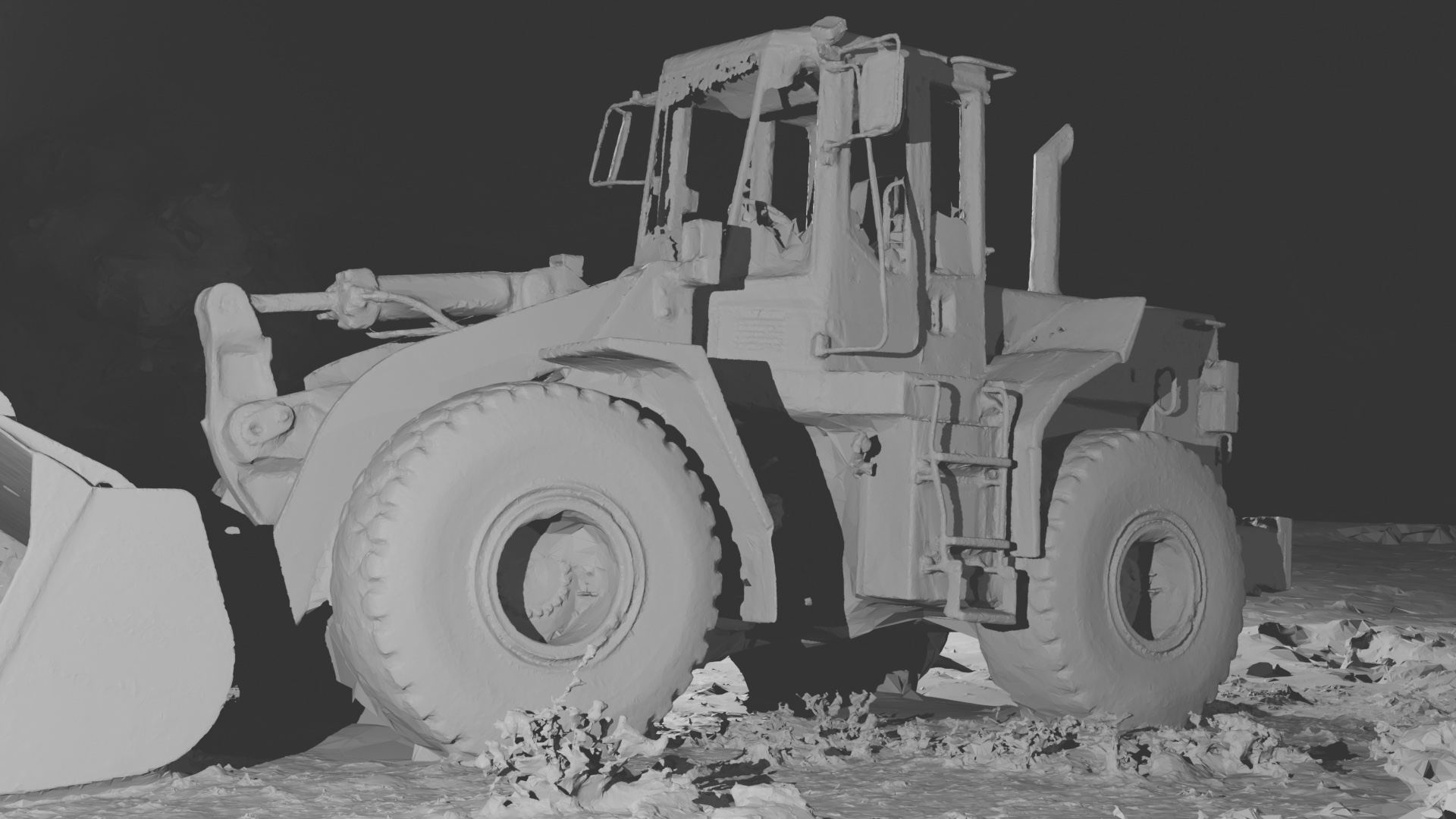} \\
\end{tabular}
\caption{\textbf{Qualitative effect of Normal Field (NF) supervision on RaDe-GS.} Adding our NF losses and densification to RaDe-GS yields more faithful geometry. For instance, we mititage the erosion of the bike lock in the Barn scene, and avoid holes in the wheel of the Caterpillar scene. This shows our approach can be used to improved other existing methods.}
\label{fig:ablation_radegs}
\end{figure}

\end{document}